\documentclass{article}

\pdfoutput=1

\usepackage[nonatbib, final]{neurips_2021}
\usepackage[numbers]{natbib}

\usepackage[utf8]{inputenc} 
\usepackage[T1]{fontenc}    
\usepackage{hyperref}       
\usepackage{url}            
\usepackage{booktabs}       
\usepackage{amsfonts}       
\usepackage{nicefrac}       
\usepackage{microtype}      
\usepackage{xcolor}         

\usepackage{amsmath,amsfonts,bm}

\def\eqref#1{equation~\ref{#1}}

\def\1{\bm{1}}

\def\ve{{\bm{e}}}

\def\vm{{\bm{m}}}

\def\vp{{\bm{p}}}
\def\vq{{\bm{q}}}

\def\vu{{\bm{u}}}

\def\vx{{\bm{x}}}

\def\vpi{{\bm{\pi}}}

\def\evpi{{\pi}}

\def\eve{{e}}

\def\evp{{p}}

\DeclareMathAlphabet{\mathsfit}{\encodingdefault}{\sfdefault}{m}{sl}
\SetMathAlphabet{\mathsfit}{bold}{\encodingdefault}{\sfdefault}{bx}{n}

\def\sX{{\mathbb{X}}}
\def\sY{{\mathbb{Y}}}

\newcommand{\KL}{D_{\mathrm{KL}}}

\DeclareMathOperator*{\argmax}{arg\,max}

\newcommand{\dvp}[1]{\vp^{(#1)}}

\newcommand{\dve}[1]{\ve^{(#1)}}
\newcommand{\dvm}{\bar{\vp}_{>1}}

\newcommand{\edvp}[2]{\evp^{(#1)}_{#2}}

\newcommand{\edve}[2]{\eve^{(#1)}_{#2}}

\newcommand{\s}[2]{\sum_{#1}^{#2}}

\newcommand{\DJS}{\mathrm{JS}_{\vpi}}
\newcommand{\DGJS}{\mathrm{GJS}_{\vpi}}

\newcommand{\DJSnp}{\mathrm{JS}}
\newcommand{\DGJSnp}{\mathrm{GJS}}
\newcommand{\JS}{D_{\mathrm{JS}_{\vpi}}}
\newcommand{\GJS}{D_{\mathrm{GJS}_{\vpi}}}
\newcommand{\JSnp}{D_{\mathrm{JS}}}
\newcommand{\GJSnp}{D_{\mathrm{GJS}}}
\newcommand{\LJS}{\mathcal{L}_{\mathrm{JS}}}
\newcommand{\LGJS}{\mathcal{L}_{\mathrm{GJS}}}

\newcommand{\DKL}[2]{ \KL ( #1 \Vert #2 ) }

\newcommand{\xlogxy}[2]{ #1 \log{ \Big( \frac{#1}{#2} \Big) }  }

\newcommand{\xlogyz}[3]{ #1 \log{ \Big( \frac{#2}{#3} \Big) }  }

\usepackage{mathtools}      
\usepackage{amsthm}         

\newtheorem{theorem}{Theorem}
\newtheorem{lemma}{Lemma} 
\newtheorem{remark}{Remark}

\usepackage{thmtools,thm-restate}

\usepackage{siunitx} 
\usepackage{multirow}
\usepackage{array}
\usepackage{caption}
\usepackage{subcaption}
\usepackage{graphicx}
\usepackage{bbm} 

\usepackage{xspace}
\newcommand*{\eg}{\textit{e.g.}\@\xspace}

\newcommand*{\ie}{\textit{i.e.}\@\xspace}
\newcommand*{\etal}{\textit{et~al.}\@\xspace}

\title{Generalized Jensen-Shannon Divergence Loss \\ for Learning with Noisy Labels}

\author{%
  Erik Englesson \\
  KTH \\
  Stockholm, Sweden \\
  \texttt{engless@kth.se} \\
   \And
  Hossein Azizpour \\
  KTH \\
  Stockholm, Sweden \\
  \texttt{azizpour@kth.se} \\
}

\begin{document}

\maketitle

\begin{abstract}
Prior works have found it beneficial to combine provably noise-robust loss functions \textit{e.g.}, mean absolute error (MAE) with standard categorical loss function \textit{e.g.} cross entropy (CE) to improve their learnability. Here, we propose to use Jensen-Shannon divergence as a noise-robust loss function and show that it interestingly interpolate between CE and MAE with a controllable mixing parameter. Furthermore, we make a crucial observation that CE exhibits lower consistency around noisy data points. Based on this observation, we adopt a generalized version of the Jensen-Shannon divergence for multiple distributions to encourage consistency around data points. Using this loss function, we show 
state-of-the-art results on both synthetic (CIFAR), and real-world (\eg WebVision) noise with varying noise rates.
\end{abstract}
\vspace{-0.4cm}
\section{Introduction}
\label{sec:intro}
Labeled datasets, even the systematically annotated ones, contain  noisy labels~\citep{Beyer_arXiv_2020_done_ImageNet}. Therefore, designing noise-robust learning algorithms are crucial for the real-world tasks. An important avenue to tackle noisy labels is to devise noise-robust loss functions~\citep{Ghosh_AAAI_2017_MAE,Zhang_NeurIPS_2018_Generalized_CE,Wang_ICCV_2019_Symmetric_CE, Ma_ICML_2020_Normalized_Loss}. Similarly, in this work, we propose two new noise-robust loss functions based on two central observations as follows.

Observation I: \textit{Provably-robust loss functions can underfit the training data}~\cite{Ghosh_AAAI_2017_MAE,Zhang_NeurIPS_2018_Generalized_CE,Wang_ICCV_2019_Symmetric_CE,Ma_ICML_2020_Normalized_Loss}.\\
Observation II: \textit{Standard networks show low consistency around noisy data points}
\footnote{we call a network \textit{consistent} around a sample~($\vx$) if it predicts the same class for $\vx$ and its perturbations~($\tilde{\vx}$).}
\hspace{-0.1cm}, see Figure \ref{fig:consistency-observation}. 

We first propose to use Jensen-Shannon divergence (JS) as a loss function, which we crucially show interpolates between the noise-robust mean absolute error (MAE) and the cross entropy (CE) that better fits the data through faster convergence. Figure \ref{fig:js-ce-mae} illustrates the CE-MAE interpolation.\\
Regarding Observation II, we adopt the generalized version of Jensen-Shannon divergence (GJS) to encourage predictions on perturbed inputs to be consistent, see Figure \ref{fig:gjs-dissection-loss}.
Notably, Jensen-Shannon divergence has previously shown promise for test-time robustness to domain shift~\cite{hendrycks2020augmix}, here we further argue for its \textit{training-time} robustness to \textit{label noise}. The key contributions of this work\footnote{implementation available at \url{https://github.com/ErikEnglesson/GJS}} are:
\begin{itemize}
\item We make a novel observation that a network predictions' consistency is reduced for noisy-labeled data when overfitting to noise, which motivates the use of consistency regularization. 
\item We propose using Jensen-Shannon divergence~($\DJSnp$) and its multi-distribution generalization~($\DGJSnp$) as loss functions for learning with noisy labels.
We relate $\DJSnp$ to loss functions that are based on the noise-robustness theory of Ghosh~\etal~\citep{Ghosh_AAAI_2017_MAE}. In particular, we prove that $\DJSnp$ generalizes CE and MAE. Furthermore, we prove that $\DGJSnp$ generalizes $\DJSnp$ by incorporating consistency regularization in a single principled loss function.
\item We provide an extensive set of empirical evidences on several datasets, noise types and rates. They show state-of-the-art results and give in-depth studies of the proposed losses.   
\end{itemize}
\vspace{-0.2cm}
\begin{figure*}
        \centering
        \begin{subfigure}[b]{0.32\textwidth} \centering \includegraphics[width=0.9\textwidth]{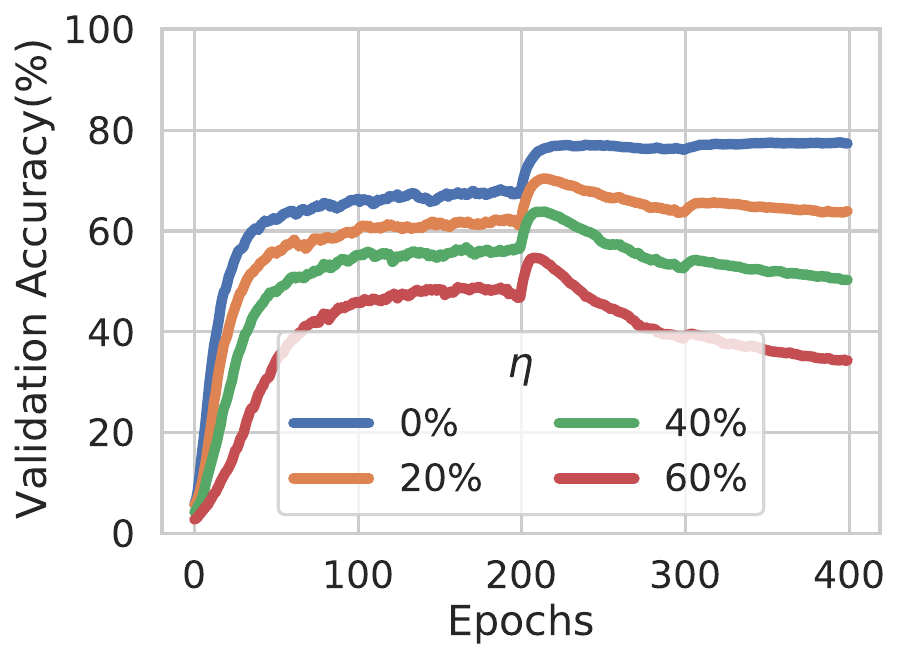} \caption{Validation Accuracy} \label{fig:consistency-observation-a}
         \end{subfigure} 
         \hfill
         \begin{subfigure}[b]{0.32\textwidth} \centering \includegraphics[width=0.9\textwidth]{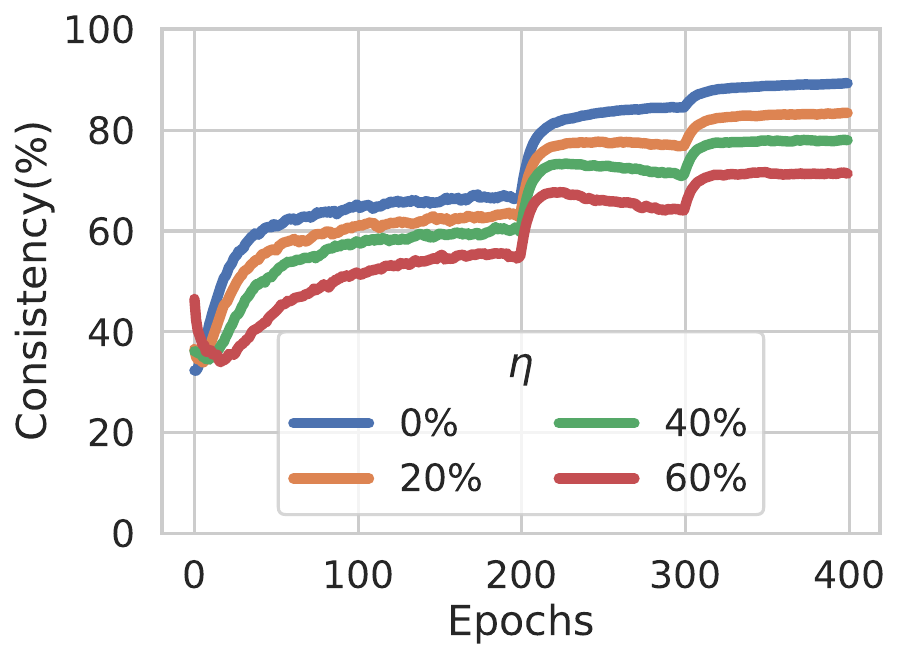} \caption{Consistency Clean} \label{fig:consistency-observation-b} 
         \end{subfigure} 
         \begin{subfigure}[b]{0.32\textwidth} \centering \includegraphics[width=0.9\textwidth]{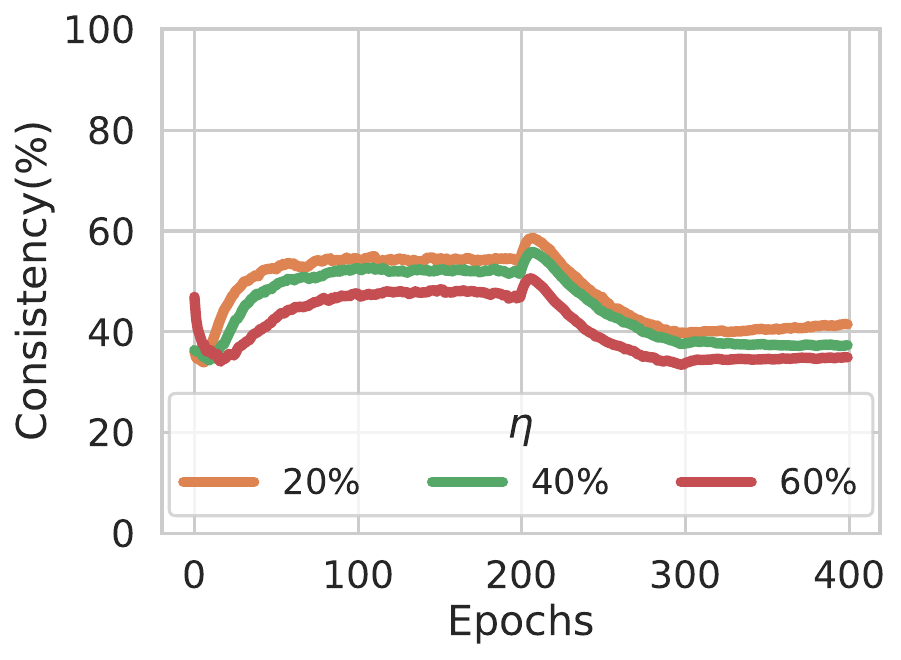} \caption{Consistency Noisy} \label{fig:consistency-observation-c} 
         \end{subfigure} 
         \caption{\textbf{Evolution of a trained network's consistency as  it overfits to noise using CE loss.} Here we plot the evolution of the validation accuracy~(a) and network's consistency (as measured by GJS) on clean~(b) and noisy~(c) examples of the training set of CIFAR-100 for varying symmetric noise rates when learning with the cross-entropy loss. The consistency of the learnt function and the accuracy closely correlate. This suggests that enforcing consistency may help avoid fitting to noise. Furthermore, the consistency is degraded more significantly for the noisy data points.} \label{fig:consistency-observation}
         \vspace{-0.2cm}
\end{figure*}

\section{Generalized Jensen-Shannon Divergence}
\label{sec:GJS}
\vspace{-0.2cm}
We propose two loss functions, the Jensen-Shannon divergence~($\DJSnp$) and its multi-distribution generalization~($\DGJSnp$). In this section, we first provide background and two observations that motivate our proposed loss functions. This is followed by definition of the losses, and then we show that $\DJSnp$ generalizes CE and MAE similarly to other robust loss functions. Finally, we show how $\DGJSnp$ generalizes $\DJSnp$ to incorporate consistency regularization into a single principled loss function. We provide proofs of all theorems, propositions, and remarks in this section in Appendix \ref{sup:proofs}.
\subsection{Background \& Motivation}

\textbf{Supervised Classification.}  Assume a general function class\footnote{\eg softmax neural network classifiers in this work} $\mathcal{F}$ where each $f\in \mathcal{F}$ maps an input $\vx \in \sX$ to the probability simplex $\Delta^{K-1}$, \ie to a categorical distribution over $K$ classes $y\in \sY = \{1,2,\dots,K\}$. 
We seek $f^* \in \mathcal{F}$ that minimizes a risk $R_\mathcal{L}(f)=\mathbb{E}_\mathcal{D}[\mathcal{L}(\dve{y}, f(\vx))]$, for some loss function $\mathcal{L}$ and joint distribution $\mathcal{D}$ over $\sX \times \sY$, where $\dve{y}$ is a $K$-vector with one at index $y$ and zero elsewhere. In practice, $\mathcal{D}$ is unknown and, instead, we use $\mathcal{S} = \{ (\vx_i,y_i)\}_{i=1}^N$ which are independently sampled from $\mathcal{D}$ to minimize an empirical risk $\frac{1}{N}\sum_{i=1}^N \mathcal{L}(\dve{y_i},f(\vx_i))$. 

\textbf{Learning with Noisy Labels.} In this work, the goal is to learn from a noisy training distribution $\mathcal{D}_\eta$ where the labels are changed, with probability $\eta$, from their true distribution $\mathcal{D}$. The noise is called \textit{instance-dependent} if it depends on the input, \textit{asymmetric} if it dependents on the true label, and \textit{symmetric} if it is independent of both $\vx$ and $y$. 
Let $f^*_\eta$ be the optimizer of the noisy distribution risk $R^\eta_\mathcal{L}(f)$. A loss function $\mathcal{L}$ is then called \textit{robust} if $f^*_\eta$ also minimizes $R_\mathcal{L}$. The MAE loss~($\mathcal{L}_{MAE}(\dve{y}, f(\vx)) \coloneqq \Vert\dve{y}-f(\vx)\Vert_1 $) is robust but not CE~\citep{Ghosh_AAAI_2017_MAE}.

\textbf{Issue of Underfitting.} Several works propose such robust loss functions and demonstrate their efficacy in preventing noise fitting~\cite{Ghosh_AAAI_2017_MAE,Zhang_NeurIPS_2018_Generalized_CE,Wang_ICCV_2019_Symmetric_CE,Ma_ICML_2020_Normalized_Loss}. However, all those works have observed slow convergence of such robust loss functions leading to underfitting. This can be contrasted with CE that has fast convergence but overfits to noise. Ghosh~\etal~\cite{Ghosh_AAAI_2017_MAE} mentions slow convergence of MAE and GCE~\cite{Zhang_NeurIPS_2018_Generalized_CE} extensively analyzes the undefitting thereof. SCE~\cite{Wang_ICCV_2019_Symmetric_CE} reports similar problems for the reverse cross entropy and proposes a linear combination with CE. Finally, Ma~\etal~\cite{Ma_ICML_2020_Normalized_Loss} observe the same problem and consider a combination of ``active'' and ``passive'' loss functions.

\textbf{Consistency Regularization.} This encourages a network to have consistent predictions for different perturbations of the same image, which has mainly been used for semi-supervised learning~\citep{Oliver_arXiv_2018_Realistic_Eval_SSL}. 

\textbf{Motivation.} In Figure \ref{fig:consistency-observation}, we show the validation accuracy and a measure of consistency during training with the CE loss for varying amounts of noise. First, we note that training with CE loss eventually overfits to noisy labels. Figure~\ref{fig:consistency-observation-a}, indicates that the higher the noise rate, the more accuracy drop when it starts to overfit to noise. Figure~\ref{fig:consistency-observation}(b-c) shows the consistency of predictions for correct and noisy labeled examples of the training set, with the consistency measured as the ratio of examples that have the same class prediction for two perturbations of the same image, see Appendix~\ref{sup:sec:consistencyMeasure} for more details. A clear correlation is observed between the accuracy and consistency of the noisy examples. This suggests that maximizing consistency of predictions may improve the robustness to noise. Next, we define simple loss functions that (i) encourage consistency around data points and (ii) alleviate the ``issue of underfitting'' by interpolating between CE and MAE.
%
\begin{figure*}[t!]
     \begin{minipage}{0.35\textwidth}
     \centering
    \includegraphics[trim={0.25cm 0.25cm 0.25cm 0.25cm},clip,width=.95\columnwidth]{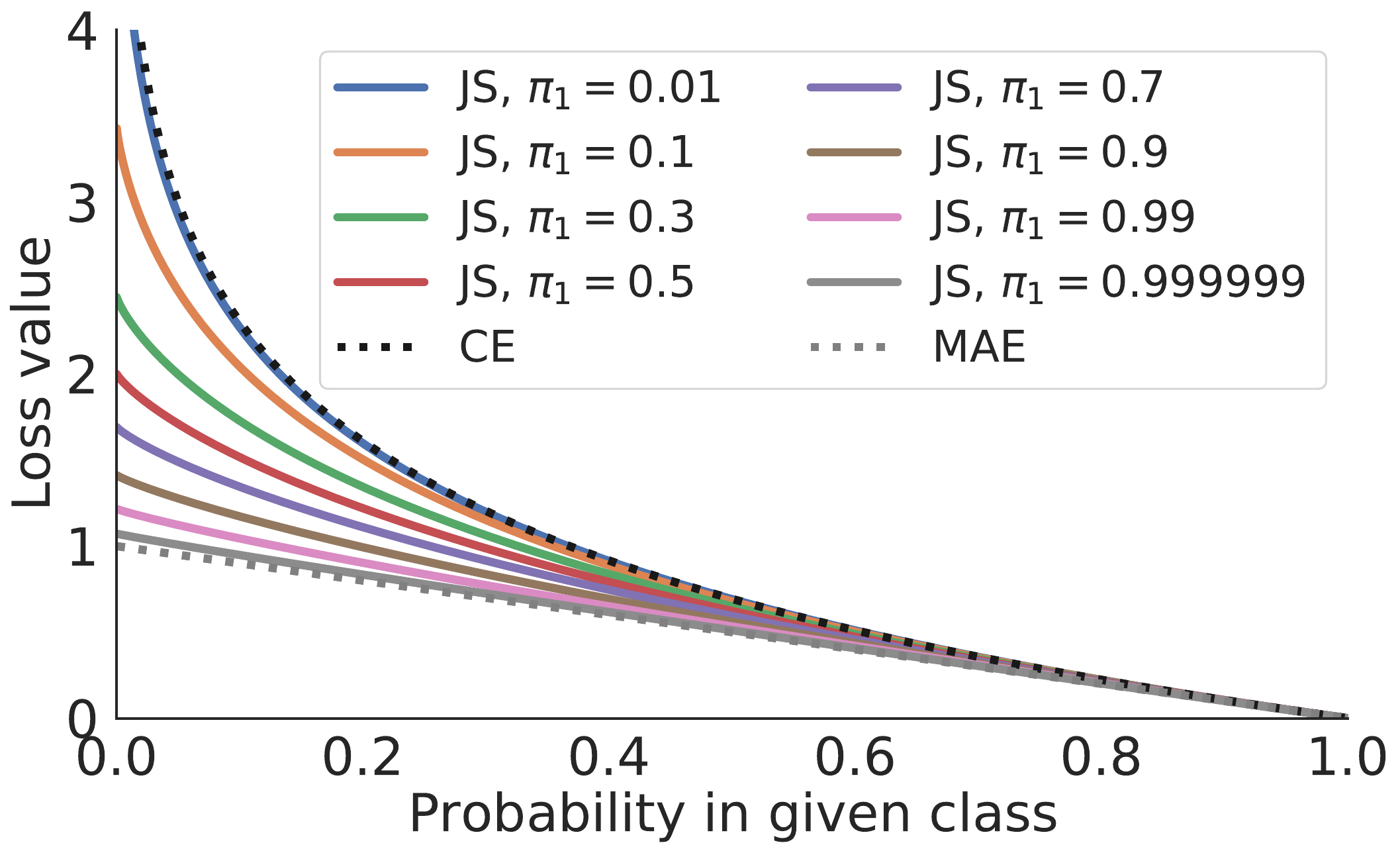}
    \caption{\textbf{JS loss generalizes CE and MAE.} The Jensen-Shannon loss~($\LJS$) for different values of the hyperparameter $\pi_1$. 
    The $\DJSnp$ loss interpolates between CE and MAE. 
    For low values of $\pi_1$, $\LJS$ behaves like CE and for increasing values of $\pi_1$ it behaves more like the noise robust MAE loss.}
    \label{fig:js-ce-mae}
         \end{minipage}
    \hfill
    \begin{minipage}{0.6\textwidth}
        \centering
         \begin{subfigure}[b]{0.32\textwidth}
     \captionsetup{skip=0pt} 
         \centering
         \includegraphics[trim={3.0cm 0cm 3.0cm 1.1cm},clip,width=\textwidth]{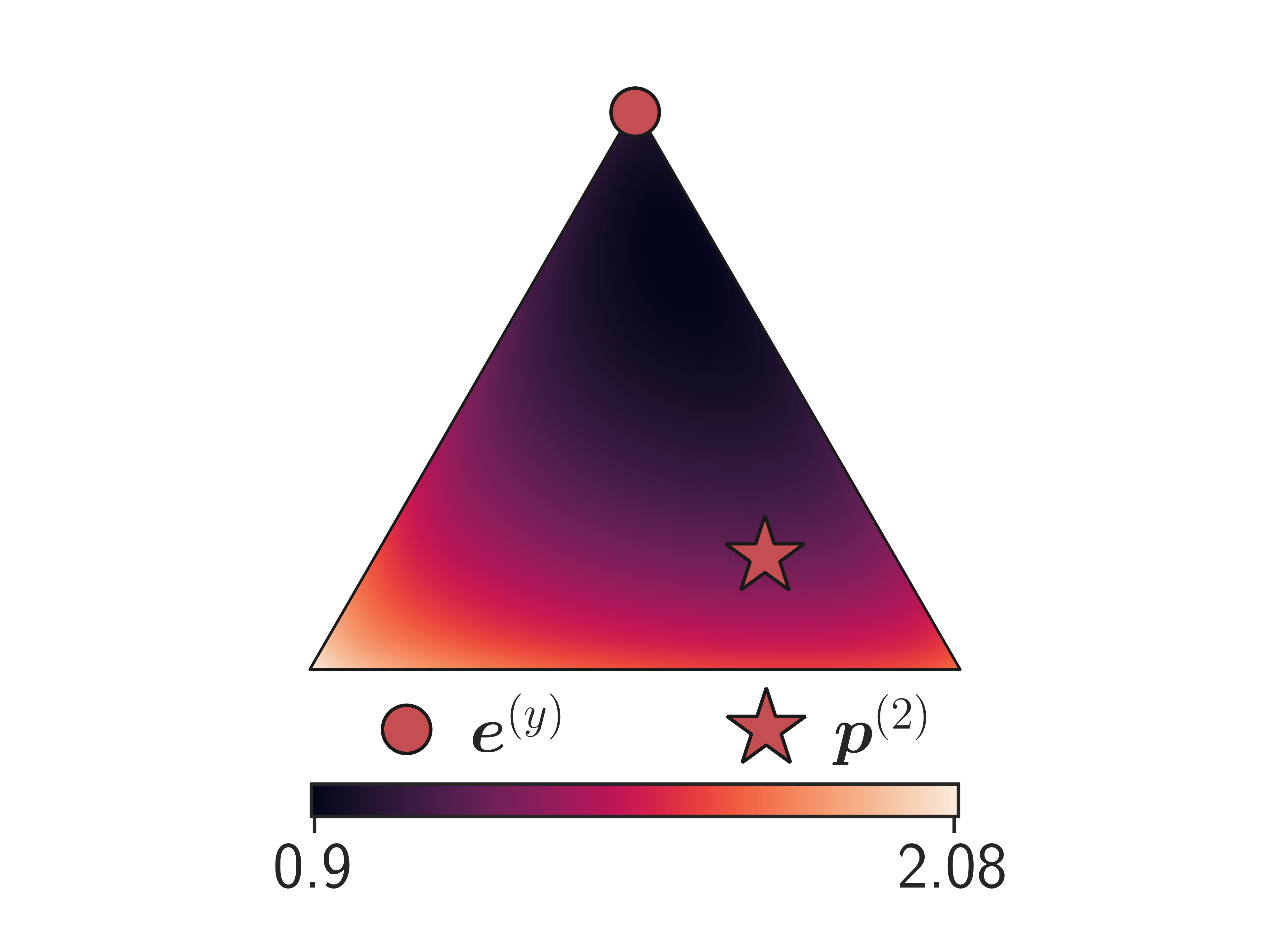}
         \label{fig:gjs-dissection-loss-GJS-full}
     \end{subfigure}
     \begin{subfigure}[b]{0.32\textwidth}
          \captionsetup{skip=0pt} 
         \centering
         \includegraphics[trim={3.0cm 0cm 3.0cm 1.1cm},clip,width=\textwidth]{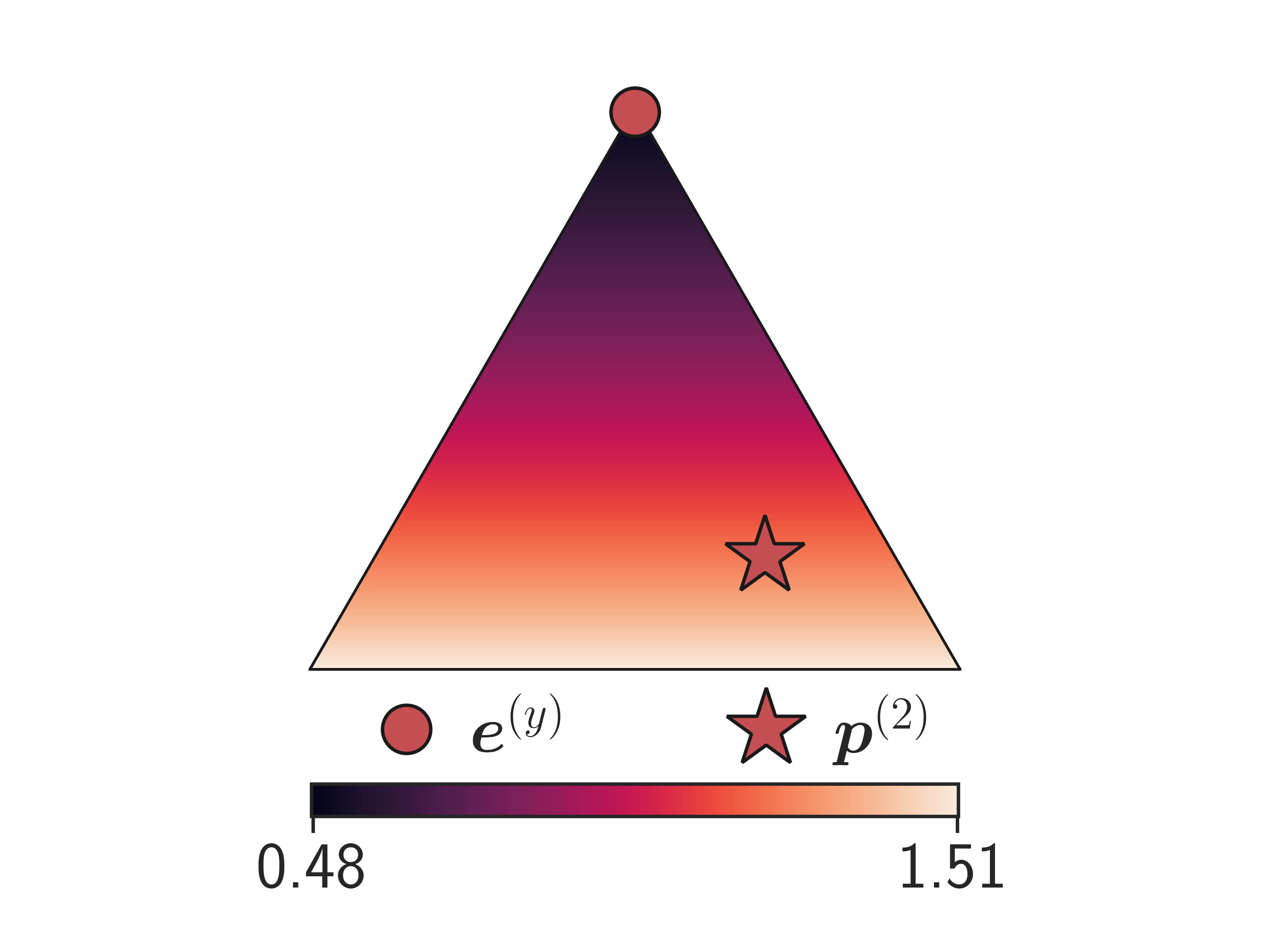}
         \label{fig:gjs-dissection-loss-JS}
     \end{subfigure}
     \begin{subfigure}[b]{0.32\textwidth}
         \captionsetup{skip=0pt} 
         \centering
         \includegraphics[trim={3.0cm 0cm 3.0cm 1.1cm},clip,width=\textwidth]{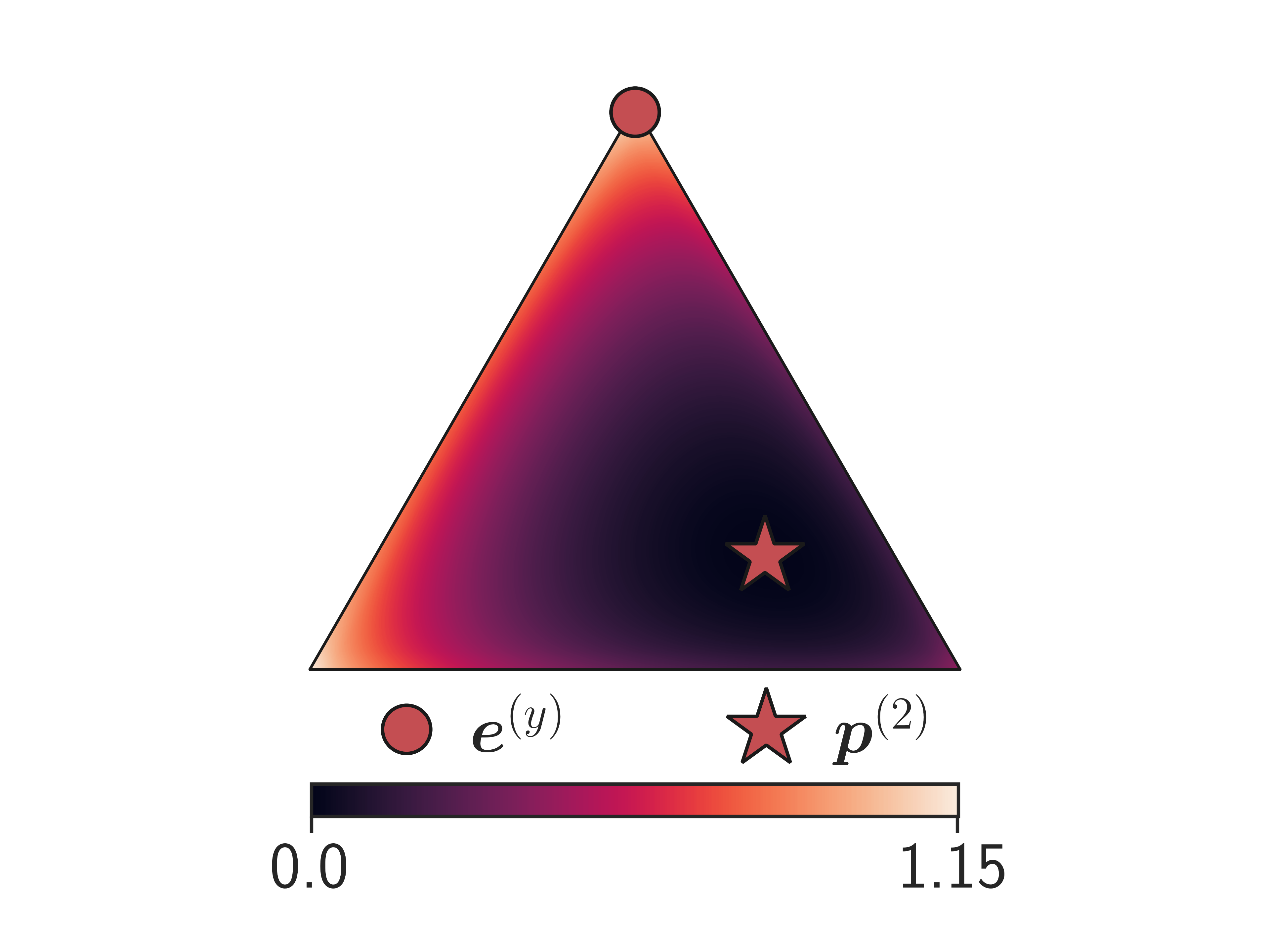}
         \label{fig:gjs-dissection-loss-consistency}
     \end{subfigure}
        \vspace{-0.55cm}
        \caption{\textbf{GJS Dissection for M=K=3:} The decomposition of $\LGJS$~(left) into a $\DJSnp$ term~(middle) and a consistency term~(right) from Proposition \ref{prop:GJS-JS-Consistency}. Each point in the simplex correspond to a $\dvp{3}\in \Delta^{2}$, where the color represents the value of the loss at that point. It can be seen that there are two ways to minimize $\LGJS$, either by making the predictions similar to the label~(middle) or similar to the other predictions~(right) to increase consistency. To better highlight the variations of the losses, each loss has its own range of values.\vspace{-0.1cm}}
    \label{fig:gjs-dissection-loss}
    \end{minipage}
\end{figure*}
\vspace{-0.1cm}
\subsection{Definitions}
\label{sec:gjs-def}

$\bm{\JSnp}$. Let $\dvp{1}, \dvp{2} \in \Delta^{K-1}$ have corresponding weights $\boldsymbol{\pi}=[\pi_1,\pi_2]^T\in\Delta$. Then, the Jensen-Shannon divergence between $\dvp{1}$ and $\dvp{2}$ is

\vspace{-0.2cm}
{
\begin{align}
    \hspace{-0.2cm}\JS(\dvp{1}, \dvp{2}) \hspace{-0.015cm} \coloneqq \hspace{-0.015cm} H(\vm) \hspace{-0.015cm} - \hspace{-0.015cm} \pi_1 H(\dvp{1}) \hspace{-0.015cm} - \hspace{-0.015cm} \pi_2H(\dvp{2}) \hspace{-0.015cm} = \hspace{-0.015cm}\evpi_1\KL(\dvp{1} \Vert \vm) \hspace{-0.015cm} + \hspace{-0.015cm} \evpi_2\KL(\dvp{2} \Vert \vm) \label{eq:JSasKL}
    \hspace{-0.1cm} 
\end{align}
}%
with $H$ the Shannon entropy, and $\vm=\pi_1 \dvp{1} + \pi_2\dvp{2}$. 
Unlike Kullback–Leibler divergence~($\DKL{\dvp{1}}{\dvp{2}}$) or cross entropy~(CE), JS is symmetric, bounded, does not require absolute continuity, and has a crucial weighting mechanism~($\boldsymbol{\pi}$), as we will see later.

$\bm{\GJSnp.}$ Similar to $\KL$, $\JSnp$ satisfies $\JS(\dvp{1}, \dvp{2}) \geq 0$, with equality iff $\dvp{1} = \dvp{2}$. For $\JSnp$, this is derived from Jensen's inequality for the concave Shannon entropy. This property holds for finite number of distributions and motivates a generalization of $\JSnp$ to multiple distributions~\citep{Lin_TIT_1991_JS_Divergence}:

\vspace{-0.4cm}
{
\begin{align}
    \hspace{-0.19cm}\GJS(\dvp{1}, \dots, \dvp{M}) \coloneqq H\Big(\sum_{i=1}^M \pi_i \dvp{i}\Big) - \s{i=1}{M}\evpi_i H(\dvp{i}) = \sum_{i=1}^M \pi_i\KL\Big(\dvp{i} \Big\Vert \sum_{j=1}^M \pi_j \dvp{j}\Big) \label{eq:GJSasKL}
\end{align}
}%
where $M$ is the number of distributions, and $\boldsymbol{\pi}=[\pi_1,\dots,\pi_M]^T\in \Delta^{M-1}$.

\textbf{Loss functions.} We aim to use $\JSnp$ and $\GJSnp$ divergences, to measure deviation of the predictive distribution(s), $f(\vx)$, from the target distribution, $\dve{y}$. Without loss of generality, hereafter, we dedicate $\dvp{1}$ to denote the target distribution. 
$\DJSnp$ loss, therefore, can take the form of $\JS(\dve{y}, f(\vx))$. 
Generalized $\DJSnp$ loss is a less straight-forward construction since $\GJSnp$ can accommodate more predictive distributions. While various choices can be made for these distributions, in this work, we consider predictions associated with different random perturbations of a sample, denoted by $\mathcal{A}(\vx)$. This choice, as shown later, implies an interesting analogy to \textit{consistency regularization}. 
The choice, also entails no distinction between the $M-1$ predictive distributions. Therefore, we consider $\pi_2=\dots=\pi_M=\frac{1-\pi_1}{M-1}$ in all our experiments. 
Finally, we scale the loss functions by a constant factor $Z=-(1-\pi_1)\log(1-\pi_1)$. As we will see later, the role of this scaling is merely to strengthen the already existing and desirable behaviors of these losses as $\pi_1$ approaches zero and one.
Formally, we have $\DJSnp$ and $\DGJSnp$ losses:

\vspace{-0.4cm}
{
\begin{align}
    \LJS(y,f,\vx) \coloneqq \frac{\JS(\dve{y}, f(\tilde{\vx}))}{Z}, 
    \quad
    \LGJS(y,f,\vx) \coloneqq \frac{\GJS(\dve{y}, f(\tilde{\vx}^{(2)}), \dots, f(\tilde{\vx}^{(M)}))}{Z}
    \label{eq:js-gjs-loss}
\end{align}
}%
with $\tilde{\vx}^{(i)}\sim \mathcal{A}(\vx)$. Next, we study the connection between $\DJSnp$ and losses which are based on the robustness theory of Ghosh~\etal~\citep{Ghosh_AAAI_2017_MAE}.

\subsection{JS's Connection to Robust Losses}
\label{sec:GJS-CE-MAE}
Cross Entropy (CE) is the prevalent loss function for deep classifiers with remarkable successes. However, CE is prone to fitting noise~\citep{Zhang17RethinkingGeneralization}. On the other hand, Mean Absolute Error (MAE) is theoretically noise-robust~\citep{Ghosh_AAAI_2017_MAE}. Evidently, standard optimization algorithms struggle to minimize MAE, especially for more challenging datasets \eg CIFAR-100~\citep{Zhang_NeurIPS_2018_Generalized_CE, Ma_ICML_2020_Normalized_Loss}. Therefore, there have been several proposals that combine CE and MAE, such as 
Generalized CE~(GCE)~\citep{Zhang_NeurIPS_2018_Generalized_CE}, Symmetric CE~(SCE)~\citep{Wang_ICCV_2019_Symmetric_CE}, and Normalized CE~(NCE+MAE)~\citep{Ma_ICML_2020_Normalized_Loss}.
The rationale is for CE to help with the learning dynamics of MAE. Next, we show $\DJSnp$ has CE and MAE as its asymptotes w.r.t. $\vpi_1$. 
\begin{restatable}{proposition}{propJsCeMae}
\label{prop:JS_limits_CE_MAE}
Let $\vp \in \Delta^{K-1}$, then
{
\begin{align}
     &\lim_{\pi_1 \rightarrow 0} \LJS(\dve{y}, \vp) = H(\dve{y}, \vp), \quad
     &\lim_{\pi_1 \rightarrow 1} \LJS(\dve{y}, \vp) = \frac{1}{2}\Vert\dve{y}-\vp\Vert_1 \nonumber
\end{align}
}%
where $H(\dve{y},\vp)$ is the cross entropy of $\dve{y}$ relative to $\vp$. 
\label{prop:js-ce-mae}
\end{restatable}
Figure \ref{fig:js-ce-mae} depicts how $\DJSnp$ interpolates between CE and MAE for $\pi_1 \in (0, 1)$.
The proposition reveals an interesting connection to state-of-the-art robust loss functions, however, there are important differences. SCE is not bounded~(so it cannot be used in Theorem \ref{T:sym}), and GCE is not symmetric, while $\DJSnp$ and MAE are both symmetric and bounded. In Appendix \ref{sup:sec:js-dissection}, we perform a dissection to better understand how these properties affect learning with noisy labels. GCE is most similar to $\DJSnp$ and is compared further in Appendix \ref{sup:js-gce-comp}.

A crucial difference to these other losses is that $\DJSnp$ naturally extends to multiple predictive distributions~($\DGJSnp$). Next, we show how $\DGJSnp$ generalizes $\DJSnp$ by incorporating consistency regularization.

\subsection{GJS's Connection to Consistency Regularization}
\label{sec:gjs-js-consistency}
In Figure \ref{fig:consistency-observation}, it was shown how the consistency of the noisy labeled examples was reduced when the network overfitted to noise. The following proposition shows how $\DGJSnp$ naturally encourages consistency in a single principled loss function. 
\begin{restatable}{proposition}{propGjsJsConsistency}
Let $\dvp{2},\dots,\dvp{M} \in \Delta^{K-1}$ with $M \geq 3$ and $\dvm = \frac{\sum_{j=2}^M \pi_j\dvp{j}}{1-\pi_1}$,  then 
{
\begin{align}
\LGJS(\dve{y}, \dvp{2},\dots,\dvp{M}) = \mathcal{L}_{\mathrm{JS_{\vpi'}}}(\dve{y},\dvm) + (1-\pi_1)\mathcal{L}_{\mathrm{GJS_{\vpi''}}}(\dvp{2},\dots,\dvp{M})
\nonumber
\end{align}
}%
where $\vpi'=[\pi_1, 1-\pi_1]^T$ and $\vpi''=\frac{[\pi_2, \dots, \pi_M]^T}{(1-\pi_1)}$.
\label{prop:GJS-JS-Consistency}
\end{restatable}
Importantly, Proposition \ref{prop:GJS-JS-Consistency} shows that $\DGJSnp$ can be decomposed into two terms: 1) a $\DJSnp$ term between the label and the mean prediction $\dvm$, and 2) a $\DGJSnp$ term, but without the label. Figure \ref{fig:gjs-dissection-loss} illustrates the effect of this decomposition. The first term, similarly to the standard $\DJSnp$ loss, encourages the predictions' mean to be closer to the label~(Figure \ref{fig:gjs-dissection-loss} middle). However, the second term encourages all predictions to be similar, that is, \textit{consistency regularization}~(Figure \ref{fig:gjs-dissection-loss} right). 

\subsection{Noise Robustness}
\label{sec:robustness}
Here, the robustness properties of $\DJSnp$ and $\DGJSnp$ are analyzed in terms of lower~($B_L$) and upper bounds~($B_U$) for the following theorem, which generalizes the results by Zhang~\etal~\citep{Zhang_NeurIPS_2018_Generalized_CE} to any bounded loss function, even with multiple predictive distributions.

\begin{restatable}{theorem}{tSym}
\label{T:sym}
Under symmetric noise with $\eta < \frac{K-1}{K}$, if  $B_L \leq \sum_{i=1}^K \mathcal{L}(\dve{i}, \vx, f) \leq B_U$, $\forall\vx,f$ is satisfied for a loss $\mathcal{L}$, then 
{
\begin{align}
    0 \leq R^\eta_{\mathcal{L}}(f^*) - R^\eta_{\mathcal{L}}(f^*_\eta) \leq \eta\frac{B_U-B_L}{K-1}, \quad \text{and} \quad 
    -\frac{\eta(B_U-B_L)}{K-1-\eta K} \leq R_{\mathcal{L}}(f^*) - R_{\mathcal{L}}(f^*_\eta) \leq 0, \nonumber
\end{align}
\vspace{-0.4cm}
}%
\end{restatable}
A tighter bound $B_U-B_L$, implies a smaller worst case risk difference of the optimal classifiers~(robust when $B_U=B_L$).  
Importantly, while $\mathcal{L}(\dve{i}, \vx, f)=\mathcal{L}(\dve{i}, f(\vx))$ usually, this subtle distinction is useful for losses with multiple predictive distributions, see Equation \ref{eq:js-gjs-loss}. In Theorem \ref{T:asym} in Appendix \ref{sup:noisetheorems}, we further prove the robustness of the proposed losses to asymmetric noise.

For losses with multiple predictive distributions, the bounds in Theorem \ref{T:sym} and \ref{T:asym} must hold for any $\vx$ and $f$, \ie, for any combination of $M-1$ categorical distributions on $K$ classes. Proposition \ref{prop:gjs-bounds} provides such bounds for $\DGJSnp$.
\begin{restatable}{proposition}{propGjsBounds}
$\DGJSnp$ loss with $M \leq K+1$ satisfies $B_L \leq \sum_{k=1}^K \LGJS(\dve{k}, \dvp{2}, \dots, \dvp{M}) \leq B_U$ for all $\dvp{2}, \dots, \dvp{M} \in \Delta^{K-1}$, with the following bounds
{
\vspace{-0.2cm}
\begin{align}
    B_L &= \sum_{k=1}^K \LGJS(\dve{k}, \vu, \dots, \vu), \quad B_U = \sum_{k=1}^K \LGJS(\dve{k}, \dve{1}, \dots, \dve{M-1}) \nonumber 
\end{align}
\vspace{-0.3cm}
}%
where $\vu \in \Delta^{K-1}$ is the uniform distribution.\label{prop:gjs-bounds}
\end{restatable}
Note the bounds for the $\DJSnp$ loss is a special case of Proposition \ref{prop:gjs-bounds} for $M=2$.

\begin{restatable}{remark}{rJSGJSRobustness}
\label{r:JSGJSRobustness}
$\LJS$ and $\LGJS$ are robust~($B_L=B_U$) in the limit of $\pi_1\to 1$.
\end{restatable}
Remark \ref{r:JSGJSRobustness} is intuitive from Section \ref{sec:GJS-CE-MAE} which showed that $\LJS$ is equivalent to the robust MAE in this limit and that the consistency term in Proposition \ref{prop:GJS-JS-Consistency} vanishes. 

In Proposition \ref{prop:gjs-bounds}, the lower bound~($B_L$) is the same for $\DJSnp$ and $\DGJSnp$. However, the upper bound~($B_U$) increases for more distributions, which makes $\DJSnp$ have a tighter bound than $\DGJSnp$ in Theorem \ref{T:sym} and \ref{T:asym}. In Proposition \ref{prop:js-gjs-same-risk-bound}, we show that $\DJSnp$ and $\DGJSnp$ have the same bound for the risk difference, given an assumption based on Figure \ref{fig:consistency-observation} that the optimal classifier on clean data ($f^*$) is at least as consistent as the optimal classifier on noisy data ($f^*_{\eta}$).
\begin{restatable}{proposition}{propGjsBetterBounds} \label{prop:js-gjs-same-risk-bound}
$\LJS$ and $\LGJS$ have the same risk bounds in Theorem \ref{T:sym} and \ref{T:asym} if 
$ \mathbb{E}_{ \mathbf{x} }[\mathcal{L}^{f^*}_{\mathrm{GJS_{\vpi''}}}(\dvp{2},\dots,\dvp{M})] \leq \mathbb{E}_{ \mathbf{x}}[\mathcal{L}^{f^*_\eta}_{\mathrm{GJS_{\vpi''}}}(\dvp{2},\dots,\dvp{M})]$, where $\mathcal{L}^f_{\mathrm{GJS_{\vpi''}}}(\dvp{2},\dots,\dvp{M})$ is the consistency term from Proposition \ref{prop:GJS-JS-Consistency}. 
\end{restatable}

\section{Related Works}
\label{sec:rel_works}
\vspace{-0.1cm}
Interleaved in the previous sections, we covered most-related works to us, \ie the avenue of identification or construction of \textit{theoretically-motivated robust loss functions}~\citep{Ghosh_AAAI_2017_MAE,Zhang_NeurIPS_2018_Generalized_CE,Wang_ICCV_2019_Symmetric_CE,Ma_ICML_2020_Normalized_Loss}. These works, similar to this paper, follow the theoretical construction of Ghosh~\etal~\citep{Ghosh_AAAI_2017_MAE}. Furthermore, Liu\&Guo~\cite{Liu_ICML_2020_Peer_Loss} use ``peer prediction'' to propose a new family of robust loss functions. Different to these works, here, we propose loss functions based on 
$\JSnp$ which holds various desirable properties of those prior works while exhibiting novel ties to consistency regularization; a recent important regularization technique.

Next, we briefly cover other lines of work. A more thorough version can be found in Appendix \ref{sup:rel-works}.

A direction, that similar to us does not alter training, \textit{reweights a loss function} by confusion matrix~\citep{Natarajan_NIPS_2013,Sukhbaatar_ICLR_2015_confusion_matrix,Patrini_CVPR_2017,Han_NeurIPS_2018,Xia_NeurIPS_2019}. Assuming a class-conditional noise model, loss correction is theoretically motivated and perfectly orthogonal to noise-robust losses.

\textit{Consistency regularization} is a recent technique that imposes smoothness in the learnt function for semi-supervised learning~\citep{Oliver_arXiv_2018_Realistic_Eval_SSL} and recently for noisy data~\citep{Li_ICLR_2020_dividemix}. These works use different complex pipelines for such regularization. $\DGJSnp$ encourages consistency in a simple way that exhibits other desirable properties for learning with noisy labels. 
Importantly, Jensen-Shannon-based consistency loss functions have been used to improve test-time robustness to image corruptions~\citep{hendrycks2020augmix} and adversarial examples~\citep{tack2021consistency}, which further verifies the general usefulness of $\DGJSnp$.
In this work, we study such loss functions for a different goal:  \textit{training-time label-noise robustness}. In this context, our thorough analytical and empirical results are, to the best of our knowledge, novel.

Recently, loss functions with \textit{information-theoretic} motivations have been proposed~\citep{Xu_NeurIPS_2019_Information_Theoretic_Mutual_Info_Loss,Wei_ICLR_2021_f_Divergence}. $\DJSnp$, with an apparent information-theoretic interpretation, has a strong connection to those. Especially, the latter is a close concurrent work studying JS and other divergences from the family of f-divergences~\citep{csiszar1967fdiv}. However, in this work, we consider a generalization to more than two distributions and study the role of $\pi_1$, which they treat as a constant~$(\pi_1=\frac{1}{2}$). These differences lead to improved performance and novel theoretical results, e.g., Proposition \ref{prop:JS_limits_CE_MAE} and \ref{prop:GJS-JS-Consistency}. Lastly, another generalization of JS was recently presented by Nielsen~\citep{nielsen2019JSMeans}, where the arithmetic mean is generalized to abstract means.
\vspace{-0.1cm}

\renewcommand{\arraystretch}{1.2}
\begin{table*}[t!]
\tiny 
\caption{\label{tab:cifar} \textbf{Synthetic Noise Benchmark on CIFAR.} We \textit{reimplement} other noise-robust loss functions into the \textit{same learning setup} and ResNet-34, including label smoothing~(LS), Bootstrap~(BS), Symmetric CE~(SCE), Generalized CE~(GCE), and Normalized CE~(NCE+RCE). We used \textit{same hyperparameter optimization budget and mechanism} for all the prior works and ours. Mean test accuracy and standard deviation are reported from five runs and the statistically-significant top performers are boldfaced. The thorough analysis is evident from the higher performance of CE in our setup compared to prior works. $\DGJSnp$ achieves state-of-the-art results for different noise rates, types, and datasets. Generally, $\DGJSnp$'s efficacy is more evident for the more challenging CIFAR-100 dataset.
}
\begin{center}
\tabcolsep=0.22cm 
\begin{tabular}{ @{}l l c c c c c c c} 
 \toprule
 \multirow{2}{*}{Dataset} & \multirow{2}{*}{Method} & No Noise & \multicolumn{4}{c}{Symmetric Noise Rate} & \multicolumn{2}{c}{Asymmetric Noise Rate} \\ \cmidrule(lr){3-3}\cmidrule(lr){4-7} \cmidrule(lr){8-9}
 & & 0\% & 20\% & 40\% & 60\% & 80\% & 20\% & 40\% \\
 \midrule
 \multirow{8}{5em}{CIFAR-10} & CE & \textbf{95.77 $\pm$ 0.11} & 91.63 $\pm$ 0.27 & 87.74 $\pm$ 0.46 & 81.99 $\pm$ 0.56 & 66.51 $\pm$ 1.49 & 92.77 $\pm$ 0.24 & 87.12 $\pm$ 1.21 \\
 & BS & 94.58 $\pm$ 0.25 & 91.68 $\pm$ 0.32 & 89.23 $\pm$ 0.16 & 82.65 $\pm$ 0.57 & 16.97 $\pm$ 6.36 & 93.06 $\pm$ 0.25 & 88.87 $\pm$ 1.06 \\
 & LS & 95.64 $\pm$ 0.12 & 93.51 $\pm$ 0.20 & 89.90 $\pm$ 0.20 & 83.96 $\pm$ 0.58 & 67.35 $\pm$ 2.71 & 92.94 $\pm$ 0.17 & 88.10 $\pm$ 0.50 \\
 & SCE & \textbf{95.75 $\pm$ 0.16} & 94.29 $\pm$ 0.14 & 92.72 $\pm$ 0.25 & 89.26 $\pm$ 0.37 & \textbf{80.68 $\pm$ 0.42} & 93.48 $\pm$ 0.31 & 84.98 $\pm$ 0.76 \\
& GCE & \textbf{95.75 $\pm$ 0.14} & 94.24 $\pm$ 0.18 & 92.82 $\pm$ 0.11 & 89.37 $\pm$ 0.27 & \textbf{79.19 $\pm$ 2.04} & 92.83 $\pm$ 0.36 & 87.00 $\pm$ 0.99 \\
 & NCE+RCE & 95.36 $\pm$ 0.09 & 94.27 $\pm$ 0.18 & 92.03 $\pm$ 0.31 & 87.30 $\pm$ 0.35 & 77.89 $\pm$ 0.61 & \textbf{93.87 $\pm$ 0.03} & 86.83 $\pm$ 0.84 \\
  & JS & \textbf{95.89 $\pm$ 0.10} & 94.52 $\pm$ 0.21 & 93.01 $\pm$ 0.22 & 89.64 $\pm$ 0.15 & 76.06 $\pm$ 0.85 & 92.18 $\pm$ 0.31 & 87.99 $\pm$ 0.55  \\
 & GJS & \textbf{95.91 $\pm$ 0.09} & \textbf{95.33 $\pm$ 0.18} & \textbf{93.57 $\pm$ 0.16} & \textbf{91.64 $\pm$ 0.22} & 79.11 $\pm$ 0.31 & \textbf{93.94 $\pm$ 0.25} & \textbf{89.65 $\pm$ 0.37} \\
\midrule
 \multirow{8}{5em}{CIFAR-100} & CE & 77.60 $\pm$ 0.17 & 65.74 $\pm$ 0.22 & 55.77 $\pm$ 0.83 & 44.42 $\pm$ 0.84 & 10.74 $\pm$ 4.08 & 66.85 $\pm$ 0.32 & 49.45 $\pm$ 0.37  \\
 & BS & 77.65 $\pm$ 0.29 & 72.92 $\pm$ 0.50 & 68.52 $\pm$ 0.54 & 53.80 $\pm$ 1.76 & 13.83 $\pm$ 4.41 & 73.79 $\pm$ 0.43 & \textbf{64.67 $\pm$ 0.69} \\
 & LS & 78.60 $\pm$ 0.04 & 74.88 $\pm$ 0.15 & 68.41 $\pm$ 0.20 & 54.58 $\pm$ 0.47 & 26.98 $\pm$ 1.07 & 73.17 $\pm$ 0.46 & 57.20 $\pm$ 0.85 \\
 & SCE & 78.29 $\pm$ 0.24 & 74.21 $\pm$ 0.37 & 68.23 $\pm$ 0.29 & 59.28 $\pm$ 0.58 & 26.80 $\pm$ 1.11 & 70.86 $\pm$ 0.44 & 51.12 $\pm$ 0.37 \\
 & GCE & 77.65 $\pm$ 0.17 & 75.02 $\pm$ 0.24 & 71.54 $\pm$ 0.39 & 65.21 $\pm$ 0.16 & \textbf{49.68 $\pm$ 0.84} & 72.13 $\pm$ 0.39 & 51.50 $\pm$ 0.71 \\
 & NCE+RCE & 74.66 $\pm$ 0.21 & 72.39 $\pm$ 0.24 & 68.79 $\pm$ 0.29 & 62.18 $\pm$ 0.35 & 31.63 $\pm$ 3.59 & 71.35 $\pm$ 0.16 & 57.80 $\pm$ 0.52 \\
 & JS & 77.95 $\pm$ 0.39 & 75.41 $\pm$ 0.28 & 71.12 $\pm$ 0.30 & 64.36 $\pm$ 0.34 & 45.05 $\pm$ 0.93 & 71.70 $\pm$ 0.36 & 49.36 $\pm$ 0.25 \\
 & GJS & \textbf{79.27 $\pm$ 0.29} & \textbf{78.05 $\pm$ 0.25} & \textbf{75.71 $\pm$ 0.25} & \textbf{70.15 $\pm$ 0.30} & 44.49 $\pm$ 0.53 & \textbf{74.60 $\pm$ 0.47} & 63.70 $\pm$ 0.22 \\
 \bottomrule
\end{tabular}
\end{center}
\vspace{-0.3cm}
\end{table*}
\section{Experiments}
\label{sec:results}
\vspace{-0.1cm}
This section, first, empirically investigates the effectiveness of the proposed losses for learning with noisy labels on synthetic~(Section \ref{sec:synthetic-noise}) and real-world noise~(Section \ref{sec:real-noise}). This is followed by several experiments and ablation studies~(Section \ref{sec:ablation}) to shed light on the properties of $\DJSnp$ and $\DGJSnp$ through empirical substantiation of the theories and claims provided in Section \ref{sec:GJS}.
All these additional experiments are done on the more challenging CIFAR-100 dataset. 

\textbf{Experimental Setup.} We use ResNet 34 and 50 for experiments on CIFAR and WebVision datasets respectively and optimize them using SGD with momentum.  
The complete details of the training setup can be found in Appendix \ref{sup:training-details}. Most importantly, we take three main measures to ensure a fair and reliable comparison throughout the experiments: 1) we reimplement all the loss functions we compare with in a single shared learning setup, 2) we use the same hyperparameter optimization budget and mechanism for all the prior works and ours, and 3) we train and evaluate five networks for individual results, where in each run the synthetic noise, network initialization, and data-order are differently randomized. The thorough  analysis is evident from the higher performance of CE in our setup compared to prior works. Where possible, we report mean and standard deviation and denote the statistically-significant top performers with student t-test.
\subsection{Synthetic Noise Benchmarks: CIFAR}
\label{sec:synthetic-noise}
Here, we evaluate the proposed loss functions on the CIFAR datasets with two types of synthetic noise: symmetric and asymmetric. For symmetric noise, the labels are, with probability $\eta$, re-sampled from a uniform distribution over all labels. For asymmetric noise, we follow the standard setup of Patrini~\etal~\citep{patrini2017making}. For CIFAR-10, the labels are modified, with probability $\eta$, as follows: \textit{truck}~$\to$~\textit{automobile}, \textit{bird}~$\to$~\textit{airplane}, \textit{cat}~$\leftrightarrow$~\textit{dog}, and \textit{deer}~$\to$~\textit{horse}. For CIFAR-100, labels are, with probability $\eta$, cycled to the next sub-class of the same ``super-class'', \eg the labels of super-class ``vehicles 1'' are modified as follows: \textit{bicycle}~$\rightarrow$~$\textit{bus}$~$\rightarrow$~\textit{motorcycle}~$\rightarrow$~\textit{pickup truck}~$\rightarrow$~\textit{train}~$\rightarrow$~\textit{bicycle}.

We compare with other noise-robust loss functions such as label smoothing~(LS)~\citep{Lukasik_ICML_2020_label_smoothing_label_noisy}, Bootstrap~(BS)~\citep{Reed_arXiv_2014_bootstrapping}, Symmetric Cross-Entropy~(SCE)~\citep{Wang_ICCV_2019_Symmetric_CE}, Generalized Cross-Entropy~(GCE)~\citep{Zhang_NeurIPS_2018_Generalized_CE}, and the NCE+RCE loss of Ma~\etal~\citep{Ma_ICML_2020_Normalized_Loss}. Here, we do not compare to methods that propose a full pipeline since, first, a conclusive comparison would require re-implementation and individual evaluation of several components and second, robust loss functions can be considered orthogonal to them.

\textbf{Results.}
Table \ref{tab:cifar} shows the results for
symmetric and asymmetric noise on CIFAR-10 and CIFAR-100. 
$\DGJSnp$ performs similarly or better than other methods for different noise rates, noise types, and data sets. Generally, $\DGJSnp$'s efficacy is more evident for the more challenging CIFAR-100 dataset. For example, on 60\% uniform noise on CIFAR-100, the difference between $\DGJSnp$ and the second best~(GCE) is 4.94 percentage points, while our results on 80\% noise is lower than GCE. We attribute this to the high sensitivity of the results to the hyperparameter settings in such a high-noise rate which are also generally unrealistic (WebVision has $\sim$20\%).
The performance of $\DJSnp$ is consistently similar to the top performance of the prior works across different noise rates, types and datasets. 
In Section~\ref{sec:ablation}, we substantiate the importance of the consistency term, identified in Proposition~\ref{prop:GJS-JS-Consistency}, when going from $\DJSnp$ to $\DGJSnp$ that helps with the learning dynamics and reduce the susceptibility to noise. In Appendix \ref{sup:cifar-instance}, we provide results for $\DGJSnp$ on instance-dependent synthetic noise~\citep{zhang2021learning}.  Next, we test the proposed losses on a naturally-noisy dataset to see their efficacy in a real-world scenario.
\begin{table}
\scriptsize
\tabcolsep=0.095cm 
\caption{\label{tab:webvision}\textbf{Real-world Noise Benchmark on WebVision.} Mean test accuracy and standard deviation from five runs are reported for the validation sets of (mini) WebVision and ILSVRC12. GJS with two networks correspond to the mean prediction of two independently trained GJS networks with different seeds for data augmentation and weight initialization. Here, $\DGJSnp$ uses $Z=1$. Results marked with $\dagger$ are from Zheltonozhskii~\etal~\cite{zheltonozhskii2021contrast}.\vspace{0.2cm}}
\centering
\begin{tabular}{lcccccccc}\toprule
\multirow{2}{4em}{Method} & \multirow{2}{5em}{Architecture} &  \multirow{2}{5em}{Augmentation} & \multirow{2}{5em}{~~Networks} & \multicolumn{2}{c}{WebVision} & \multicolumn{2}{c}{ILSVRC12}
\\\cmidrule(lr){5-6}\cmidrule(lr){7-8}
        & & &  & Top 1  & Top 5     & Top 1  & Top 5 \\\midrule
ELR+~\citep{liu2020earlylearning}{$\dagger$} & Inception-ResNet-V2 & Mixup & 2  & 77.78 & $\bm{91.68}$ & 70.29 & 89.76 \\
DivideMix~\citep{Li_ICLR_2020_dividemix}{$\dagger$} & Inception-ResNet-V2 & Mixup & 2  & 77.32 & $\bm{91.64}$ & $\bm{75.20}$  & 90.84  \\
DivideMix~\citep{Li_ICLR_2020_dividemix}{$\dagger$} & ResNet-50 & Mixup & 2  & 76.32 $\pm$ 0.36 & 90.65 $\pm$ 0.16 & 74.42 $\pm$ 0.29 & $\bm{91.21 \pm 0.12}$ \\
 \midrule
CE & ResNet-50 & ColorJitter & 1 & 70.69 $\pm$ 0.66 & 88.64 $\pm$ 0.17 & 67.32 $\pm$ 0.57 & 88.00 $\pm$ 0.49 \\
 JS & ResNet-50 & ColorJitter & 1 & 74.56 $\pm$ 0.32 & 91.09 $\pm$ 0.08 & 70.36 $\pm$ 0.12 & 90.60 $\pm$ 0.09  \\
GJS & ResNet-50 & ColorJitter & 1  & 77.99 $\pm$ 0.35 & 90.62 $\pm$ 0.28 & 74.33 $\pm$ 0.46 & 90.33 $\pm$ 0.20 \\
GJS & ResNet-50 & ColorJitter & 2  & $\bm{79.28 \pm 0.24}$ & 91.22 $\pm$ 0.30 & $\bm{75.50 \pm 0.17}$ & $\bm{91.27 \pm 0.26 }$\\ \bottomrule
\end{tabular}
\vspace{-0.3cm}
\end{table}
\subsection{Real-World Noise Benchmark: WebVision}
\label{sec:real-noise}
WebVision v1 is a large-scale image dataset collected by crawling Flickr and Google, which resulted in an estimated  20\% of noisy labels~\citep{li2017webvision}. There are 2.4 million images of the same thousand classes as ILSVRC12. Here, we use a smaller version called mini WebVision~\citep{Jiang18MentorNet} consisting of the first 50 classes of the Google subset. We compare CE, $\DJSnp$, and $\DGJSnp$ on WebVision following the same rigorous procedure as for the synthetic noise. However, upon request by the reviewers, we also compare with the reported results of some state-of-the-art elaborate techniques. This comparison deviates from our otherwise systematic analysis. 

\textbf{Results.}
Table \ref{tab:webvision}, as the common practice, reports the performances on the validation sets of WebVision and ILSVRC12~(first 50 classes). Both $\DJSnp$ and $\DGJSnp$ exhibit large margins with standard CE, especially for top-1 accuracy. Top-5 accuracy, due to its admissibility of wrong top predictions, can obscure the susceptibility to noise-fitting and thus indicates smaller but still significant improvements. 

The two state-of-the-art methods on this dataset were DivideMix~\citep{Li_ICLR_2020_dividemix} and ELR+~\citep{liu2020earlylearning}. Compared to our setup, both these methods use a stronger network~(Inception-ResNet-V2 vs ResNet-50), stronger augmentations~(Mixup vs color jittering) and co-train two networks. Furthermore, ELR+ uses an exponential moving average of weights and DivideMix treats clean and noisy labeled examples differently after separating them using Gaussian mixture models. Despite these differences, $\DGJSnp$ performs as good or better in terms of top-1 accuracy on WebVision and significantly outperforms ELR+ on ILSVRC12~(70.29 vs 74.33). 
The importance of these differences becomes apparent as 1) the top-1 accuracy for DivideMix degrades when using ResNet-50, and 2) the performance of $\DGJSnp$ improves by adding one of their components, \ie the use of two networks.
We train an ensemble of two independent networks with the $\DGJSnp$ loss and average their predictions~(last row of Table \ref{tab:webvision}). This simple extension, which requires no change in the training code, gives significant improvements. To the best of our knowledge, this is the highest reported top-1 accuracy on WebVision and ILSVRC12 when no pre-training is used.

In Appendix \ref{sup:real-world-noise}, we show state-of-the-art results when using $\DGJSnp$  on two other real-world noisy datasets: ANIMAL-10N~\citep{song2019selfie} and Food-101N~\citep{lee2017cleannet}. 

So far, the experiments demonstrated the robustness of the proposed loss function~(regarding Proposition~\ref{prop:gjs-bounds}) via the significant improvement of the final accuracy on noisy datasets. While this was central and informative, it is also important to investigate whether this improvement comes from the theoretical properties that were argued for $\DJSnp$ and $\DGJSnp$. In what follows, we devise several such experiments, in an effort to substantiate the theoretical claims and conjectures.
%
%
%
%
\begin{figure}[t!]
    \captionsetup[subfigure]{aboveskip=-1pt,belowskip=-1pt}
    \centering
    \begin{minipage}{0.48\textwidth}
    
     \centering
     \subcaptionbox[b]{$\DJSnp, \eta=20\%$}
    {\includegraphics[width=0.49\textwidth]{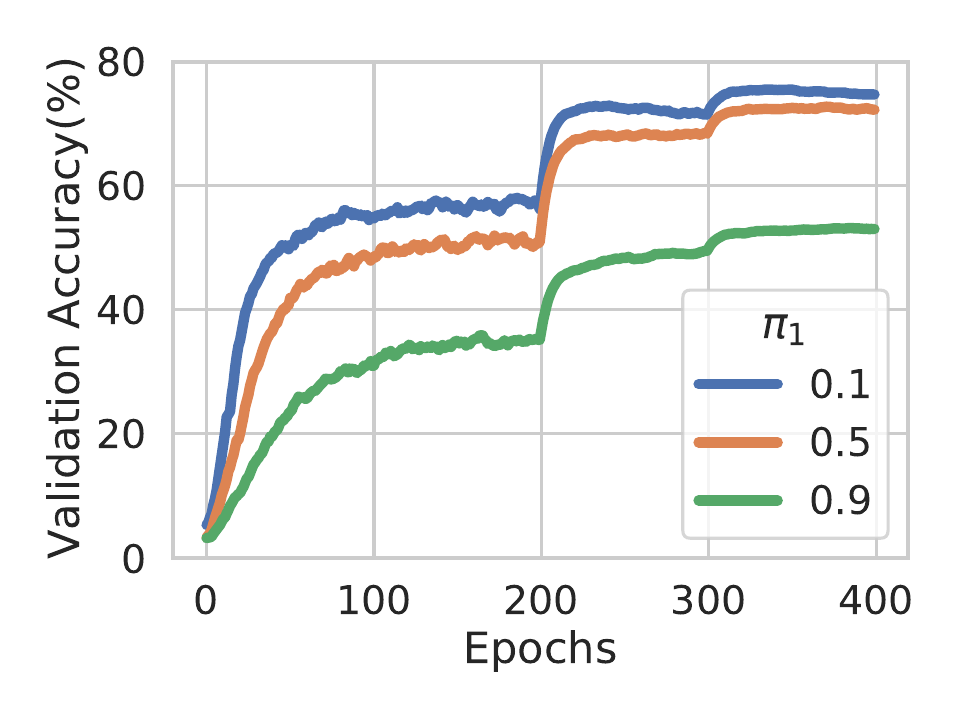}}
      \subcaptionbox[b]{$\DJSnp, \eta=60\%$}
    {\includegraphics[width=0.49\textwidth]{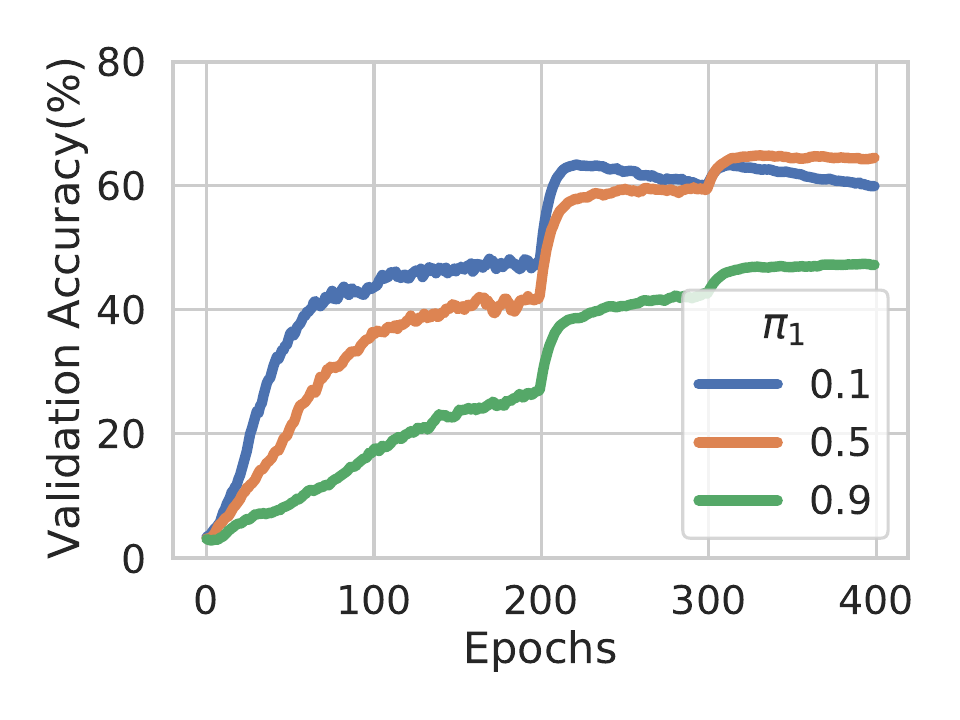}}
   \subcaptionbox[b]{$\DGJSnp, \eta=20\%$}
    {\includegraphics[width=0.49\textwidth]{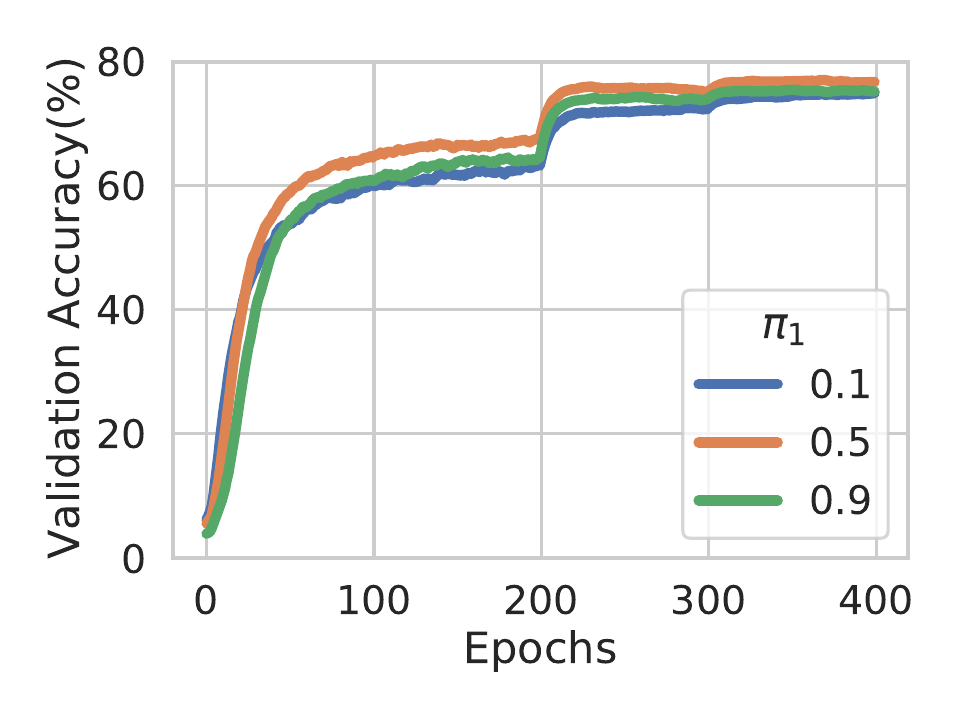}}
    \subcaptionbox[b]{$\DGJSnp, \eta=60\%$}
      {\includegraphics[width=0.49\textwidth]{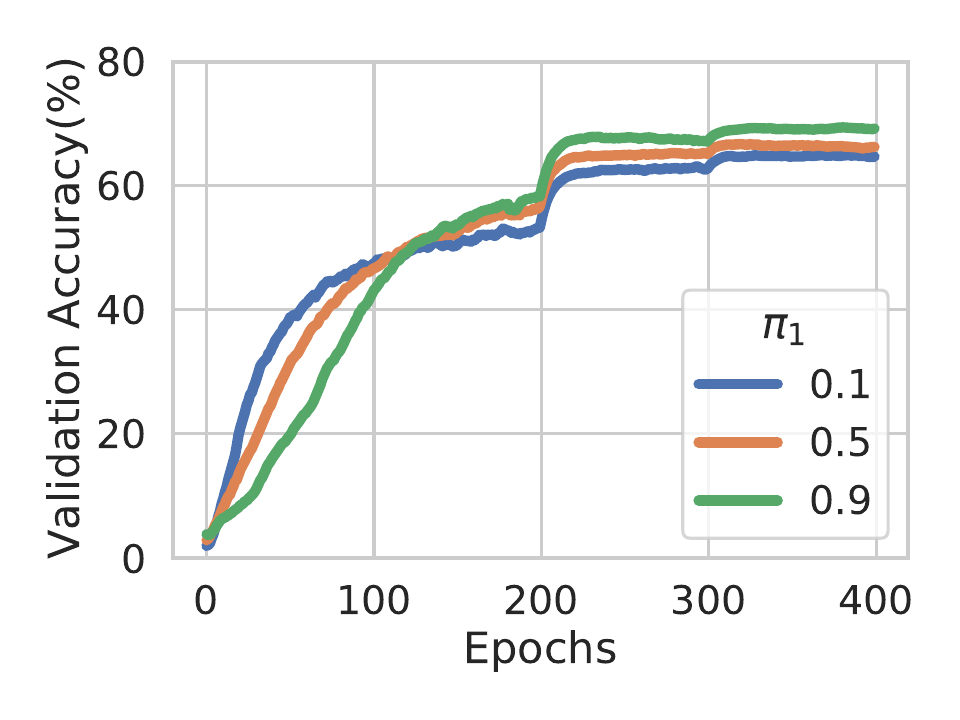}}

        \caption{\textbf{Effect of $\boldsymbol{\pi_1}$.} Validation accuracy of $\DJSnp$ and $\DGJSnp$ during training with symmetric noise on CIFAR100. From Proposition~\ref{prop:JS_limits_CE_MAE}, $\DJSnp$ behaves like CE and MAE for low and high values of $\pi_1$, respectively. The signs of noise-fitting for $\pi_1=0.1$ on 60\% noise~(b), and slow learning of $\pi_1=0.9$~(a-b), show this in practice. The $\DGJSnp$ loss does not exhibit overfitting for low values of $\pi_1$ and learns quickly for large values of $\pi_1$~(c-d). \vspace{-0.7cm}
        }
    \label{fig:piAblation}
    \end{minipage}
    \hfill
    \begin{minipage}{0.48\textwidth}
        \centering
        {\includegraphics[trim={0cm 0cm 1.5cm 1.3cm},clip,width=1.0\textwidth]{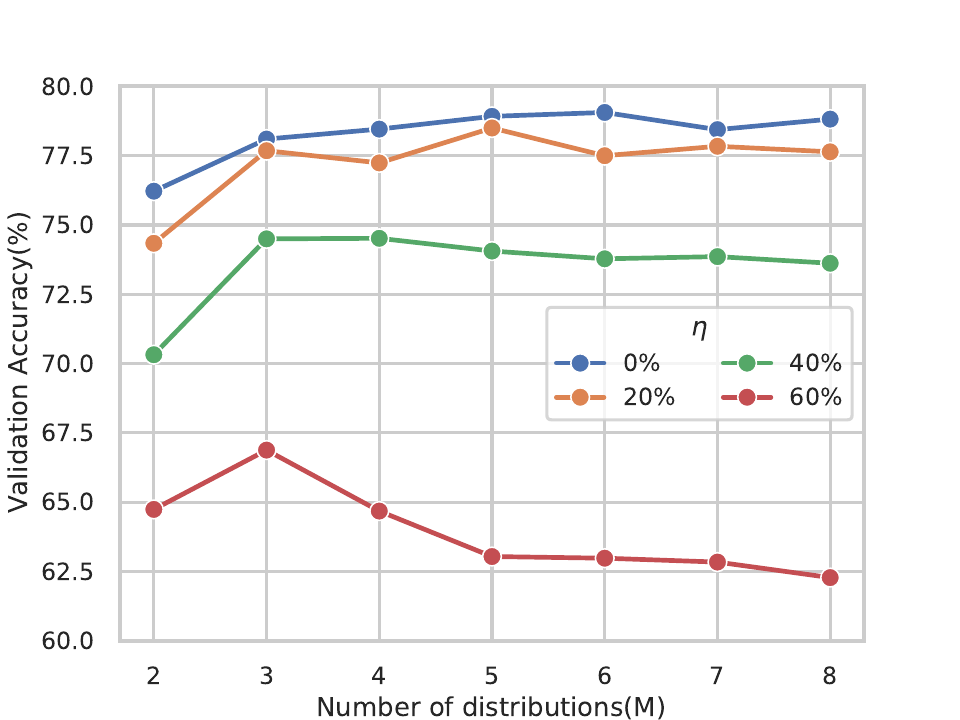}}
        \caption{\textbf{Effect of $\mathbf{M}$.} Validation accuracy for increasing number of distributions~($M$) and different symmetric noise rates on CIFAR-100 with $\pi_1=\frac{1}{2}$.
        For all noise rates, using three instead of two distributions results in a higher accuracy. Going beyond three distributions is only helpful for lower noise rates. For simplicity we use $M=3$ (corresponding to two augmentations) for all of our experiments. \vspace{-0.9cm}
        }
        \label{fig:MAblation}
    \end{minipage}
    \vspace{-0.23cm}
\end{figure}

\subsection{Towards a Better Understanding of the Jensen-Shannon-based Loss Functions}
\vspace{-0.14cm}
\label{sec:ablation}
Here, we study the behavior of the losses for different distribution weights $\pi_1$, number of distributions $M$, and epochs. We also provide insights on why $\DGJSnp$ performs better than $\DJSnp$.

\textbf{How does} $\bm{\pi_1}$ \textbf{control the trade-off of robustness and learnability?} In Figure \ref{fig:piAblation}, we plot the validation accuracy during training for both $\DJSnp$ and $\DGJSnp$ at different values of $\pi_1$ and noise rates $\eta$. 
From Proposition~\ref{prop:JS_limits_CE_MAE}, we expect $\DJSnp$ to behave as CE for low values of $\pi_1$ and as MAE for larger values of $\pi_1$. Figure \ref{fig:piAblation} (a-b) confirms this. Specifically, $\pi_1=0.1$ learns quickly and performs well for low noise but overfits for $\eta=0.6$ (characteristic of non-robust CE), on the other hand, $\pi_1=0.9$ learns slowly but is robust to high noise rates (characteristic of noise-robust MAE).

In Figure \ref{fig:piAblation} (c-d), we observe three qualitative improvements of $\DGJSnp$ over $\DJSnp$: 1) no signs of overfitting to noise for large noise rates with low values of $\pi_1$, 2) better learning dynamics for large values of $\pi_1$ that otherwise learns slowly, and 3) converges to a higher validation accuracy.

\textbf{How many distributions to use?} Figure \ref{fig:MAblation} depicts validation accuracy for varying number of distributions $M$. For all noise rates, we observe a performance increase going from $M=2$ to $M=3$. However, the performance of $M>3$ depends on the noise rate. For lower noise rates, having more than three distributions can improve the performance. For higher noise rates \eg $60\%$, having $M>3$ degrades the performance. We hypothesise this is due to: 1) at high noise rates, there are only a few correctly labeled examples that can help guide the learning, and 2) going from $M=2$ to $M=3$ adds a consistency term, while $M>3$ increases the importance of the consistency term in Proposition \ref{prop:GJS-JS-Consistency}. Therefore, for a large enough M, the loss will find it easier to keep the consistency term low (keep predictions close to uniform as at the initialization), instead of generalizing based on the few clean examples.
For simplicity, we have used $M=3$ for all experiments with $\DGJSnp$. 

\textbf{Is the improvements of GJS over JS due to mean prediction or consistency?}
Proposition \ref{prop:GJS-JS-Consistency} decomposed $\DGJSnp$ into a $\DJSnp$ term with a mean prediction~($\dvm$) and a consistency term operating on all distributions but the target. In Table \ref{tab:gjs-wo-consistency}, we compare the performance of $\DJSnp$ and $\DGJSnp$ to $\DGJSnp$ without the consistency term, i.e., $\mathcal{L}_{\mathrm{JS_{\vpi'}}}(\dve{y},\dvm)$.
The results suggest that the improvement of $\DGJSnp$ over $\DJSnp$ can be attributed to the consistency term. 

Figure \ref{fig:piAblation} (a-b) showed that $\DJSnp$ improves the learning dynamics of MAE by blending it with CE, controlled by $\pi_1$. Similarly, we see here that the consistency term also improves the learning dynamics~(underfitting and convergence speed) of MAE. Interestingly, Figure~\ref{fig:piAblation} (c-d), shows the higher values of $\pi_1$~(closer to MAE) work best for $\DGJSnp$, hinting that, the consistency term improves the learning dynamics of MAE so much so that the role of CE becomes less important.

\begin{table*}[t]
    \begin{minipage}{0.48\textwidth}
    \centering
    \vspace{-0.25cm}
    \captionof{table}{\textbf{Effect of Consistency.} Validation accuracy for $\DJSnp$, $\DGJSnp$ w/o the consistency term in Proposition \ref{prop:GJS-JS-Consistency}, and $\DGJSnp$ for $40\%$ noise on the CIFAR-100 dataset. Using the mean of two predictions in the $\DJSnp$ loss does not improve performance. On the other hand, adding the consistency term significantly helps.\vspace{0.2cm} \label{tab:gjs-wo-consistency}}
    \scriptsize
    \begin{tabular}{ l c }  \toprule
        Method & Accuracy \\ \midrule
        $\LJS(\dve{y},\dvp{2})$ & 71.0\\ 
        $\mathcal{L}_{\mathrm{JS_{\vpi'}}}(\dve{y},\dvm)$ & 68.7  \\ 
        $\LGJS(\dve{y}, \dvp{2},\dvp{3})$ & 74.3 \\ 
        \bottomrule 
    \end{tabular}
     \end{minipage}
    \hfill
    \begin{minipage}{0.48\textwidth}
        \centering
        \scriptsize
        \tabcolsep=0.16cm 
        \captionof{table}{\label{tab:cifar-concleannoisy} \textbf{Effect of }$\bm{\DGJSnp}$\textbf{.} Validation accuracy when using different loss functions for clean and noisy examples of the CIFAR-100 training set with 40\% symmetric noise. Noisy examples benefit significantly more from $\DGJSnp$ than clean examples~(74.1 vs 72.9).
}
        \begin{tabular}{c c c c c }\toprule
        \multicolumn{2}{c}{Method} & \multicolumn{3}{c}{$\pi_1$} \\ \cmidrule(lr){1-2}\cmidrule(lr){3-5}
        Clean & Noisy & 0.1 & 0.5 & 0.9
        \\\midrule
        JS & JS & 70.0 & $\bm{71.5}$ & 55.3 \\
        GJS & JS & 72.6 & $\bm{72.9}$ & 70.2 \\
        JS & GJS & 71.0 & $\bm{74.1}$ & 68.0 \\
        GJS & GJS & 71.3 & $\bm{74.7}$ & 73.8 \\ \bottomrule
        \end{tabular}
    \end{minipage} 
\end{table*}
\begin{table*}[b]
    \begin{minipage}{0.49\textwidth}
        \centering
        \scriptsize
        \tabcolsep=0.16cm 
        \captionof{table}{\label{tab:cifar-augs} \textbf{Effect of Augmentation Strategy.} Validation accuracy for training w/o CutOut(-CO) or w/o RandAug(-RA) or w/o both(weak) on 40\% symmetric and asymmetric noise on CIFAR-100. All methods improves by stronger augmentations. GJS performs best for all types of augmentations.}
        \begin{tabular}{p{0.9cm} >{\centering}p{0.365cm} >{\centering}p{0.45cm} >{\centering}p{0.45cm} >{\centering}p{0.365cm} >{\centering}p{0.365cm} >{\centering}p{0.45cm} >{\centering}p{0.45cm} >{\centering\arraybackslash}p{0.365cm}}\toprule
        \multirow{2}{4em}{Method} & \multicolumn{4}{c}{Symmetric} & \multicolumn{4}{c}{Asymmetric} \\
        \cmidrule(lr){2-5} \cmidrule(lr){6-9}
        & Full & -CO & -RA & Weak & Full & -CO & -RA & Weak
        \\\midrule
        GCE & 70.8 & 64.2 & 64.1 & 58.0 & 51.7 & 44.9 & 46.6 & 42.9  \\
        NCE+RCE & 68.5 & 66.6 & 68.3 & 61.7 & 57.5 & 52.1 & 49.5 & 44.4  \\
        GJS & $\bm{74.8}$ & $\bm{71.3}$ & $\bm{70.6}$ & $\bm{66.5}$ & \textbf{62.6} & \textbf{56.8} & \textbf{52.2} & \textbf{44.9}  \\ \bottomrule
        \end{tabular}
    \end{minipage}
    \hfill
    \begin{minipage}{0.45\textwidth}
    
    \centering
    \caption{\label{sup:tab:cifar-epochs} \textbf{Effect of Number of Epochs.} Validation accuracy for training with 200 and 400 epochs for 40\% symmetric and asymmetric noise on CIFAR-100. GJS still outperforms the baselines and NCE+RCE's performance is reduced heavily by the decrease in epochs.}
    \scriptsize
    \tabcolsep=0.290cm 
    \begin{tabular}{lcccc}\toprule
        \multirow{2}{4em}{Method} & \multicolumn{2}{c}{Symmetric} & \multicolumn{2}{c}{Asymmetric} \\
        \cmidrule(lr){2-3} \cmidrule(lr){4-5}
        & 200 & 400 & 200 & 400
        \\\midrule
        GCE & 70.3 & 70.8  & 39.1 & 51.7  \\
        NCE+RCE & 60.0 & 68.5 & 35.0 & 57.5  \\
        GJS & $\bm{72.9}$ & $\bm{74.8}$ & $\bm{43.2}$ & $\bm{62.6}$  \\ \bottomrule
    \end{tabular}
    \end{minipage}
    \vspace{-0.2cm}
\end{table*}

\textbf{Is GJS mostly helping the clean or noisy examples?} 
To better understand the improvements of $\DGJSnp$ over $\DJSnp$, we perform an ablation with different losses for clean and noisy examples, see Table~\ref{tab:cifar-concleannoisy}. 
We observe that using $\DGJSnp$ instead of $\DJSnp$ improves performance in all cases. Importantly, using $\DGJSnp$ only for the noisy examples performs significantly better than only using it for the clean examples~(74.1 vs 72.9). The best result is achieved when using $\DGJSnp$ for both clean and noisy examples but still close to the noisy-only case~(74.7 vs 74.1). 

\textbf{How is different choices of perturbations affecting GJS?} In this work, we use stochastic augmentations for $\mathcal{A}$, see Appendix \ref{sup:cifar-training-details} for details. Table \ref{tab:cifar-augs} reports validation results on 40\% symmetric and asymmetric noise on CIFAR-100 for varying types of augmentation. We observe that all methods improve their performance with stronger augmentation and that $\DGJSnp$ achieves the best results in all cases. Also, note that we use weak augmentation for all naturally-noisy datasets~(WebVision, ANIMAL-10N, and Food-101N) and still get state-of-the-art results.

\textbf{How fast is the convergence?}
We found that some baselines~(especially the robust NCE+RCE) had slow convergence. Therefore, we used 400 epochs for all methods to make sure all had time to converge properly. Table \ref{sup:tab:cifar-epochs} shows results on 40\% symmetric and asymmetric noise on CIFAR-100 when the number of epochs has been reduced by half. 

\textbf{Is training with the proposed losses leading to more consistent networks?} Our motivation for investigating losses based on Jensen-Shannon divergence was partly due to the observation in Figure~\ref{fig:consistency-observation} that consistency and accuracy correlate when learning with CE loss. In Figure~\ref{fig:consistency-observation-multiLosses}, we compare CE, $\DJSnp$, and $\DGJSnp$ losses in terms of validation accuracy and consistency during training on CIFAR-100 with 40\% symmetric noise. We find that the networks trained with $\DJSnp$ and $\DGJSnp$ losses are more consistent and has higher accuracy. In Appendix \ref{sup:sec:cifar-con}, we report the consistency of the networks in Table \ref{tab:cifar}.

\textbf{Summary of experiments in the appendix.} Due to space limitations, we report several important experiments in the appendix. We evaluate the effectiveness of $\DGJSnp$ on 1) instance-dependent synthetic noise~(Section~\ref{sup:cifar-instance}), and 2) real-world noisy datasets ANIMAL-10N and Food-101N~(Section~\ref{sup:real-world-noise}). We also investigate the importance of 1) losses being symmetric and bounded for learning with noisy labels~(Section~\ref{sup:sec:js-dissection}), and 2) a clean vs noisy validation set for hyperparameter selection and the effect of a single set of parameters for all noise rates~(Section~\ref{sup:sec:noisy-single-HPs}).

\begin{figure*}
        \centering
        \begin{subfigure}[b]{0.32\textwidth} \centering \includegraphics[width=0.9\textwidth]{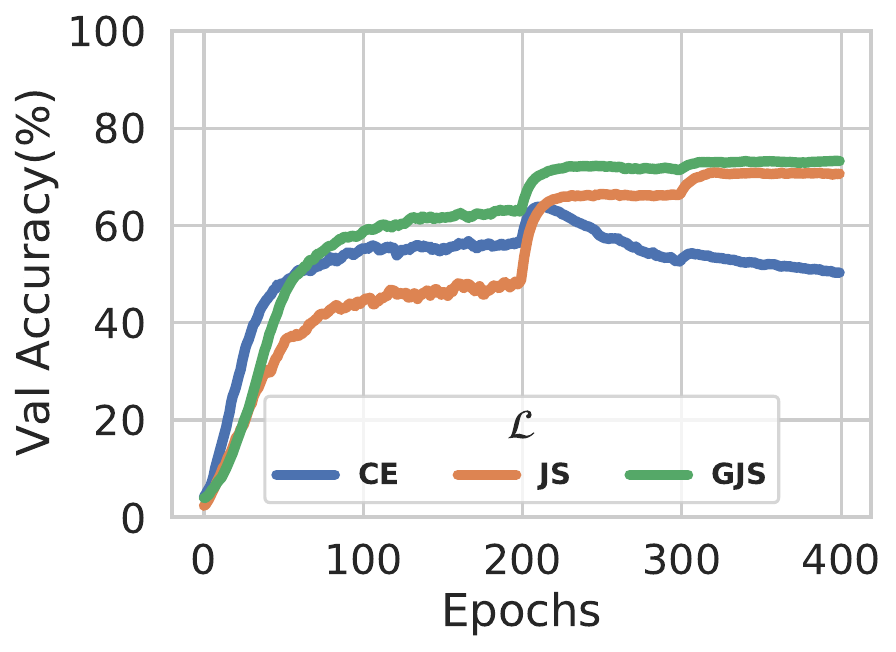} \caption{Validation Accuracy} 
         \end{subfigure} 
         \hfill
         \begin{subfigure}[b]{0.32\textwidth} \centering \includegraphics[width=0.9\textwidth]{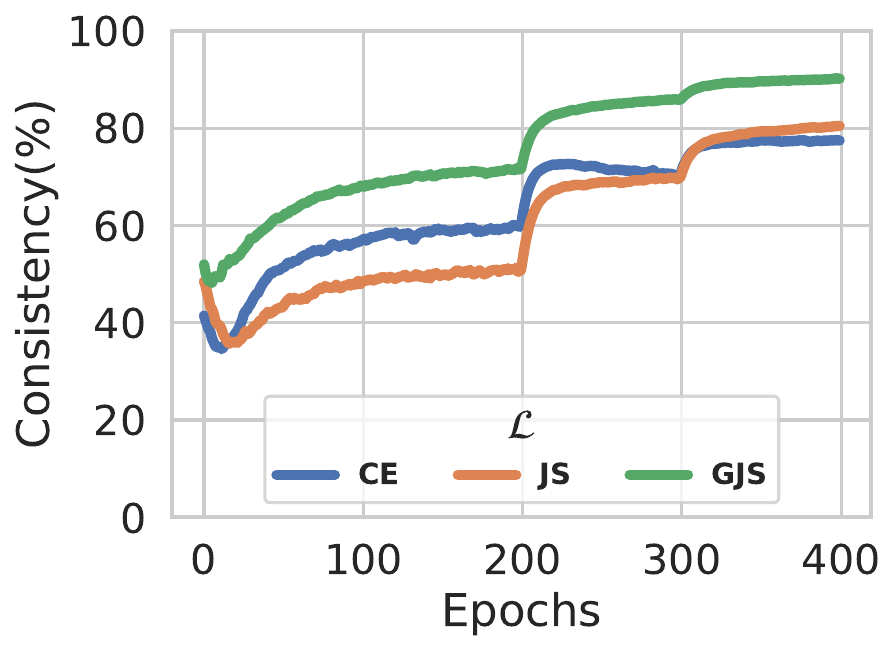} \caption{Consistency Clean} 
         \end{subfigure} 
         \begin{subfigure}[b]{0.32\textwidth} \centering \includegraphics[width=0.9\textwidth]{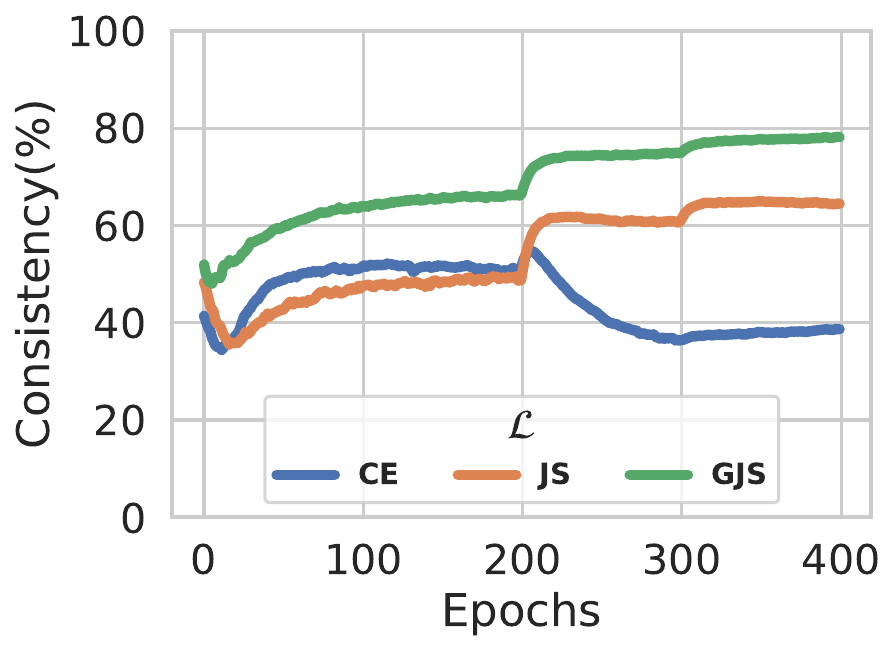} \caption{Consistency Noisy} 
         \end{subfigure} 
         \caption{\textbf{Evolution of a trained network's consistency for the CE, JS, and GJS losses.} We plot the evolution of the validation accuracy~(a) and network's consistency on clean~(b) and noisy~(c) examples of the training set of CIFAR-100 when learning with 40\% symmetric noise. All losses use the same learning rate and weight decay and both $\DJSnp$ and $\DGJSnp$ use $\pi_1=0.5$. The consistency of the learnt function and the accuracy closely correlate. The accuracy and consistency of $\DJSnp$ and $\DGJSnp$ improve during training, while both degrade when learning with CE loss. 
         } \label{fig:consistency-observation-multiLosses}
         \vspace{-0.2cm}
\end{figure*}
\section{Limitations \& Future Directions}
\vspace{-0.1cm}
We empirically showed that the consistency of the network around noisy data degrades as it fits noise and accordingly proposed a loss based on generalized Jensen-Shannon divergence ($\DGJSnp$). 
While we empirically verified the significant role of consistency regularization in robustness to noise, we only theoretically showed the robustness~($B_L=B_U$) of $\DGJSnp$ at its limit~($\pi_1 \rightarrow 1$) where the consistency term gradually vanishes. Therefore, the main limitation is the lack of a theoretical proof of the robustness of the consistency term in Proposition~\ref{prop:GJS-JS-Consistency}.
This is, in general, an important but understudied area, also for the literature of self- or semi-supervised learning and thus is of utmost importance for future works. 

Secondly, we had an important observation that $\DGJSnp$ with $M>3$ might not perform well under high noise rates. While we have some initial conjectures, this phenomenon deserves a systematic analysis both empirically and theoretically.

Finally, a minor practical limitation is the added computations for $\DGJSnp$ forward passes, however this applies to \textit{training time} only and in all our experiments, we only use one extra prediction~($M=3$).

\section{Final Remarks}
\label{sec:final-remarks}
\vspace{-0.1cm}
We first made two central observations that (i) robust loss functions have an underfitting issue and (ii) consistency of noise-fitting networks is significantly lower around noisy data points. Correspondingly, we proposed two loss functions, $\DJSnp$ and $\DGJSnp$, based on Jensen-Shannon divergence that (i) interpolates between noise-robust MAE and fast-converging CE, and (ii) encourages consistency around training data points. This simple proposal led to state-of-the-art performance on both synthetic and real-world noise datasets even when compared to the more elaborate pipelines such as DivideMix or ELR+. Furthermore,
we discussed their robustness within the theoretical construction of Ghosh~\etal~\citep{Ghosh_AAAI_2017_MAE}. By drawing further connections to other seminal loss functions such as CE, MAE, GCE, and consistency regularization, we uncovered other desirable or informative properties. We further empirically studied different aspects of the losses that corroborate various theoretical properties.

Overall, we believe the paper provides informative theoretical and empirical evidence for the usefulness of two simple and novel JS divergence-based loss functions for learning under noisy data that achieve state-of-the-art results. At the same time, it opens interesting future directions.

\textbf{Ethical Considerations.} Considerable resources are needed to create labeled data sets due to the burden of manual labeling process. Thus, the creators of large annotated datasets are mostly limited to well-funded companies and academic institutions. In that sense, developing robust methods against label noise enables less affluent organizations or individuals to benefit from labeled datasets since imperfect or automatic labeling can be used instead. On the other hand, proliferation of such harvested datasets can increase privacy concerns arising from redistribution and malicious use.
\paragraph{Acknowledgement.} This work was partially supported by the Wallenberg AI, Autonomous Systems and Software Program (WASP) funded
by the Knut and Alice Wallenberg Foundation.

\bibliography{main}
\bibliographystyle{unsrt}

\appendix
\onecolumn

\section{Training Details}
\label{sup:training-details}
All proposed losses and baselines use the same training settings, which are described in detail here.

\subsection{CIFAR}
\label{sup:cifar-training-details}
\textbf{General training details.} For all the results on the CIFAR datasets, we use a PreActResNet-34 with a standard SGD optimizer with Nesterov momentum, and a batch size of 128. For the network, we use three stacks of five residual blocks with 32, 64, and 128 filters for the layers in these stacks, respectively. The learning rate is reduced by a factor of 10 at 50\% and 75\% of the total 400 epochs. For data augmentation, we use RandAugment~\citep{cubuk2019randaugment} with $N=1$ and $M=3$ using random cropping~(size 32 with 4 pixels as padding), random horizontal flipping, normalization and lastly Cutout~\citep{devries2017cutout} with length 16. We set random seeds for all methods to have the same network weight initialization, order of data for the data loader, train-validation split, and noisy labels in the training set. We use a clean validation set corresponding to 10\% of the training data. A clean validation set is commonly provided with real-world noisy datasets~\citep{li2017webvision,li2019learning}. Any potential gain from using a clean instead of a noisy validation set is the same for all methods since all share the same setup. 

\textbf{JS and GJS implementation.} We implement the Jensen-Shannon-based losses using the definitions based on KL divergence, see Equation \ref{eq:GJSasKL}. To make sure the gradients are propagated through the target argument, we do not use the KL divergence in PyTorch. Instead, we write our own based on the official implementation.  

\textbf{Search for learning rate and weight decay.} We do a separate hyperparameter search for learning rate and weight decay on 40\% noise using both asymmetric and symmetric noises on CIFAR datasets. For CIFAR-10, we search for learning rates in $[0.001, 0.005, 0.01, 0.05, 0.1]$ and weight decays in $[1e-4, 5e-4, 1e-3]$. The method-specific hyperparameters used for this search were 0.9, 0.7, (0.1,1.0), 0.7, (1.0,1.0), 0.5, 0.5 for BS($\beta$), LS($\epsilon$), SCE($\alpha,\beta$), GCE($q$), NCE+RCE($\alpha,\beta$), JS($\pi_1$) and GJS($\pi_1$), respectively. For CIFAR-100, we search for learning rates in $[0.01, 0.05, 0.1, 0.2, 0.4]$ and weight decays in $[1e-5, 5e-5, 1e-4]$. The method-specific hyperparameters used for this search were 0.9, 0.7, (6.0,0.1), 0.7, (10.0,0.1), 0.5, 0.5 for BS($\beta$), LS($\epsilon$), SCE($\alpha,\beta$), GCE($q$), NCE+RCE($\alpha,\beta$), JS($\pi_1$) and GJS($\pi_1$), respectively. Note that, these fixed method-specific hyperparameters for both CIFAR-10 and CIFAR-100 are taken from their corresponding papers for this initial search of learning rate and weight decay but they will be further optimized systematically in the next steps.

\textbf{Search for method-specific parameters.} We fix the obtained best learning rate and weight decay for all other noise rates, but then for each noise rate/type, we search for method-specific parameters. For the methods with a single hyperparameter, BS~($\beta$), LS~($\epsilon$), GCE~($q$), JS~($\pi_1$), GJS~($\pi_1$), we try values in $[0.1, 0.3, 0.5, 0.7, 0.9]$. On the other hand, NCE+RCE and SCE have three hyperparameters, \ie $\alpha$ and $\beta$ that scale the two loss terms, and $A\coloneqq \log(0)$ for the RCE term. We set $A=\log{(1e-4)}$ and do a grid search for three values of $\alpha$ and two of beta $\beta$~(six in total) around the best reported parameters from each paper.\footnote{We also tried using $\beta=1-\alpha$, and mapping the best parameters from the papers to this range, combined with a similar search as for the single parameter methods, but this resulted in worse performance.}
 
\textbf{Test evaluation.} The best parameters are then used to train on the full training set with five different seeds. The final parameters that were used to get the results in Table \ref{tab:cifar} are shown in Table \ref{tab:cifar-hps}. 

For completeness, in Appendix \ref{sup:sec:noisy-single-HPs}, we provide results for a less thorough hyperparameter search(more similar to related work) which also use a noisy validation set.

 \begin{table*}[t!]
 \scriptsize
 \caption{\label{tab:cifar-hps} \textbf{Hyperparameters for CIFAR.} A hyperparameter search over learning rates and weight decays, was done for 40\% noise on both symmetric and asymmetric noise for the CIFAR datasets. The best parameters for each method are shown in this table, where the format is [learning rate, weight decay]. The hyperparameters for zero percent noise uses the same settings as for the symmetric noise. For the best learning rate and weight decay, another search is done for method-specific hyperparameters, and the best values are shown here. For methods with a single hyperparameter, the value correspond to their respective hyperparameter, \ie BS~($\beta$), LS~($\epsilon$), GCE~($q$), JS~($\pi_1$), GJS~($\pi_1$). For NCE+RCE and SCE the value correspond to [$\alpha$, $\beta$].
 }
 \begin{center}
 \tabcolsep=0.12cm 
 \begin{tabular}{ @{}l l c c c c c c c c c@{}} 
 \toprule
 \multirow{3}{4em}{Dataset} & \multirow{3}{4em}{Method} & \multicolumn{2}{c}{Learning Rate \& Weight Decay} & \multicolumn{7}{c}{Method-specific Hyperparameters}
 \\
  \cmidrule(lr){3-4} \cmidrule(lr){5-11}
 & & \multicolumn{1}{c}{Sym Noise} & \multicolumn{1}{c}{Asym Noise} & \multicolumn{1}{c}{No Noise} & \multicolumn{4}{c}{Sym Noise} & \multicolumn{2}{c}{Asym Noise} \\ 
 \cmidrule(lr){3-3} \cmidrule(lr){4-4} \cmidrule(lr){5-5} \cmidrule(lr){6-9} \cmidrule(lr){10-11}
 & & 20-80\% & 20-40\% & 0\% & 20\% & 40\% & 60\% & 80\% & 20\% & 40\%   \\
 \midrule
 \multirow{8}{5em}{CIFAR-10} & CE & [0.05, 1e-3] & [0.1, 1e-3] & - & - & - & - & - & - & -\\
 & BS & [0.1, 1e-3] & [0.1, 1e-3] & 0.5 & 0.5 & 0.7 & 0.7 & 0.9 & 0.7 & 0.5 \\
 & LS & [0.1, 5e-4] & [0.1, 1e-3] & 0.1 & 0.5 & 0.9 & 0.7 & 0.1 & 0.1 & 0.1 \\
 & SCE & [0.01, 5e-4] & [0.05, 1e-3] & [0.2, 0.1] & [0.05, 0.1] & [0.1, 0.1] & [0.2, 1.0] & [0.1,1.0] & [0.1, 0.1] & [0.2, 1.0] \\
 & GCE & [0.01, 5e-4] & [0.1, 1e-3] & 0.5 & 0.7 & 0.7 & 0.7 & 0.9 & 0.1 & 0.1 \\
 & NCE+RCE & [0.005, 1e-3] & [0.05, 1e-4] & [10, 0.1] & [10, 0.1] & [10, 0.1] & [1.0, 0.1] & [10,1.0] & [10, 0.1] & [1.0, 0.1] \\
 & JS & [0.01, 5e-4] & [0.1, 1e-3] & 0.1 & 0.7 & 0.7 & 0.9 & 0.9 & 0.3 & 0.3 \\
 & GJS & [0.1, 5e-4] & [0.1, 1e-3] & 0.5 & 0.3 & 0.9 & 0.1 & 0.1 & 0.3 & 0.3 \\
\midrule
 \multirow{8}{5em}{CIFAR-100} & CE & [0.4, 1e-4] & [0.2, 1e-4] & - & - & - & - & - & - & - \\
 & BS & [0.4, 1e-4] & [0.4, 1e-4] & 0.7 & 0.5 & 0.5 & 0.5 & 0.9 & 0.3 & 0.3  \\
 & LS & [0.2, 5e-5] & [0.4, 1e-4] & 0.1 & 0.7 & 0.7 & 0.7 & 0.9 & 0.5 & 0.7 \\
 & SCE & [0.2, 1e-4] & [0.4, 5e-5] & [0.1, 0.1] & [0.1, 0.1] & [0.1, 0.1] & [0.1, 1.0] & [0.1,0.1] & [0.1, 1.0] & [0.1, 1.0] \\
 & GCE & [0.4, 1e-5] & [0.2, 1e-4] & 0.5 & 0.5 & 0.5 & 0.7 & 0.7 & 0.7 & 0.7 \\
 & NCE+RCE & [0.2, 5e-5] & [0.2, 5e-5] & [20, 0.1] & [20, 0.1] & [20, 0.1] & [20, 0.1] & [20,0.1] & [20, 0.1] & [10, 0.1] \\
 & JS & [0.2, 1e-4] & [0.1, 1e-4] & 0.1 & 0.1 & 0.3 & 0.5 & 0.3 & 0.5 & 0.5 \\
 & GJS & [0.2, 5e-5] & [0.4, 1e-4] & 0.3 & 0.3 & 0.5 & 0.9 & 0.1 & 0.5 & 0.1 \\
 \bottomrule
\end{tabular}
 \end{center}
 \end{table*}

\subsection{WebVision}
All methods train a randomly initialized ResNet-50 model from PyTorch using the SGD optimizer with Nesterov momentum, and a batch size of 32 for $\DGJSnp$ and 64 for CE and $\DJSnp$. For data augmentation, we do a random resize crop of size 224, random horizontal flips, and color jitter~(torchvision ColorJitter transform with brightness=0.4, contrast=0.4, saturation=0.4, hue=0.2). We use a fixed weight decay of $1e-4$ and do a grid search for the best learning rate in $[0.1, 0.2, 0.4]$ and $\pi_1 \in [0.1, 0.3, 0.5, 0.7, 0.9]$. The learning rate is reduced by a multiplicative factor of $0.97$ every epoch, and we train for a total of 300 epochs. 
The best starting learning rates were 0.4, 0.2, 0.1 for CE, JS and GJS, respectively. Both $\DJSnp$ and $\DGJSnp$ used $\pi_1=0.1$. With the best learning rate and $\pi_1$, we ran four more runs with new seeds for the network initialization and data loader.

\section{Additional Experiments and Insights}
\label{sup:experiments}

\subsection{Instance-Dependent Synthetic Noise}
\label{sup:cifar-instance}
In Section \ref{sec:synthetic-noise}, we showed results on two types of synthetic noise: symmetric~$(\eta)$ and asymmetric~($\eta(y)$). Although these noise types are simple to empirically and theoretically analyze, they might be different from noise observed in real-world datasets. Recently, a new type of synthetic noise has been proposed by Zhang~\etal~\cite{zhang2021learning}, where the risks of mislabeling an example of class $i$ to class $j$ vary per example~($\eta_{ij}(\vx)$). This type of noise is called instance-dependent and is more similar the noise in real-world datasets.

In Table \ref{tab:cifar-instance-dependent}, we compare CE, Generalized CE~(GCE) and $\DGJSnp$ on three different types of 35\% instance-dependent noise on the CIFAR datasets. The training setup is the same as for the results in Table \ref{tab:cifar}, described in detail in Section \ref{sup:cifar-training-details}. For all methods, we search for the best hyperparameters on the Type-I noise and use the same settings for the other two types. For CIFAR-10, the optimal hyperparameters~(learning rate, weight decay, method-specific) were: (0.1, 1e-3, -), (0.005, 1e-3, 0.9), (0.001, 5e-4, 0.5) for CE, GCE, and $\DGJSnp$, respectively. For CIFAR-100, they were: (0.1, 5e-4, -), (0.4, 5e-5, 0.7), (0.1, 5e-4, 0.3) for CE, GCE, and $\DGJSnp$, respectively.

On the simpler CIFAR-10, GCE and $\DGJSnp$ perform similarly, but on the more challenging CIFAR-100, $\DGJSnp$ significantly outperform GCE.

\renewcommand{\arraystretch}{1.2}
\begin{table*}[t!]
\tiny 
\caption{\label{tab:cifar-instance} \textbf{Instance-Dependent Synthetic Noise Benchmark on CIFAR.} We \textit{reimplement} the Generalized CE~(GCE) loss function into the \textit{same learning setup} and a ResNet-34 network. We used \textit{same hyperparameter optimization budget and mechanism} for all methods. We evaluate on 35\% noise for the three types of instance-dependent synthetic noise proposed by Zhang~\etal~\cite{zhang2021learning}. Mean test accuracy and standard deviation are reported from five runs and the statistically-significant top performers are boldfaced. As for the symmetric and asymmetric synthetic noise, the efficacy of $\DGJSnp$ is more evident on the more challenging CIFAR-100 dataset, where $\DGJSnp$ significantly outperforms the baselines. \label{tab:cifar-instance-dependent}
}
\begin{center}
\tabcolsep=0.22cm 
\begin{tabular}{ @{}l l c c c c} 
 \toprule
 \multirow{2}{*}{Dataset} & \multirow{2}{*}{Method} & No Noise & \multicolumn{3}{c}{Instance-Dependent Noise} \\ \cmidrule(lr){3-3}\cmidrule(lr){4-6} 
 & & 0\% & Type-I & Type-II & Type-III \\
 \midrule
 \multirow{3}{5em}{CIFAR-10} & CE & 94.35 $\pm$ 0.10 & 83.16 $\pm$ 0.36 & 81.18 $\pm$ 0.38 & 81.80 $\pm$ 0.13 \\
& GCE & 94.00 $\pm$ 0.08 & \textbf{86.50 $\pm$ 0.16} & \textbf{83.80 $\pm$ 0.26} & \textbf{84.85 $\pm$ 0.12} \\
 & GJS & \textbf{94.78 $\pm$ 0.06} & 85.98 $\pm$ 0.12 & \textbf{83.81 $\pm$ 0.12} & \textbf{84.83 $\pm$ 0.26} \\
\midrule
 \multirow{3}{5em}{CIFAR-100} & CE & 77.60 $\pm$ 0.17 & 62.46 $\pm$ 0.31 & 63.51 $\pm$ 0.41 & 62.44 $\pm$ 0.47  \\
 & GCE & 77.65 $\pm$ 0.17 & 65.62 $\pm$ 0.32 & 65.84 $\pm$ 0.35 & 65.85 $\pm$ 0.32 \\
 & GJS & \textbf{79.27 $\pm$ 0.29} & \textbf{68.49 $\pm$ 0.14} & \textbf{69.21 $\pm$ 0.16} & \textbf{69.04 $\pm$ 0.16} \\
\bottomrule
\end{tabular}
\end{center}
\vspace{-0.3cm}
\end{table*}

\subsection{Real-World Noise: ANIMAL-10N \& Food-101N}
\label{sup:real-world-noise}
Here, we evaluate $\DGJSnp$ on two naturally-noisy datasets: ANIMAL-10N~\citep{song2019selfie} and Food-101N~\citep{lee2017cleannet}.

\textbf{Food-101N.} The dataset contains 301k images classified as 101 different food recipes. The images were collected using Google, Bing, Yelp, and TripAdvisor. The noise rate is estimated to be 20\%. 

We follow the same training setup as the recent label correction method called Progressive Label Correction (PLC)~\cite{zhang2021learning}, \ie we use the same network architecture, augmentation strategy, optimizer, batch size, number of epochs, and learning rate scheduling. We use an initial learning rate and weight decay of 0.001, and $\pi_1=0.3$.

\textbf{ANIMAL-10N.} The dataset contains 55k images of 10 classes. The 10 classes can be grouped into 5 pairs of similar classes that are more likely to be confused: (cat, lynx), (jaguar, cheetah), (wolf, coyote), (chimpanzee, orangutan), (hamster, guinea pig). The images were collected using Google and Bing. The noise rate is estimated to be 8\%.

We use the same training setup(network, optimizer, number of epochs, learning rate scheduling, etc) as PLC, but use cropping instead of random horizontal flipping as augmentation to reduce the risk of both augmentations being equal for $\DGJSnp$. We use an initial learning rate of 0.05, a weight decay of 5e-4, and $\pi_1=0.5$.

\textbf{Results.} The mean test accuracy and standard deviation from three runs for ANIMAL-10N and Food-101N are in Table \ref{sup:tab:animal-10N} and \ref{sup:tab:Food-101N}, respectively. The results for all baselines are from Zhang~\etal~\cite{zhang2021learning}. Our $\DGJSnp$ loss outperforms all other methods on both datasets.

\begin{table}[!htb]
    \scriptsize
    \begin{minipage}{.50\linewidth}
      \caption{\textbf{Real-world Noise: ANIMAL-10N.} \label{sup:tab:animal-10N}}
      \centering
        \begin{tabular}{lc}
         \toprule
            Method & Accuracy \\
            \midrule
            CE & 79.4 $\pm$ 0.14 \\
            SELFIE & 81.8 $\pm$ 0.09 \\
            PLC & 83.4 $\pm$ 0.43 \\
            \textbf{GJS} & \textbf{84.2 $\pm$ 0.07} \\ 
        \bottomrule
        \end{tabular}
    \end{minipage}%
    \hfill
    \begin{minipage}{.45\linewidth}
      \centering
        \caption{\textbf{Real-world Noise: Food-101N.}\label{sup:tab:Food-101N}}
        \begin{tabular}{lc}
         \toprule
            Method & Accuracy \\
            \midrule
            CE & 81.67 \\
            CleanNet & 83.95 \\
            PLC & 85.28 $\pm$ 0.04 \\
            \textbf{GJS} & \textbf{86.56 $\pm$ 0.13} \\ 
        \bottomrule
        \end{tabular}
    \end{minipage} 
\end{table}

\subsection{Towards a better understanding of JS}
\label{sup:sec:js-dissection}

\begin{figure}[t!]
    \centering
    \begin{minipage}{0.47\textwidth}
    \centering
        \sisetup{group-four-digits}
        \tiny
        \vspace{-4.6cm}
        \captionof{table}{\textbf{Ablation Study of $\boldsymbol{\DJSnp}$.} 
    A comparison of $\DJSnp$ and other KL-based divergences and their relationship to symmetry and boundedness. The distribution $\vm$ is the mean of $\vp$ and $\vq$.\label{tab:dissection}}
        \begin{tabular}{ p{0.5cm} p{2.985cm} >{\centering}m{0.6cm} >{\centering\arraybackslash}m{0.6cm} }  \toprule
            Method & Formula & Symmetric & Bounded \\ \midrule
            KL & $KL(\vp,\vq)$ & &   \\
            KL' & $KL(\vq,\vp)$ & &   \\
            Jeffrey's & $(KL(\vp,\vq)+KL(\vq,\vp))/2$ & \checkmark &  \\
            K & $KL(\vp,\vm)$ & & \checkmark  \\
            K' & $KL(\vq,\vm)$ & & \checkmark  \\
            JS & $(KL(\vp,\vm)+KL(\vq,\vm))/2$ & \checkmark & \checkmark  \\
            \bottomrule
        \end{tabular}
    \vspace{0.5cm}
    \end{minipage}
    \hfill
    \begin{minipage}{0.51\textwidth}
        \centering
        \begin{subfigure}{1.0\columnwidth}
            \includegraphics[trim={0.3cm 0cm 0.1cm 0.1cm},clip,width=1.0\columnwidth]{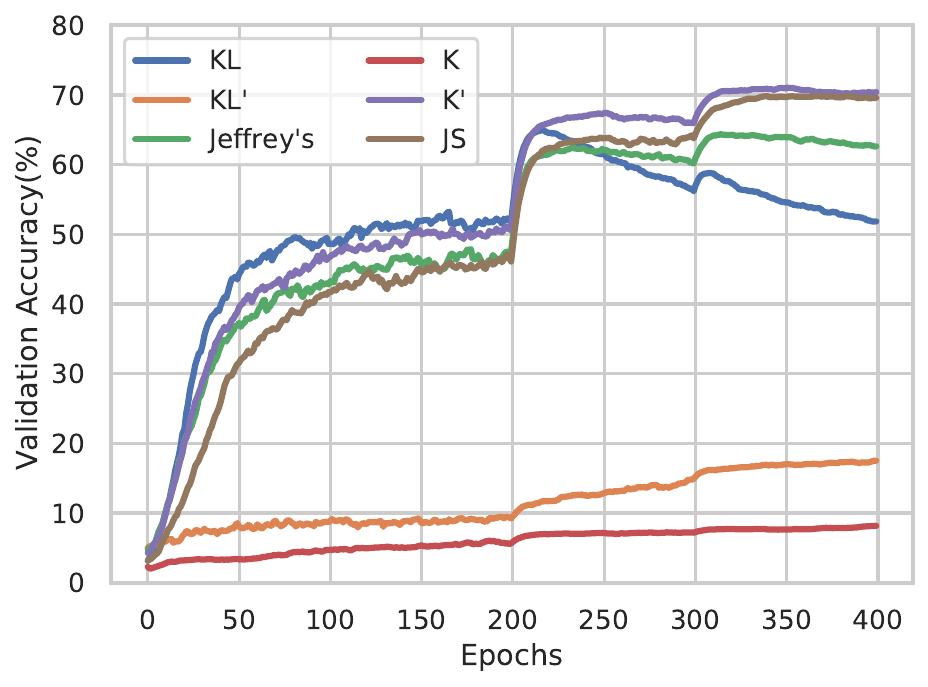}
        \end{subfigure}
        \caption{\textbf{Ablation Study of $\boldsymbol{\DJSnp}$.} Validation accuracy of the divergences in Table \ref{tab:dissection} are plotted during training with $40\%$ symmetric noise on the CIFAR-100 dataset. Notably, the only two losses that show signs of overfitting ($KL$ and Jeffrey's) are unbounded. Interestingly, $K$~(bounded $KL$) makes the learning slower, while $K'$~(bounded $KL'$) considerably improves the learning dynamics. Finally, it can be seen that, $\DJSnp$, in contrast to its unbounded version (Jeffrey's), does not overfit to noise.\label{fig:js-dissect}}
    \end{minipage}
\end{figure}
\label{sec:js-dissection}
In Proposition \ref{prop:GJS-JS-Consistency}, we showed that $\DJSnp$ is an important part of $\DGJSnp$, and therefore deserves attention. Here, we make a systematic \textit{ablation study} to empirically examine the contribution of the difference(s) between $\DJSnp$ loss and CE. We decompose the $\DJSnp$ loss following the gradual construction of the Jensen-Shannon divergence in the work of Lin~\citep{Lin_TIT_1991_JS_Divergence}. This construction, interestingly, lends significant empirical evidence to bounded losses' robustness to noise, in connection to Theorem~\ref{T:sym} and \ref{T:asym} and Proposition~\ref{prop:gjs-bounds}.

Let $KL({\vp},{\vq})$ denote the KL-divergence of a predictive distribution $\vq \in \Delta^{K-1}$ from a target distribution $\vp \in \Delta^{K-1}$.
$KL$ divergence is neither symmetric nor bounded. $K$ divergence, proposed by Lin~\etal~\citep{Lin_TIT_1991_JS_Divergence}, is a bounded version defined as $K(\vp,\vq) \coloneqq KL(\vp,(\vp + \vq)/2)=KL(\vp,\vm)$. However, this divergence is not symmetric. A simple way to achieve symmetry is to take the average of forward and reverse versions of a divergence. For $KL$ and $K$, this gives rise to Jeffrey's divergence and $\DJSnp$ with $\vpi=[\frac{1}{2},\frac{1}{2}]^T$, respectively. Table \ref{tab:dissection} provides an overview of these divergences and Figure \ref{fig:js-dissect} shows their validation accuracy during training on CIFAR-100 with 40\% symmetric noise.

\textbf{Bounded.} Notably, the only two losses that show signs of overfitting ($KL$ and Jeffrey's) are unbounded.
Interestingly, $K$~(bounded $KL$) makes the learning much slower, while $K'$~(bounded $KL'$) considerably improves the learning dynamics. Finally, it can be seen that, $\DJSnp$, in contrast to its unbounded version (Jeffrey's), does not overfit to noise. 

\textbf{Symmetry.} The Jeffrey's divergence performs better than either of its two constituent $KL$ terms. This is not as clear for $\DJSnp$, where $K'$ is performing surprisingly well on its own. In the proof of Proposition \ref{prop:js-ce-mae}, we show that $K'\to$ MAE as $\pi_1\to 1$, while $K$ goes to zero, which could explain why $K'$ seems to be robust to noise. Furthermore, $K'$, which is a component of $\DJSnp$, is reminiscent of label smoothing.

Beside the bound and symmetry, other notable properties of $\DJSnp$ and $\DGJSnp$ are the connections to MAE and consistency losses. Next section investigates the effect of hyperparameters that substantiates the connection to MAE~(Proposition~\ref{prop:JS_limits_CE_MAE}).

\subsection{Comparison between JS and GCE}
\label{sup:js-gce-comp}
We were pleasantly surprised by the finding in Proposition \ref{prop:JS_limits_CE_MAE} that $\DJSnp$ generalizes CE and MAE, similarly to GCE. Here, we highlight differences between $\DJSnp$ and GCE.

\textbf{Theoretical properties.} Our inspiration to study $\DJSnp$ came from the \textit{symmetric} loss function of SCE, and the \textit{bounded} loss of GCE. $\DJSnp$ has \textit{both} properties and a rich history in the field of information theory. This is also one of the reasons we studied these properties in Section \ref{sec:js-dissection}. Finally, $\DJSnp$ generalizes naturally to more than two distributions.

\textbf{Gradients.} The gradients of CE/KL, GCE, $\DJSnp$ and MAE with respect to logit $z_i$ of prediction $\vp=[p_1,p_2,\dots,p_K]$, given a label $\dve{y}$, are of the form $-\frac{\partial p_y}{\partial z_i}g(p_y)$ with $g(p_y)$ being $\frac{1}{p_y}$, $\frac{1}{p_y^{1-q}}$, $(1-\pi_1)\log{\Big(\frac{\pi_1}{(1-\pi_1)p_y} + 1 \Big) }/Z$, and $1$, for each of these losses respectively. Note that, $q$ is the hyperparameter of GCE and $p_y$ denotes the yth component of $\vp$. 

In Figure \ref{fig:js-gce-comparison}, these gradients are compared by varying the hyperparameter of GCE, $q \in [0.1, 0.3, 0.5, 0.7, 0.9]$, and finding the corresponding $\vpi$ for $\DJS$ such that the two gradients are equal at $p_y=\frac{1}{2}$. 

Looking at the behaviour of the different losses at low-$p_y$ regime, intuitively, a high gradient scale for low $p_y$ means a large parameter update for deviating from the given class. This can make noise free learning faster by pushing the probability to the correct class, which is what CE does. However, if the given class is incorrect~(noisy) this can cause overfitting. The gradient scale of MAE induces same update magnitude for $p_y$, which can give the network more freedom to deviate from noisy classes, at the cost of slower learning for the correctly labeled examples. 

Comparing GCE and $\DJS$ in Figure \ref{fig:js-gce-comparison}, it can be seen that $\DJS$ generally penalize lower probability in the given class less than what GCE does. In this sense, $\DJS$ behaves more like MAE.

For a derivation of the gradients of $\JSnp$, see Section \ref{sup:js-grad}.

\textbf{Label distributions.} GCE requires the label distribution to be onehot which makes it harder to incorporate GCE in many of the elaborate state-of-the-art methods that use ``soft labels'' \textit{e.g.}, Mixup, co-training, or knowledge distillation.  
 \begin{figure}[t!] 
 \centering 
 \begin{subfigure}[b]{0.48\textwidth} \centering \includegraphics[width=\textwidth]{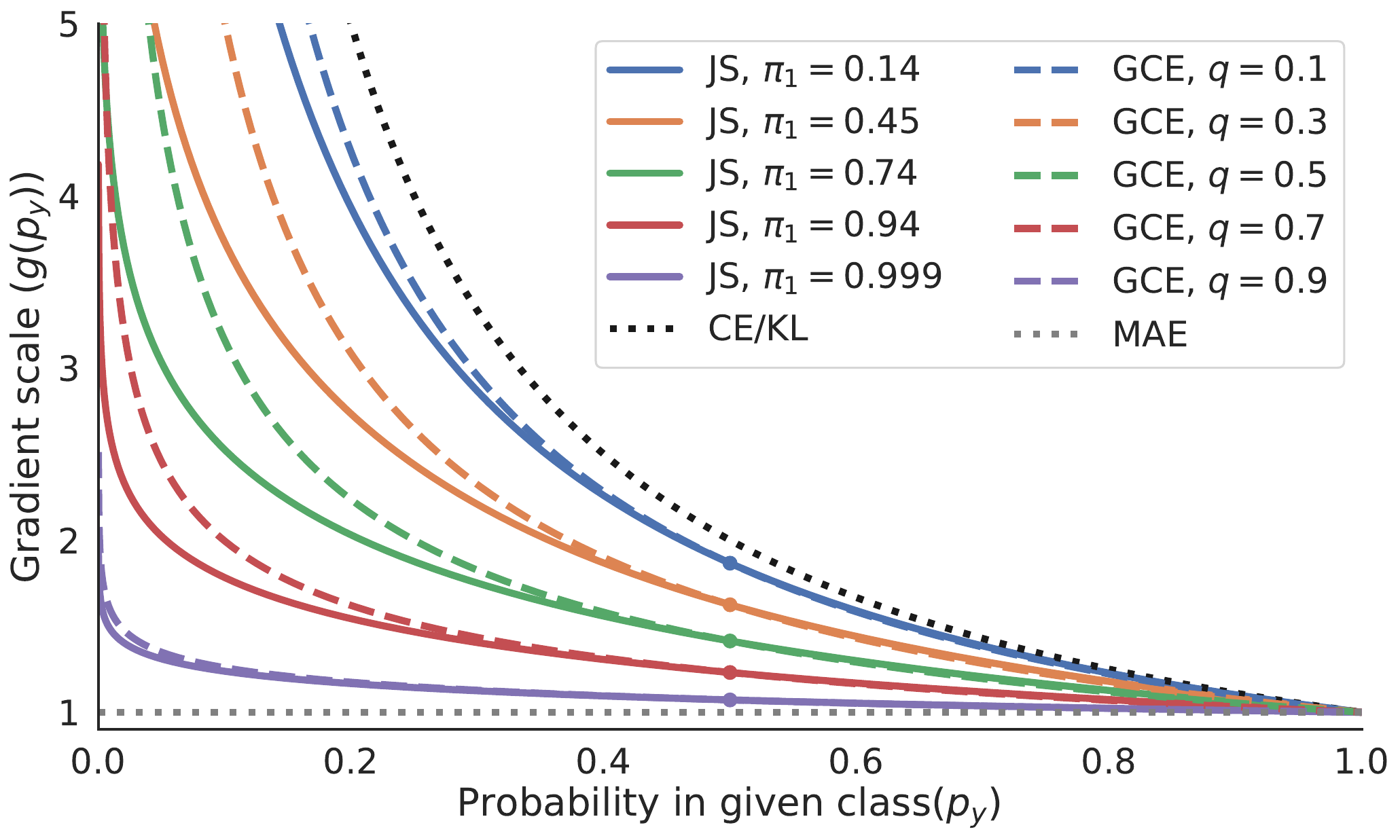} 
 \label{fig:grad-comparison} 
 \end{subfigure} 
 \caption{\textbf{Comparison between JS and GCE.} A comparison of gradients scales between $\DJSnp$ and GCE. For each $q$ of GCE, a corresponding $\pi_1$ of JS is chosen such that the gradient scales are equal at $p_y=\frac{1}{2}$.} \label{fig:js-gce-comparison}
 \vskip -0.1in
 \end{figure}

\subsection{Noisy Validation Set \& Single Set of Parameters}
\label{sup:sec:noisy-single-HPs}
Our systematic procedure to search for hyperparameters~(\ref{sup:cifar-training-details}) is done to have a more conclusive comparison to other methods. The most common procedure in related works is for each dataset, all methods use the same learning rate and weight decay(chosen seemingly arbitrary), and each method uses a single set of method-specific parameters for all noise rates and types. Baselines typically use the same method-specific parameters as reported in their respective papers. First, using the same learning rate and weight decay is problematic when comparing loss functions that have different gradient magnitudes. Second, directly using the parameters reported for the baselines is also problematic since the optimal hyperparameters depend on the training setup, which could be different, e.g., network architecture, augmentation, learning rate schedule, etc. Third, using a fixed method-specific parameter for all noise rates makes the results highly dependent on this choice. Lastly, it is not possible to know if other methods would have performed better if a proper hyperparameter search was done.

Here, for completeness, we use the same setup as in Section \ref{sup:cifar-training-details}, except we use the same learning rate and weight decay for all methods and search for hyperparameters based on a noisy validation set~(more similar to related work). 

The learning rate and weight decay for all methods are chosen based on noisy validation accuracy for CE on 40\% symmetric noise for each dataset. The optimal learning rates and weight decays([lr,wd]) were [0.05, 1e-3] and [0.4, 1e-4] for CIFAR-10 and CIFAR-100, respectively. The method-specific parameters are found by a similar search as in Section \ref{sup:cifar-training-details}, except it is only done for 40\% symmetric noise and the optimal parameters are used for all other noise rates and types. For CIFAR-10, the optimal method-specific hyperparameters were 0.5, 0.5, (0.1,0.1), 0.5, (10, 0.1), 0.5, 0.3 for BS($\beta$), LS($\epsilon$), SCE($\alpha,\beta$), GCE($q$), NCE+RCE($\alpha,\beta$), JS($\pi_1$) and GJS($\pi_1$), respectively. For CIFAR-100, the optimal method-specific hyperparameters were 0.5, 0.7, (0.1, 0.1), 0.5, (20, 0.1), 0.1, 0.5 for BS($\beta$), LS($\epsilon$), SCE($\alpha,\beta$), GCE($q$), NCE+RCE($\alpha,\beta$), JS($\pi_1$) and GJS($\pi_1$), respectively. The results with this setup can be seen in Table \ref{sup:tab:cifar-noisy-single}.

\renewcommand{\arraystretch}{1.2}
\begin{table*}[t!]
\tiny 
\tabcolsep=0.22cm 
\caption{\label{sup:tab:cifar-noisy-single} \textbf{Synthetic Noise Benchmark on CIFAR.} We \textit{reimplement} other noise-robust loss functions into the \textit{same learning setup} and ResNet-34, including label smoothing~(LS), Bootstrap~(BS), Symmetric CE~(SCE), Generalized CE~(GCE), and Normalized CE~(NCE+RCE). We used \textit{same hyperparameter optimization budget and mechanism} for all the prior works and ours. All methods use the same learning rate and weight decay and use the optimal method-specific parameters from a search on 40\% symmetric noise based on noisy validation accuracy. Mean test accuracy and standard deviation are reported from five runs and the statistically-significant top performers are boldfaced. 
}
\begin{center}
\begin{tabular}{ @{}l l c c c c c c c} 
 \toprule
 \multirow{2}{4em}{Dataset} & \multirow{2}{4em}{Method} & No Noise & \multicolumn{4}{c}{Symmetric Noise Rate} & \multicolumn{2}{c}{Asymmetric Noise Rate} \\ \cmidrule(lr){3-3}\cmidrule(lr){4-7} \cmidrule(lr){8-9}
 & & 0\% & 20\% & 40\% & 60\% & 80\% & 20\% & 40\% \\
 \midrule
 \multirow{8}{5em}{CIFAR-10} & CE & \textbf{95.66 $\pm$ 0.18} & 91.47 $\pm$ 0.28 & 87.31 $\pm$ 0.29 & 81.96 $\pm$ 0.38 & 65.28 $\pm$ 0.90 & 92.80 $\pm$ 0.64 & 85.82 $\pm$ 0.42  \\
 & BS & 95.47 $\pm$ 0.11 & 93.65 $\pm$ 0.23 & 90.77 $\pm$ 0.30 & 49.80 $\pm$ 20.64 & 32.91 $\pm$ 5.43 & 93.86 $\pm$ 0.14 & 85.37 $\pm$ 1.07 \\
 & LS & 95.45 $\pm$ 0.15 & 93.52 $\pm$ 0.09 & 89.94 $\pm$ 0.17 & 84.13 $\pm$ 0.80 & 62.76 $\pm$ 2.00 & 92.71 $\pm$ 0.41 & 83.61 $\pm$ 1.21 \\
 & SCE & 94.92 $\pm$ 0.18 & 93.41 $\pm$ 0.20 & 90.99 $\pm$ 0.20 & 86.04 $\pm$ 0.31 & 41.04 $\pm$ 4.56 & 93.26 $\pm$ 0.13 & 84.46 $\pm$ 1.22\\
& GCE & 94.94 $\pm$ 0.09 & 93.79 $\pm$ 0.19 & 91.45 $\pm$ 0.17 & 86.00 $\pm$ 0.20 & 62.01 $\pm$ 2.54 & 93.23 $\pm$ 0.12 & 85.92 $\pm$ 0.61\\
 & NCE+RCE & 94.31 $\pm$ 0.16 & 92.79 $\pm$ 0.16 & 90.31 $\pm$ 0.23 & 84.80 $\pm$ 0.47 & 34.47 $\pm$ 14.66 & 92.99 $\pm$ 0.15 & 87.00 $\pm$ 1.05 \\
  & JS & 94.74 $\pm$ 0.21 & 93.53 $\pm$ 0.23 & 91.57 $\pm$ 0.22 & 86.21 $\pm$ 0.48 & 65.87 $\pm$ 2.92 & 92.97 $\pm$ 0.26 & 86.42 $\pm$ 0.36 \\
 & GJS & \textbf{95.86 $\pm$ 0.10} & \textbf{95.20 $\pm$ 0.11} & \textbf{94.13 $\pm$ 0.19} & \textbf{89.65 $\pm$ 0.26} & \textbf{76.74 $\pm$ 0.75} & \textbf{94.81 $\pm$ 0.10} & \textbf{90.29 $\pm$ 0.26} \\
\midrule
 \multirow{8}{5em}{CIFAR-100} & CE & 77.84 $\pm$ 0.17 & 65.74 $\pm$ 0.06 & 55.57 $\pm$ 0.55 & 44.60 $\pm$ 0.79 & 10.74 $\pm$ 5.11 & 66.61 $\pm$ 0.45                        & 50.42 $\pm$ 0.44 \\
 & BS & 77.63 $\pm$ 0.25 & 73.01 $\pm$ 0.28 & 68.35 $\pm$ 0.43 & 54.07 $\pm$ 1.16 & 2.43 $\pm$ 0.49 & 69.75 $\pm$ 0.35 & 50.61 $\pm$ 0.32 \\
 & LS & 77.60 $\pm$ 0.28 & 74.22 $\pm$ 0.30 & 66.84 $\pm$ 0.28 & 54.09 $\pm$ 0.71     & 21.00 $\pm$ 2.14 & 73.30 $\pm$ 0.42 & \textbf{57.02 $\pm$ 0.57} \\
 & SCE & 77.46 $\pm$ 0.39 & 73.26 $\pm$ 0.29 & 66.96 $\pm$ 0.27 & 54.09 $\pm$ 0.49 & 13.26 $\pm$ 2.31 & 71.22 $\pm$ 0.33 & 49.91 $\pm$ 0.28 \\
 & GCE & 76.70 $\pm$ 0.39 & 74.14 $\pm$ 0.32 & 70.41 $\pm$ 0.40 & 62.14 $\pm$ 0.27 & 12.38 $\pm$ 3.74 & 69.40 $\pm$ 0.30 & 48.54 $\pm$ 0.30 \\
 & NCE+RCE & 73.23 $\pm$ 0.34 & 70.19 $\pm$ 0.27 & 65.61 $\pm$ 0.87 & 50.33 $\pm$ 1.58 & 5.55 $\pm$ 1.67 & 69.47 $\pm$ 0.25 & 56.32 $\pm$ 0.33 \\
 & JS & 77.20 $\pm$ 0.53 & 74.47 $\pm$ 0.25 & 70.12 $\pm$ 0.39 & 61.69 $\pm$ 0.63 & \textbf{27.77 $\pm$ 4.11} & 67.21 $\pm$ 0.37 & 49.39 $\pm$ 0.13 \\
 & GJS & \textbf{78.76 $\pm$ 0.32} & \textbf{77.14 $\pm$ 0.45} & \textbf{74.69 $\pm$ 0.12} & \textbf{64.06 $\pm$ 0.52} & 12.95 $\pm$ 2.40 & \textbf{74.44 $\pm$ 0.49} & 52.34 $\pm$ 0.81 \\
 \bottomrule
\end{tabular}
\end{center}
\end{table*}

\subsection{Consistency Measure}
\label{sup:sec:consistencyMeasure}
 In this section, we provide more details about the consistency measure used in Figure \ref{fig:consistency-observation}. To be independent of any particular loss function, we considered a measure similar to standard Top-1 accuracy. We measure the ratio of samples that predict the same class on both the original image and an augmented version of it
 \begin{align}
     \frac{1}{N}\sum_{i=1}^N \mathbbm{1}\big(\argmax_y f(\vx_i) = \argmax_y f(\tilde{\vx}_i)\big)
 \end{align}
where the sum is over all the training examples, and $\mathbbm{1}$ is the indicator function, the argmax is over the predicted probability of $K$ classes, and $\tilde{\vx}_i \sim \mathcal{A}(\vx_i)$ is an augmented version of $\vx_i$. Notably, this measure does not depend on the labels. 
 
 In the experiment in Figure~\ref{fig:consistency-observation}, the original images are only normalized, while the augmented images use the same augmentation strategy as the benchmark experiments, see Section \ref{sup:cifar-training-details}.

\subsection{Consistency of Trained Networks on CIFAR}
\label{sup:sec:cifar-con}
In Table \ref{sup:tab:cifar-con-acc}, we report the training consistency of the networks used for the main CIFAR results in Table \ref{tab:cifar}. We use the same consistency measure (Section \ref{sup:sec:consistencyMeasure}) as was used in Figure \ref{fig:consistency-observation} and Figure \ref{fig:consistency-observation-multiLosses}. When learning with noisy labels, the networks trained with $\DGJSnp$ is significantly more consistent than all the other methods. This is directly in line with Proposition \ref{prop:GJS-JS-Consistency}, that shows how $\LGJS$ encourages consistency. 

In Table \ref{tab:cifar}, we noticed better performance for CE compared to reported results in related work, which we mainly attribute to our thorough hyperparameter search. In Table \ref{sup:tab:cifar-con-acc}, we observe better consistency for CE than in Figure \ref{fig:consistency-observation}, which we believe is for the same reason. Compared to Figure \ref{fig:consistency-observation}, the networks trained with the CE loss in Table \ref{sup:tab:cifar-con-acc} use a higher learning rate and weight decay, both of which have a regularizing effect, which could help against overfitting to noise.

\renewcommand{\arraystretch}{1.2}
\begin{table*}[t!]
\tiny 
\tabcolsep=0.22cm 
\caption{\label{sup:tab:cifar-con-acc} \textbf{Consistency of Trained Networks on CIFAR.} The training consistency of the networks from Table \ref{tab:cifar}. Mean train consistency and standard deviation are reported from five runs and the networks with significantly higher consistency are boldfaced. As observed in Figure \ref{fig:consistency-observation}, the consistency is reduced for all methods for increasing noise rates. When learning with noisy labels, the networks trained with $\DGJSnp$ are the most consistent for all noise rates and datasets.
}

\begin{center}
\begin{tabular}{ @{}l l c c c c c c c} 
 \toprule
 \multirow{2}{4em}{Dataset} & \multirow{2}{4em}{Method} & No Noise & \multicolumn{4}{c}{Symmetric Noise Rate} & \multicolumn{2}{c}{Asymmetric Noise Rate} \\ \cmidrule(lr){3-3}\cmidrule(lr){4-7} \cmidrule(lr){8-9}
 & & 0\% & 20\% & 40\% & 60\% & 80\% & 20\% & 40\% \\
 \midrule
 \multirow{8}{5em}{CIFAR-10} & CE & 94.35 $\pm$ 0.10 & 88.17 $\pm$ 0.19 & 82.66 $\pm$ 0.37 & 75.75 $\pm$ 0.29 & 64.28 $\pm$ 1.15 & 89.28 $\pm$ 0.20 & 85.26 $\pm$ 0.67  \\
 & BS & 91.18 $\pm$ 0.22 & 86.50 $\pm$ 0.24 & 82.90 $\pm$ 0.31 & 75.59 $\pm$ 0.51 & \textbf{70.68 $\pm$ 24.17} & 89.27 $\pm$ 0.12 & 85.77 $\pm$ 0.72  \\
 & LS & 94.22 $\pm$ 0.12 & 90.20 $\pm$ 0.18 & 84.42 $\pm$ 0.06 & 77.29 $\pm$ 0.17 & 62.16 $\pm$ 2.07 & 89.31 $\pm$ 0.22 & 85.76 $\pm$ 0.49  \\
 & SCE & 94.65 $\pm$ 0.18 & 91.11 $\pm$ 0.12 & 88.98 $\pm$ 0.14 & 84.70 $\pm$ 0.20 & 75.73 $\pm$ 0.20 & 90.16 $\pm$ 0.19 & 83.69 $\pm$ 0.36  \\
& GCE & 94.00 $\pm$ 0.08 & 91.12 $\pm$ 0.07 & 89.00 $\pm$ 0.15 & 84.58 $\pm$ 0.17 & 75.86 $\pm$ 0.41 & 89.07 $\pm$ 0.27 & 84.88 $\pm$ 0.51  \\
 & NCE+RCE & 92.99 $\pm$ 0.16 & 91.15 $\pm$ 0.17 & 88.00 $\pm$ 0.15 & 82.01 $\pm$ 0.33 & 73.24 $\pm$ 0.69 & 91.09 $\pm$ 0.10 & 85.27 $\pm$ 0.37  \\
  & JS & \textbf{94.95 $\pm$ 0.06} & 91.46 $\pm$ 0.10 & 89.31 $\pm$ 0.09 & 84.77 $\pm$ 0.11 & 70.57 $\pm$ 0.68 & 87.47 $\pm$ 0.07 & 84.26 $\pm$ 0.21  \\
 & GJS & 94.78 $\pm$ 0.06 & \textbf{94.24 $\pm$ 0.12} & \textbf{91.21 $\pm$ 0.05} & \textbf{90.36 $\pm$ 0.08} & \textbf{78.42 $\pm$ 0.29} & \textbf{91.88 $\pm$ 0.17} & \textbf{89.08 $\pm$ 0.36}  \\
\midrule
 \multirow{8}{5em}{CIFAR-100} & CE & 86.24 $\pm$ 0.49 & 71.33 $\pm$ 0.27 & 59.45 $\pm$ 0.51 & 46.67 $\pm$ 0.71 & 33.07 $\pm$ 1.96 & 78.26 $\pm$ 0.10 & 71.94 $\pm$ 0.30  \\
 & BS & 86.04 $\pm$ 0.32 & 77.59 $\pm$ 0.54 & 70.70 $\pm$ 0.50 & 65.44 $\pm$ 1.60 & 33.78 $\pm$ 1.77 & 76.45 $\pm$ 0.57 & 72.54 $\pm$ 0.74  \\
 & LS & 88.40 $\pm$ 0.07 & 80.83 $\pm$ 0.11 & 73.18 $\pm$ 0.09 & 59.11 $\pm$ 0.10 & 36.69 $\pm$ 0.39 & 78.78 $\pm$ 0.49 & 67.76 $\pm$ 0.37  \\
 & SCE & 85.72 $\pm$ 0.11 & 79.60 $\pm$ 0.20 & 71.50 $\pm$ 0.24 & 61.63 $\pm$ 0.80 & 39.98 $\pm$ 1.08 & 75.40 $\pm$ 0.70 &  63.66 $\pm$ 0.33 \\
 & GCE & 85.63 $\pm$ 0.19 & 82.22 $\pm$ 0.14 & 77.69 $\pm$ 0.13 & 68.00 $\pm$ 0.25 & 53.28 $\pm$ 0.83 & 76.32 $\pm$ 0.21 & 64.77 $\pm$ 0.43  \\
 & NCE+RCE & 78.14 $\pm$ 0.16 & 75.04 $\pm$ 0.19 & 70.59 $\pm$ 0.29 & 63.60 $\pm$ 0.41 & 43.63 $\pm$ 2.00 & 74.07 $\pm$ 0.31 & 64.47 $\pm$ 0.30  \\
 & JS & 85.99 $\pm$ 0.24 & 82.58 $\pm$ 0.28 & 75.92 $\pm$ 0.38 & 66.80 $\pm$ 0.58 & 48.09 $\pm$ 1.14 & 78.25 $\pm$ 0.14 & 66.94 $\pm$ 0.46  \\
 & GJS & \textbf{89.54 $\pm$ 0.10} & \textbf{87.73 $\pm$ 0.13} & \textbf{85.67 $\pm$ 0.15} & \textbf{79.09 $\pm$ 0.19} & \textbf{59.74 $\pm$ 0.70} & \textbf{84.52 $\pm$ 0.13} & \textbf{74.98 $\pm$ 0.25}  \\
 \bottomrule
\end{tabular}
\end{center}
\end{table*}

\section{Proofs}
\label{sup:proofs}

\subsection{JS's Connection to CE and MAE}
\propJsCeMae*
\begin{proof}[Proof of Proposition \ref{prop:JS_limits_CE_MAE}]
We want to show
\begin{align}
     &\lim_{\pi_1 \rightarrow 0} \LJS(\dve{y}, \vp) = \lim_{\pi_1 \rightarrow 0} \frac{\DJS(\dve{y}, \vp)}{H(1-\pi_1)} = H(\dve{y}, \vp) & \\
     &\lim_{\pi_1 \rightarrow 1} \LJS(\dve{y}, \vp) = \lim_{\pi_1 \rightarrow 1} \frac{\DJS(\dve{y}, \vp)}{H(1-\pi_1)} = \frac{1}{2}\Vert\dve{y}-\vp\Vert_1 &
\end{align}
More specifically, we have $\DJS(\dve{y}, \vp) = \evpi_1\KL(\dve{y} \Vert \vm) + \evpi_2\KL(\vp \Vert \vm)$, where $\vm = \pi_1\dve{y} + \pi_2\vp$, and
\begin{align}
     &\lim_{\pi_1 \rightarrow 0} \frac{\pi_1\DKL{\dve{y}}{\vm}}{H(1-\pi_1)} = H(\dve{y}, \vp) & \label{eq:js-to-ce} \\
     &\lim_{\pi_1 \rightarrow 1} \frac{\pi_2\DKL{\vp}{\vm}}{H(1-\pi_1)} = \frac{1}{2}\Vert\dve{y}-\vp\Vert_1 & \label{eq:js-to-mae}
\end{align}

First, the we prove Equations \ref{eq:js-to-ce} and \ref{eq:js-to-mae}, then show that the other two limits are zero. 

Proof of Equation \ref{eq:js-to-ce}.
\begin{flalign}
    \lim_{\pi_1 \rightarrow 0}\frac{\pi_1\DKL{\dve{y}}{\vm}}{H(1-\pi_1)} &= \lim_{\pi_1 \rightarrow 0}\frac{-\pi_1\log{(m_y)}}{-(1-\pi_1)\log{(1-\pi_1)}} & \\
    &= \lim_{\pi_1 \rightarrow 0}\log{(m_y)}\frac{1}{1-\pi_1}\frac{\pi_1}{\log{(1-\pi_1)}} &  \\
    &= \lim_{\pi_1 \rightarrow 0}\log{(m_y)}\frac{1}{1-\pi_1} \cdot - (1-\pi_1) &  \\
    &= \log{p_y} \cdot 1 \cdot -1 = H(\dve{y}, \dvp{2}) &
\end{flalign}
where we used L'H\^{o}pital's rule for $\lim_{\pi_1 \rightarrow 0}\frac{\pi_1}{\log{(1-\pi_1)}}$ which is indeterminate of the form $\frac{0}{0}$.

Proof of Equation \ref{eq:js-to-mae}.
Before taking the limit, we first rewrite the equation
\begin{flalign}
 \frac{\pi_2\DKL{\vp}{\vm}}{H(1-\pi_1)} &= -\frac{1}{\log{(1-\pi_1)}} \sum_{k=1}^K p_k\log{\frac{p_k}{m_k}}  &\\
&= -\frac{1}{\log{(1-\pi_1)}} \Big[p_y\log{\frac{p_y}{m_y}} + \sum_{k\not=y}^K  p_k\log{\frac{p_k}{(1-\pi_1)p_k}}\Big]  &\\
&= -\frac{1}{\log{(1-\pi_1)}} \Big[p_y\log{\frac{p_y}{m_y}} - \log{(1-\pi_1)}\sum_{k\not=y}^K  p_k\Big]  &\\
&= -\frac{1}{\log{(1-\pi_1)}} \Big[p_y\log{\frac{p_y}{m_y}} - \log{(1-\pi_1)}(1-p_y)\Big]  &\\
&=  -p_y\log{\frac{p_y}{m_y}}\frac{1}{\log{(1-\pi_1)}} + 1-p_y   & \label{eq:2ndTermSimp}
\end{flalign}
Now, we take the limit
\begin{flalign}
\lim_{\pi_1 \rightarrow 1} \frac{\pi_2\DKL{\vp}{\vm}}{H(1-\pi_1)} &= \lim_{\pi_1 \rightarrow 1}  -p_y\log{\frac{p_y}{m_y}}\frac{1}{\log{(1-\pi_1)}} + 1-p_y & \\
&= 0\cdot 0 + 1-p_y & \\
&= \frac{1}{2}(1-p_y + 1-p_y) \\
&= \frac{1}{2}(1-p_y + \sum_{k\not=y}^K p_k) \\
&= \frac{1}{2}\sum_{k=1}^K \left| \edve{y}{k} -  p_k \right| = \frac{1}{2}\Vert\dve{y}-\vp\Vert_1
\end{flalign}

What is left to show is that the last two terms goes to zero in their respective limits. 
\begin{flalign}
    \lim_{\pi_1 \rightarrow 1}\frac{\pi_1\DKL{\dve{y}}{\vm}}{H(1-\pi_1)} &= \lim_{\pi_1 \rightarrow 1}\frac{-\pi_1\log{(m_y)}}{-(1-\pi_1)\log{(1-\pi_1)}} & \\
 &= \lim_{\pi_1 \rightarrow 1}\frac{-\pi_1\log{(\pi_1 + (1-\pi_1)p_y)}}{-(1-\pi_1)\log{(1-\pi_1)}} & \\
  &= \lim_{\pi_1 \rightarrow 1}\frac{\pi_1}{\log{(1-\pi_1)}} \frac{\log{(\pi_1 + (1-\pi_1)p_y)}}{1-\pi_1} & \\
 &= 0 \cdot (p_y-1) = 0
\end{flalign}

Finally, the last term. Starting from Equation~\ref{eq:2ndTermSimp}, we get
\begin{flalign}
\lim_{\pi_1 \rightarrow 0} \frac{\pi_2\DKL{\dvp{2}}{\vm}}{H(1-\pi_1)} &= \lim_{\pi_1 \rightarrow 0} -p_y\frac{\log{\frac{p_y}{m_y}}}{\log{(1-\pi_1)}} + 1-p_y \\
&= \lim_{\pi_1 \rightarrow 0}-p_y\Big(-\frac{1 - p_y}{\pi_1 + (1-\pi_1)p_y} \cdot -(1-\pi_1) \Big)  + 1-p_y \\
&= \lim_{\pi_1 \rightarrow 0}-p_y\Big(\frac{(1 - p_y)(1-\pi_1) }{\pi_1 + (1-\pi_1)p_y} \Big)  + 1-p_y \\
&= \lim_{\pi_1 \rightarrow 0}p_y\Big(\frac{-1 +\pi_1 + (1-\pi_1)p_y  }{\pi_1 + (1-\pi_1)p_y} \Big)  + 1-p_y \\
&= \lim_{\pi_1 \rightarrow 0}p_y\Big(\frac{-1}{\pi_1 + (1-\pi_1)p_y} + 1 \Big)  + 1-p_y \\
&= - 1+p_y + 1-p_y = 0  & 
\end{flalign}
where L'H\^{o}pital's rule was used for $\lim_{\pi_1 \rightarrow 0} -p_y\frac{\log{\frac{p_y}{m_y}}}{\log{(1-\pi_1)}}$ which is indeterminate of the form $\frac{0}{0}$.
\end{proof}

\subsection{GJS's Connection to Consistency Regularization}
\label{sup:gjs-consistency}
\propGjsJsConsistency*
\begin{proof}[Proof of Proposition \ref{prop:GJS-JS-Consistency}]
The Generalized Jensen-Shannon divergence can be simplified as below
\begin{flalign}
    &\DGJS(\dve{y}, \dvp{2},\dots,\dvp{M}) = H(\pi_1\dve{y} + (1-\pi_1)\dvm) - \sum_{j=2}^M\pi_j H(\dvp{j}) \\
    &= H(\pi_1 + (1-\pi_1)m_y) + \sum_{i\not=y}^K H((1-\pi_1)m_{i}) - \sum_{j=2}^M\pi_j H(\dvp{j})  \\
    &= / H(\pi_2p_i) = p_iH(\pi_2) + \pi_2H(p_i) / \label{eq:entropyProd} \\
    &= H(\pi_1 + (1-\pi_1)m_y) + \sum_{i\not=y}^K [m_iH(1-\pi_1) + (1-\pi_1)H(m_i)] - \sum_{j=2}^M\pi_j H(\dvp{j})  \\
    &= H(\pi_1 + (1-\pi_1)m_y) + \sum_{i\not=y}^K [m_iH(1-\pi_1)] - (1-\pi_1)H(m_y) \\
    &+ (1-\pi_1)\Big( H(\dvm)  - \frac{1}{1-\pi_1}\sum_{j=2}^M\pi_j H(\dvp{j})\Big)  \\
    &= H(\pi_1 + (1-\pi_1)m_y) + \sum_{i\not=y}^K [m_iH(1-\pi_1) + (1-\pi_1)(H(m_i) - H(m_i))]  \\
    &- (1-\pi_1)H(m_y) + (1-\pi_1)D_{\mathrm{GJS}_{\vpi''}}(\dvp{2},\dots,\dvp{M})  \\
    &= / \text{ Equation \ref{eq:entropyProd}} / \\
    &= H(\pi_1 + (1-\pi_1)m_y) + \sum_{i\not=y}^K H((1-\pi_1)m_i) - (1-\pi_1)H(\dvm) \\
    &+ (1-\pi_1)D_{\mathrm{GJS}_{\vpi''}}(\dvp{2},\dots,\dvp{M})  \\
    &= H(\pi_1\dve{y} + (1-\pi_1)\dvm) - (1-\pi_1)H(\dvm) + (1-\pi_1)D_{\mathrm{GJS}_{\vpi''}}(\dvp{2},\dots,\dvp{M})  \\
    &= D_{\mathrm{JS}_{\vpi'}}(\dve{y},\dvm) + (1-\pi_1)D_{\mathrm{GJS}_{\vpi''}}(\dvp{2},\dots,\dvp{M})  
\end{flalign}
where $\vpi'=[\pi_1, 1-\pi_1]$ and $\vpi''=[\pi_2, \dots, \pi_M]/(1-\pi_1)$.
\end{proof}
That is, when using onehot labels, the generalized Jensen-Shannon divergence is a combination of two terms, one term encourages the mean prediction to be similar to the label and another term that encourages consistency between the predictions. For $M=2$, the consistency term is zero.

\subsection{Noise Robustness}
\label{sup:noisetheorems}

The proofs of the theorems in this sections are generalizations of the proofs in by Zhang~\etal~\citep{Zhang_NeurIPS_2018_Generalized_CE}. The original theorems are specific to their particular GCE loss and cannot directly be used for other loss functions. We generalize the theorems to be useful for any loss function satisfying certain conditions(bounded and conditions in Lemma \ref{L:asymConditions}). To be able to use the theorems for $\DGJSnp$, we also generalize them to work for more than a single predictive distribution.
Here, we use $(\vx,y)$ to denote a sample from $\mathcal{D}$ and $(\vx,\tilde{y})$ to denote a sample from $\mathcal{D}_\eta$. Let $\eta_{ij}$ denote the probability that a sample of class $i$ was changed to class $j$ due to noise. 

\subsubsection{Symmetric Noise}
\label{sup:sec:sym-proof}
\tSym*
\begin{proof}[Proof of Theorem \ref{T:sym}]
For any function, $f$, mapping an input $\vx \in \sX$ to $\Delta^{K-1}$, we have 
\begin{align*}
    R_{\mathcal{L}}(f) = \mathbb{E}_{\mathcal{D}}[\mathcal{L}(\dve{y}, \vx, f)] = \mathbb{E}_{\vx,y}[\mathcal{L}(\dve{y}, \vx, f)]
\end{align*}
and for uniform noise with noise rate $\eta$, the probability of a class not changing label due to noise is $\eta_{ii}=1-\eta$, while the probability of changing from one class to any other is $\eta_{ij}= \frac{\eta}{K-1}$. Therefore,
\begin{align*}
R^{\eta}_{\mathcal{L}}(f) &= \mathbb{E}_{\mathcal{D}_\eta}[\mathcal{L}(\dve{\tilde{y}},\vx, f)] = \mathbb{E}_{\vx,\tilde{y} }[\mathcal{L}(\dve{\tilde{y}}, \vx, f)] \\
&= \mathbb{E}_{\vx}\mathbb{E}_{y|\vx}\mathbb{E}_{\tilde{y} |y,\vx}[\mathcal{L}(\dve{\tilde{y}}, \vx, f)] \\
&= \mathbb{E}_{\vx}\mathbb{E}_{y|\vx}[(1 - \eta)\mathcal{L}(\dve{y}, \vx, f) + \frac{\eta}{K-1}\sum_{i\not=y}^K\mathcal{L}(\dve{i}, \vx, f)] \\
&= \mathbb{E}_{\vx}\mathbb{E}_{y|\vx}\Bigg[(1 - \eta)\mathcal{L}(\dve{y}, \vx, f) + \frac{\eta}{K-1}\Big(\sum_{i=1}^K\mathcal{L}(\dve{i}, \vx, f) - \mathcal{L}(\dve{y}, \vx, f)\Big)\Bigg] \\
&= \Big(1 - \eta - \frac{\eta}{K-1}\Big)R_{\mathcal{L}}(f) + \frac{\eta}{K-1}\mathbb{E}_{\vx}\mathbb{E}_{y|\vx}\Bigg[\sum_{i=1}^K\mathcal{L}(\dve{i}, \vx, f)\Bigg] \\
&= \Big(1 - \frac{\eta K}{K-1}\Big)R_{\mathcal{L}}(f) + \frac{\eta}{K-1}\mathbb{E}_{\vx}\mathbb{E}_{y|\vx}\Bigg[\sum_{i=1}^K\mathcal{L}(\dve{i}, \vx, f)\Bigg]
\end{align*}
Using the bounds $B_L \leq \sum_{k=1}^K \mathcal{L}(\dve{k}, \vx, f) \leq B_U$, we get:
\begin{align*}
\Big(1 - \frac{\eta K}{K-1}\Big)R_{\mathcal{L}}(f) + \frac{\eta B_L}{K-1} \leq R^{\eta}_{\mathcal{L}}(f) \leq \Big(1 - \frac{\eta K}{K-1}\Big)R_{\mathcal{L}}(f) + \frac{\eta B_U}{K-1}
\end{align*}
With these bounds, the difference between $R^{\eta}_{\mathcal{L}}(f^*)$ and $R^{\eta}_{\mathcal{L}}(f^*_\eta)$ can be bounded as follows
\begin{align*}
    R^{\eta}_{\mathcal{L}}(f^*) - R^{\eta}_{\mathcal{L}}(f^*_\eta) &\leq \Big(1 - \frac{\eta K}{K-1}\Big)R_{\mathcal{L}}(f^*) + \frac{\eta B_U}{K-1} - \Bigg(\Big(1 - \frac{\eta K}{K-1}\Big)R_{\mathcal{L}}(f^*_\eta) + \frac{\eta B_L}{K-1}\Bigg) = \\
    &= \Big(1 - \frac{\eta K}{K-1}\Big)(R_{\mathcal{L}}(f^*) - R_{\mathcal{L}}(f^*_\eta)) + \frac{\eta (B_U-B_L)}{K-1} \leq \frac{\eta (B_U-B_L)}{K-1}
\end{align*}
where the last inequality follows from the assumption on the noise rate, $(1 - \frac{\eta K}{K-1}) > 0$,  and that $f^*$ is the minimizer of $R_{\mathcal{L}}(f)$ so $R_{\mathcal{L}}(f^*) - R_{\mathcal{L}}(f^*_\eta) \leq 0$. Similarly, since $f^*_\eta$ is the minimizer of $R^\eta_{\mathcal{L}}(f)$, we have $R^\eta_{\mathcal{L}}(f^*) - R^\eta_{\mathcal{L}}(f^*_\eta) \geq 0$, which is the lower bound.
\end{proof}

\subsubsection{Asymmetric Noise}
\label{sup:sec:asym-proof}
\begin{lemma}
\label{L:asymConditions}
Consider the following conditions for a loss with label $\dve{i}$, for any $i \in \{1, 2, \dots, K\}$ and M-1 distributions $\dvp{2}, \dots, \dvp{M} \in \Delta^{K-1}$:
\begin{align*}
&\textit{i) } \mathcal{L}(\dve{i}, \dvp{2}, \dots, \dvp{M})=0 \iff \dvp{2}, \dots, \dvp{M}=\dve{i}, \\
&\textit{ii) } 0 \leq \mathcal{L}(\dve{i}, \dvp{2}, \dots, \dvp{M}) \leq C_1, \\
&\textit{iii) } \mathcal{L}(\dve{i}, \dve{j}, \dots, \dve{j}) = C_2 \leq C_1, \text{ with } i\not=j.
\end{align*}
where $C_1, C_2$ are constants. 
\end{lemma}

\begin{theorem}
\label{T:asym}
 Let $\mathcal{L}$ be any loss function satisfying the conditions in Lemma \ref{L:asymConditions}. Under class dependent noise, when the probability of the noise not changing label is larger than changing it to any other class($\eta_{yi} < \eta_{yy}$, for all $i \not= y$, with $y$ being the true label), and if $R^\eta_{\mathcal{L}}(f^*)=0$, then
\begin{align}
    0 \leq R^\eta_{\mathcal{L}}(f^*) - R^\eta_{\mathcal{L}}(f^*_\eta) \leq (B_U-B_L)\mathbb{E}_\mathcal{D}[\eta_{yy}] + (C_1-C_2)\mathbb{E}_\mathcal{D}[\sum_{i\not=y}^K(\eta_{yy} - \eta_{yi})],
\end{align}
where $B_L \leq \sum_{i=1}^K \mathcal{L}(\dve{i}, \vx, f) \leq B_U$  for all $\vx$ and $f$, $f^*$ is the global minimizer of $R_{\mathcal{L}}(f)$, and $f^*_\eta$ is the global minimizer of $R^\eta_{\mathcal{L}}(f)$.
\end{theorem}

\begin{proof}[Proof of Theorem \ref{T:asym}]
For class dependent noisy(asymmetric) and any function, $f$, mapping an input $\vx \in \sX$ to $\Delta^{K-1}$, we have 
\begin{align*}
R^\eta_{\mathcal{L}}(f) &= \mathbb{E}_\mathcal{D}[\eta_{yy}\mathcal{L}(\dve{y}, \vx, f)] + \mathbb{E}_\mathcal{D}[\sum_{i\not=y}^K\eta_{yi}\mathcal{L}(\dve{i}, \vx, f)] \\
&= \mathbb{E}_\mathcal{D}[\eta_{yy}\Big(\sum_{i=1}^K\mathcal{L}(\dve{i}, \vx, f) - \sum_{i\not=y}^K\mathcal{L}(\dve{i}, \vx, f)\Big)] + \mathbb{E}_\mathcal{D}[\sum_{i\not=y}^K\eta_{yi}\mathcal{L}(\dve{i}, \vx, f)] \\
&= \mathbb{E}_\mathcal{D}[\eta_{yy}\sum_{i=1}^K\mathcal{L}(\dve{i}, \vx, f)] - \mathbb{E}_\mathcal{D}[\sum_{i\not=y}^K(\eta_{yy} - \eta_{yi})\mathcal{L}(\dve{i}, \vx, f)] \\
\end{align*}
By using the bounds $B_L, B_U$ we get
\begin{align*}
R^\eta_{\mathcal{L}}(f) &\leq B_U\mathbb{E}_\mathcal{D}[\eta_{yy}] - \mathbb{E}_\mathcal{D}[\sum_{i\not=y}^K(\eta_{yy} - \eta_{yi})\mathcal{L}(\dve{i}, \vx, f)] \\
R^\eta_{\mathcal{L}}(f) &\geq B_L\mathbb{E}_\mathcal{D}[\eta_{yy}] - \mathbb{E}_\mathcal{D}[\sum_{i\not=y}^K(\eta_{yy} - \eta_{yi})\mathcal{L}(\dve{i}, \vx, f)]
\end{align*}
Hence,
\begin{align}
    R^\eta_{\mathcal{L}}(f^*) - R^\eta_{\mathcal{L}}(f^*_\eta) &\leq (B_U-B_L)\mathbb{E}_\mathcal{D}[\eta_{yy}] + \\
    &+ \mathbb{E}_\mathcal{D}[\sum_{i\not=y}^K(\eta_{yy} - \eta_{yi})\big(\mathcal{L}(\dve{i}, \vx, f^*_\eta) - \mathcal{L}(\dve{i}, \vx, f^*)\big)] \nonumber
    \label{eq:asymDiff}
\end{align}
From the assumption that $R_{\mathcal{L}}(f^*)=0$, we have $\mathcal{L}(\dve{y}, \vx, f^*)=0$. Using the conditions on the loss function from Lemma \ref{L:asymConditions}, for all $i\not=y$, we get
\begin{align*}
    \mathcal{L}(\dve{i}, \vx, f^*_\eta) - \mathcal{L}(\dve{i}, \vx, f^*) &= /\text{ }\mathcal{L}(\dve{y}, \vx, f^*)=0 \text{ and } \text{\textit{i)}  } / \\
    &=\mathcal{L}(\dve{i}, \vx, f^*_\eta) - \mathcal{L}(\dve{i}, \dve{y})  \\
    &= / \text{  \textit{iii)}  }  / \\
    &=\mathcal{L}(\dve{i}, \vx, f^*_\eta) - C_2 \\
    &= / \text{  \textit{ii)}  } / \\
    &\leq C_1 - C_2
\end{align*}
By our assumption on the noise rates, we have $\eta_{yy} - \eta_{yi} > 0$. We have
\begin{align*}
    R^\eta_{\mathcal{L}}(f^*) - R^\eta_{\mathcal{L}}(f^*_\eta) \leq (B_U-B_L)\mathbb{E}_\mathcal{D}[\eta_{yy}] + (C_1-C_2)\mathbb{E}_\mathcal{D}[\sum_{i\not=y}^K(\eta_{yy} - \eta_{yi})]
\end{align*}
Since $f^*_\eta$ is the global minimizer of $R^\eta_{\mathcal{L}}(f)$ we have $R^\eta_{\mathcal{L}}(f^*)-R^\eta_{\mathcal{L}}(f^*_\eta) \geq 0$, which is the lower bound.
\end{proof}

\begin{remark}
\label{r:js-asymConditions}
The generalized Jensen-Shannon Divergence satisfies the conditions in Lemma \ref{L:asymConditions}, with 
\begin{align}
    C_1=H(\vpi), \quad C_2=H(\pi_1) + H(1-\pi_1).\nonumber
\end{align}
\end{remark}

\begin{proof}[Proof of Remark \ref{r:js-asymConditions}]
i). Follows directly from Jensen's inequality for the Shannon entropy. 
ii). The lower bound follows directly from Jensen's inequality for the non-negative Shannon entropy. The upper bound is shown below
\begin{align*}
\GJS(\dvp{1}, \dvp{2}, \dots, \dvp{M}) &= \sum_{j=1}^M\pi_j \DKL{\dvp{j}}{\vm} \\
&= \sum_{j=1}^M\Bigg[\pi_j \sum_{l=1}^K  \xlogxy{\edvp{j}{l}}{m_l}\Bigg]   \\
    &=\sum_{j=1}^M\Bigg[\pi_j \sum_{l=1}^K \edvp{j}{l} \Bigg( \log{ \Big( \frac{\pi_j\edvp{j}{l}}{m_l} \Big) } + \log{\frac{1}{\pi_j}}  \Bigg) \Bigg] \\
    &= \sum_{j=1}^M\Bigg[\pi_j \sum_{l=1}^K \Big[\edvp{j}{l} \log{ \Big( \frac{\pi_j\edvp{j}{l}}{m_l} \Big) } -\edvp{j}{l}\log{\pi_j}\Big] \Bigg]    \\
    &= \sum_{j=1}^M\Bigg[-\pi_j\log{\pi_j} + \pi_j \sum_{l=1}^K  \xlogyz{\edvp{j}{l}}{\pi_j\edvp{j}{l}}{m_l}\Bigg]  \\
    &= \sum_{j=1}^M\Bigg[H(\pi_j) + \pi_j \sum_{l=1}^K \xlogyz{\edvp{j}{l}} {\pi_j\edvp{j}{l}} {m_l}\Bigg] \\
&= \sum_{j=1}^M \Bigg[H(\pi_j) + \pi_j \sum_{l=1}^K \xlogyz{\edvp{j}{l}} {\edvp{j}{l}} { \edvp{j}{l} + \frac{1}{\pi_j}\s{i\not=j}{M} \evpi_i \edvp{i}{l}} \Bigg]  \\
&\leq \sum_{j=1}^M H(\pi_j) = H(\vpi)
\end{align*}
where the inequality holds with equality iff $\frac{1}{\pi_j}\s{i\not=j}{M} \evpi_i \edvp{i}{l} = 0$ when $\edvp{j}{l} > 0$ for all $j\in\{1,2,\dots,M\}$ and $l\in\{1,2,\dots,K\}$. Hence, $\DGJSnp$ is bounded above by $H(\vpi)$.\newline
iii). Let the label be $\dve{i}$ and the other M-1 distributions be $\dve{j}$ with $i\not=j$ then 
\begin{align}
    \GJS
    = H\big(\pi_1\dve{i} + \sum_{l=2}^M \pi_l \dve{j}) - \pi_1H(\dve{i}) - \sum_{l=2}^M \pi_l H(\dve{j})  = H(\pi_1\dve{i} + (1-\pi_1)\dve{j})
\end{align}
Notably, $C_1=C_2$ for $M=2$.
\end{proof}

\subsubsection{Improving GJS Risk Difference Bounds}
\label{sup:js-gjs-same-risk-bound}
\propGjsBetterBounds*
\begin{proof}[Proof of Proposition \ref{prop:js-gjs-same-risk-bound}]

\leavevmode\newline
\textbf{Symmetric Noise}
From the proof of Theorem \ref{T:sym}, we have for any function, $f$, mapping an input $\vx \in \sX$ to $\Delta^{K-1}$ 
\begin{align*}
R^{\eta}_{\mathcal{L}}(f) = \Big(1 - \frac{\eta K}{K-1}\Big)R_{\mathcal{L}}(f) + \frac{\eta}{K-1}\mathbb{E}_{\vx}\mathbb{E}_{y|\vx}\Bigg[\sum_{i=1}^K\mathcal{L}(\dve{i}, \vx, f)\Bigg]
\end{align*}
Using Proposition \ref{prop:GJS-JS-Consistency} for $\DGJSnp$, we get
\begin{align*}
R^{\eta}_{\mathcal{L}_{\DGJSnp}}(f) &= \Big(1 - \frac{\eta K}{K-1}\Big)R_{\mathcal{L}_{\DGJSnp}}(f) + \frac{\eta}{K-1}\mathbb{E}_{\vx}\mathbb{E}_{y|\vx}\Bigg[\sum_{i=1}^K\mathcal{L}^{f}_{\mathrm{JS_{\vpi'}}}(\dve{i},\dvm) \Bigg] \\
&+ (1-\pi_1)\frac{\eta K}{K-1}\mathbb{E}_{\vx}\mathbb{E}_{y|\vx}\Bigg[\mathcal{L}^{f}_{\mathrm{GJS_{\vpi''}}}(\dvp{2},\dots,\dvp{M}) \Bigg]
\end{align*}

Let  $B_L^{\DJSnp}$, $B_U^{\DJSnp}$ be the lower and upper bound for $\DJSnp$~(M=2) in Proposition \ref{prop:js-bounds}. These bounds \ref{prop:js-bounds} holds for any $\dvp{2} \in \Delta^{K-1}$ and therefore also holds for $\dvm$. Hence, we have
\begin{align*}
R^{\eta}_{\mathcal{L}_{\DGJSnp}}(f) \geq \Big(1 - \frac{\eta K}{K-1}\Big)R_{\mathcal{L}_{\DGJSnp}}(f) + \frac{\eta B_L^{\DJSnp}}{K-1} + (1-\pi_1)\frac{\eta K}{K-1}\mathbb{E}_{\vx}\mathbb{E}_{y|\vx}\Bigg[\mathcal{L}^{f}_{\mathrm{GJS_{\vpi''}}}(\dvp{2},\dots,\dvp{M}) \Bigg] \\ R^{\eta}_{\mathcal{L}_{\DGJSnp}}(f) \leq \Big(1 - \frac{\eta K}{K-1}\Big)R_{\mathcal{L}_{\DGJSnp}}(f) + \frac{\eta B_U^{\DJSnp}}{K-1} + (1-\pi_1)\frac{\eta K}{K-1}\mathbb{E}_{\vx}\mathbb{E}_{y|\vx}\Bigg[\mathcal{L}^{f}_{\mathrm{GJS_{\vpi''}}}(\dvp{2},\dots,\dvp{M}) \Bigg]
\end{align*}
With these bounds, the difference between $R^{\eta}_{\mathcal{L}}(f^*)$ and $R^{\eta}_{\mathcal{L}}(f^*_\eta)$ can be bounded as follows
\begin{align*}
    R^{\eta}_{\mathcal{L}_{\DGJSnp}}(f^*) - R^{\eta}_{\mathcal{L}_{\DGJSnp}}(f^*_\eta) &\leq \Big(1 - \frac{\eta K}{K-1}\Big)(R_{\mathcal{L}_{\DGJSnp}}(f^*) - R_{\mathcal{L}_{\DGJSnp}}(f^*_\eta)) + \frac{\eta (B_U^{\DJSnp}-B_L^{\DJSnp})}{K-1} \\
    &+ \frac{(1-\pi_1)\eta K}{K-1}\mathbb{E}_{\vx}\mathbb{E}_{y|\vx}\Bigg[\mathcal{L}^{f^*}_{\mathrm{GJS_{\vpi''}}}(\dvp{2},\dots,\dvp{M}) - \mathcal{L}^{f^*_\eta}_{\mathrm{GJS_{\vpi''}}}(\dvp{2},\dots,\dvp{M})\Bigg] \\
    &\leq \frac{\eta (B_U^{\DJSnp}-B_L^{\DJSnp})}{K-1}
\end{align*}
where the last inequality follows from the assumption on the noise rate, $(1 - \frac{\eta K}{K-1}) > 0$,  that $f^*$ is the minimizer of $R_{\mathcal{L}}(f)$ so $R_{\mathcal{L}}(f^*) - R_{\mathcal{L}}(f^*_\eta) \leq 0$, and the assumption on the consistency of $f^*$ and $f^*_\eta$. Similarly, since $f^*_\eta$ is the minimizer of $R^\eta_{\mathcal{L}}(f)$, we have $R^\eta_{\mathcal{L}}(f^*) - R^\eta_{\mathcal{L}}(f^*_\eta) \geq 0$, which is the lower bound. Hence, we have shown that $\LJS$ and $\LGJS$ have the same bounds for the risk difference for symmetric noise.
\newline\newline
\textbf{Asymmetric Noise}
For class dependent noisy(asymmetric) and any function, $f$, mapping an input $\vx \in \sX$ to $\Delta^{K-1}$, we have 
\begin{align*}
R^\eta_{\mathcal{L}_{\DGJSnp}}(f) &=\mathbb{E}_\mathcal{D}[\sum_{i=1}^K \eta_{yi}\mathcal{L}_{\DGJSnp}(\dve{i}, \vx, f)] \\ &=\mathbb{E}_\mathcal{D}[\eta_{yy}\mathcal{L}^{f}_{\mathrm{JS_{\vpi'}}}(\dve{y},\dvm) + \sum_{i\not=y}^K\eta_{yi}\mathcal{L}^{f}_{\mathrm{JS_{\vpi'}}}(\dve{i},\dvm) \\
&+ (1-\pi_1)\mathcal{L}^{f}_{\mathrm{GJS_{\vpi''}}}(\dvp{2},\dots,\dvp{M})] \\
&=\mathbb{E}_\mathcal{D}[\eta_{yy}\Big(\sum_{i=1}^K\mathcal{L}^{f}_{\mathrm{JS_{\vpi'}}}(\dve{i},\dvm) - \sum_{i\not=y}^K\mathcal{L}^{f}_{\mathrm{JS_{\vpi'}}}(\dve{i},\dvm)\Big) + \sum_{i\not=y}^K\eta_{yi}\mathcal{L}^{f}_{\mathrm{JS_{\vpi'}}}(\dve{i},\dvm) \\
&+ (1-\pi_1)\mathcal{L}^{f}_{\mathrm{GJS_{\vpi''}}}(\dvp{2},\dots,\dvp{M})] \\
&=\mathbb{E}_\mathcal{D}[\eta_{yy}\sum_{i=1}^K\mathcal{L}^{f}_{\mathrm{JS_{\vpi'}}}(\dve{i},\dvm) - \sum_{i\not=y}^K(\eta_{yy}-\eta_{yi})\mathcal{L}^{f}_{\mathrm{JS_{\vpi'}}}(\dve{i},\dvm) \\
&+ (1-\pi_1)\mathcal{L}^{f}_{\mathrm{GJS_{\vpi''}}}(\dvp{2},\dots,\dvp{M})] \\
\end{align*}
where Proposition \ref{prop:GJS-JS-Consistency} was used to separate $\DGJSnp$ into a $\DJSnp$ and a consistency term.
By using the bounds $B_L^{\DJSnp}, B_U^{\DJSnp}$ we get
\begin{align*}
R^\eta_{\mathcal{L}_{\DGJSnp}}(f) &\leq \mathbb{E}_\mathcal{D}[\eta_{yy}B_U^{\DJSnp} - \sum_{i\not=y}^K(\eta_{yy} - \eta_{yi})\mathcal{L}^{f}_{\mathrm{JS_{\vpi'}}}(\dve{i},\dvm) + (1-\pi_1)\mathcal{L}^{f}_{\mathrm{GJS_{\vpi''}}}(\dvp{2},\dots,\dvp{M})] \\
R^\eta_{\mathcal{L}_{\DGJSnp}}(f) &\geq \mathbb{E}_\mathcal{D}[\eta_{yy}B_L^{\DJSnp} - \sum_{i\not=y}^K(\eta_{yy} - \eta_{yi})\mathcal{L}^{f}_{\mathrm{JS_{\vpi'}}}(\dve{i},\dvm) +  (1-\pi_1)\mathcal{L}^{f}_{\mathrm{GJS_{\vpi''}}}(\dvp{2},\dots,\dvp{M})]
\end{align*}
Hence,
\begin{align}
    R^\eta_{\mathcal{L}}(f^*) - R^\eta_{\mathcal{L}}(f^*_\eta) &\leq (B_U^{\DJSnp}-B_L^{\DJSnp})\mathbb{E}_\mathcal{D}[\eta_{yy}] \nonumber  \\
    &+ \mathbb{E}_\mathcal{D}[\sum_{i\not=y}^K(\eta_{yy} - \eta_{yi})\big(\mathcal{L}^{f^*_\eta}_{\mathrm{JS_{\vpi'}}}(\dve{i},\dvm) - \mathcal{L}^{f^*}_{\mathrm{JS_{\vpi'}}}(\dve{i},\dvm)\big)] \label{eq:asymRiskDiffTerm2}  \\
    &+ (1-\pi_1)\Big(\mathbb{E}_\mathcal{D}[\mathcal{L}^{f^*}_{\mathrm{GJS_{\vpi''}}}(\dvp{2},\dots,\dvp{M})] - \mathbb{E}_\mathcal{D}[\mathcal{L}^{f^*_\eta}_{\mathrm{GJS_{\vpi''}}}(\dvp{2},\dots,\dvp{M})]\Big)
    \label{eq:asymRiskDiffTerm3}
\end{align}
From the assumption that $R_{\mathcal{L}_{\DGJSnp}}(f^*)=0$, we have $\mathcal{L}_{\DGJSnp}(\dve{y}, \vx, f^*)=0$. 
Using the conditions on the loss function from Lemma \ref{L:asymConditions}, for all $i\not=y$, we get
\begin{align*}
    \mathcal{L}^{f^*_\eta}_{\mathrm{JS_{\vpi'}}}(\dve{i},\dvm) - \mathcal{L}^{f^*}_{\mathrm{JS_{\vpi'}}}(\dve{i},\dvm) &= /\text{ }\mathcal{L}_{\DGJSnp}(\dve{y}, \vx, f^*)=0 \text{ and } \text{\textit{i)}  } / \\
    &=\mathcal{L}^{f^*_\eta}_{\mathrm{JS_{\vpi'}}}(\dve{i},\dvm) - \mathcal{L}^{f^*}_{\mathrm{JS_{\vpi'}}}(\dve{i}, \dve{y})  \\
    &= / \text{  \textit{iii)} and Remark \ref{r:js-asymConditions} }  / \\
    &=\mathcal{L}^{f^*_\eta}_{\mathrm{JS_{\vpi'}}}(\dve{i},\dvm) - C_1 \\
    &\leq 0
\end{align*}
From above and our assumption on the noise rates ($\eta_{yy} - \eta_{yi} > 0$), we have that the term in Equation~\ref{eq:asymRiskDiffTerm2} is less or equal to zero. Due to the assumption on the consistency of $f^*$ and $f^*_\eta$ in Proposition \ref{prop:js-gjs-same-risk-bound}, this is also the case for the term in Equation \ref{eq:asymRiskDiffTerm3}. We have
\begin{align*}
    R^\eta_{\mathcal{L}_{\DGJSnp}}(f^*) - R^\eta_{\mathcal{L}_{\DGJSnp}}(f^*_\eta) \leq (B_U^{\DJSnp}-B_L^{\DJSnp})\mathbb{E}_\mathcal{D}[\eta_{yy}]
\end{align*}
Since $f^*_\eta$ is the global minimizer of $R^\eta_{\mathcal{L}_{\DGJSnp}}(f)$ we have $R^\eta_{\mathcal{L}_{\DGJSnp}}(f^*)-R^\eta_{\mathcal{L}_{\DGJSnp}}(f^*_\eta) \geq 0$, which is the lower bound. Hence, we have shown that $\LJS$ and $\LGJS$ have the same bounds for the risk difference for asymmetric noise.
\end{proof}

\subsection{Bounds}
\label{sup:bounds}
In this section, we first introduce some useful definitions and relate them to $\DJSnp$. Then, the bounds for JS and GJS are proven. 

\subsubsection{Another Definition of Jensen-Shannon divergence}
\begin{align} 
    & f_{\pi_1}(t) \coloneqq \Big[ H(\pi_1t + 1-\pi_1) -\pi_1H(t)\Big], t > 0 \label{eq:fdivJs}\\
    & f_{\pi_1}(0) \coloneqq \lim_{t\rightarrow 0} f_{\pi_1}(t) \\
    &0 f_{\pi_1} \Big(\frac{0}{0}\Big) \coloneqq 0 , \label{eq:0f00}\\  
    &0 f_{\pi_1} (0) \coloneqq 0 \label{eq:0f0} 
\end{align}
\begin{remark}
The Jensen-Shannon divergence can be rewritten using Equation \ref{eq:fdivJs} as follows
\label{r:js-fdiv}
\begin{align}
     \JS(\dvp{1}, \dvp{2}) = \sum_{k=1}^K \edvp{2}{k}f_{\pi_1}{\Bigg( \frac{\edvp{1}{k}}{\edvp{2}{k}}\Bigg)}
\end{align}
\end{remark}
\begin{proof}[Proof of Remark \ref{r:js-fdiv}]
\begin{align}
    &\sum_{k=1}^K \edvp{2}{k} f_{\pi_1}\Bigg(\frac{\edvp{1}{k}}{\edvp{2}{k}}\Bigg) = \sum_{k=1}^K \edvp{2}{k} \Big[ \pi \frac{\edvp{1}{k}}{\edvp{2}{k}}\log(\frac{\edvp{1}{k}}{\edvp{2}{k}}) -(\pi \frac{\edvp{1}{k}}{\edvp{2}{k}} + 1-\pi)\log(\pi \frac{\edvp{1}{k}}{\edvp{2}{k}} + 1-\pi)\Big] \\
    &= \sum_{k=1}^K \pi \edvp{1}{k}\log(\frac{\edvp{1}{k}}{\edvp{2}{k}}) -(\pi \edvp{1}{k} + (1-\pi)\edvp{2}{k})\log( \frac{\pi \edvp{1}{k} + (1-\pi)\edvp{2}{k} }{\edvp{2}{k}}) \\
    &= \sum_{k=1}^K \pi \edvp{1}{k}\log(\frac{\edvp{1}{k}}{\edvp{2}{k}}) -\pi \edvp{1}{k}\log( \frac{\pi \edvp{1}{k} + (1-\pi)\edvp{2}{k} }{\edvp{2}{k}}) - (1-\pi)\edvp{2}{k}\log( \frac{\pi \edvp{1}{k} + (1-\pi)\edvp{2}{k} }{\edvp{2}{k}}) \\
    &= \sum_{k=1}^K \pi \edvp{1}{k}\log(\frac{\edvp{1}{k}}{\pi \edvp{1}{k} + (1-\pi)\edvp{2}{k}}) + (1-\pi)\edvp{2}{k}\log( \frac{\edvp{2}{k}}{\pi \edvp{1}{k} + (1-\pi)\edvp{2}{k} }) \\
    &= \sum_{k=1}^K \pi \KL\Big(\edvp{1}{k}, \pi \edvp{1}{k} + (1-\pi)\edvp{2}{k}\Big) + (1-\pi)\KL\Big(\edvp{2}{k}, \pi \edvp{1}{k} + (1-\pi)\edvp{2}{k}\Big) \\
    &= \JS(\dvp{1}, \dvp{2})
\end{align}
\end{proof}

\subsubsection{Bounds for JS}
\begin{restatable}{proposition}{propJsBounds}
$\LJS$ has {\small $B_L \leq \sum_{k=1}^K \LJS(\dve{k}, f(\vx)) \leq B_U$} with
{\small
\begin{align}
    B_L = \sum_{k=1}^K \LJS(\dve{k}, \vu), \quad B_U = \sum_{k=1}^K \LJS(\dve{k}, \dve{1}) \nonumber
\end{align}
}%
where $\vu$ is the uniform distribution.
\label{prop:js-bounds}
\end{restatable}
\vspace{-1.5cm}
\begin{proof}[Proof of Proposition \ref{prop:js-bounds}.] \quad \\
First we start with two observations: 1) $\sum_{k=1}^K \LJS(\dve{k}, \vp)$ is strictly convex. 2) $\sum_{k=1}^K \LJS(\dve{k}, \vp)$ is invariant to permutations of the components of $\vp$.

First, we show Observation 1). This is done by using Remark \ref{r:js-fdiv} and showing that the second derivatives are larger than zero
\begin{align}
        f_{\pi_1}(t) &\coloneqq \Big[ H(\pi_1t + 1-\pi_1) -\pi_1H(t)\Big], t>0 \\
        f_{\pi_1}'(t) &= \Big[ \pi_1 (-\log(\pi_1t + 1-\pi_1) + \log(t))\Big] \\
        f_{\pi_1}''(t) &= \frac{\pi_1(1-\pi_1)}{\pi_1t^2 + t(1-\pi_1)} 
\end{align}
Hence, $f_{\pi_1}(t)$ is strictly convex, since $\pi_1>0$ and $t>0$, then $f_{\pi_1}''(t) > 0$. With Remark \ref{r:js-fdiv}, and that the sum of strictly convex functions is also strictly convex, it follows that $\sum_{k=1}^K \LJS(\dve{k}, \vp)$ is strictly convex.

Next, we show Observation 2), \ie that $\sum_{k=1}^K \LJS(\dve{k}, \vp)$ is invariant to permutations of $\vp$
\begin{align}
    \sum_{k=1}^K\JSnp(\dve{k}, \vp) 
    &= \sum_{k=1}^K \Big[H(\pi_1\dve{k} + \pi_2\vp) - \pi_2H(\vp)\Big] \\
    &= \sum_{k=1}^K \Big[H(\pi_1 + \pi_2p_k) + \sum_{i\not= k}^K H(\pi_2p_i) - \pi_2H(\vp)\Big] \\
    &= \sum_{k=1}^K H(\pi_1 + \pi_2p_k) + \sum_{k=1}^K\sum_{i\not= k}^K H(\pi_2p_i) - \pi_2KH(\vp) \\
    &= \sum_{k=1}^K H(\pi_1 + \pi_2p_k) + \sum_{k=1}^K\Big[H(\pi_2\vp) - H(\pi_2p_k)\Big] - \pi_2KH(\vp) \\
    &= \sum_{k=1}^K H(\pi_1 + \pi_2p_k) + (K-1)H(\pi_2\vp) - \pi_2KH(\vp) \\
    &= \sum_{k=1}^K H(\pi_1 + \pi_2p_k) + (K-1)(H(\pi_2) +\pi_2H(\vp)) - \pi_2KH(\vp) \\
    &= \sum_{k=1}^K H(\pi_1 + \pi_2p_k) + (K-1)H(\pi_2) - \pi_2H(\vp)
\end{align}
Clearly, a permutation of the components of $\vp$ does not change the first sum or $H(\vp)$, since it would simply reorder the summands. Hence,  $\sum_{k=1}^K \LJS(\dve{k}, \vp)$ is invariant to permutations of $\vp$.

\textbf{Lower bound:} \\
The minimizer of a strictly convex function$\Big(\sum_{k=1}^K \LJS(\dve{k}, \vp)\Big)$ over a compact convex set$\Big(\Delta^{K-1}\Big)$ is unique. Since $\vu$ is the only element of $\Delta^{K-1}$ that is the same under permutation, it is the unique minimum of $\sum_{k=1}^K \LJS(\dve{k}, \vp)$ for $\vp \in \Delta^{K-1}$.

\textbf{Upper bound:} \\
The maximizer of a strictly convex function$\Big(\sum_{k=1}^K \LJS(\dve{k}, \vp)\Big)$ over a compact convex set$\Big(\Delta^{K-1}\Big)$ is at its extreme points$\Big(\dve{i}$ for $i \in \{1, 2, \dots, K\}\Big)$. All extreme points have the same value according to Observation 2).

\end{proof}

\subsubsection{Bounds for GJS}
\propGjsBounds*
\begin{proof}[Proof of Proposition \ref{prop:gjs-bounds}] \quad \\
\textbf{Lower bound:} Using Proposition \ref{prop:GJS-JS-Consistency} to rewrite $\DGJSnp$ into a $\DJSnp$ and a consistency term, we get
\begin{align}
    \sum_{k=1}^K \GJS(\dve{k}, \dvp{2}, \dots,& \dvp{M}) = \sum_{k=1}^K\Big[D_{\mathrm{JS}_{\vpi'}}(\dve{k},\dvm) + (1-\pi_1)D_{\mathrm{GJS}_{\vpi''}}(\dvp{2},\dots,\dvp{M})\Big] \\
    &= \sum_{k=1}^KD_{\mathrm{JS}_{\vpi'}}(\dve{k},\dvm) + (1-\pi_1)KD_{\mathrm{GJS}_{\vpi''}}(\dvp{2},\dots,\dvp{M}) \\
    &\geq \sum_{k=1}^KD_{\mathrm{JS}_{\vpi'}}(\dve{k},\vu) + (1-\pi_1)KD_{\mathrm{GJS}_{\vpi''}}(\dvp{2},\dots,\dvp{M}) \\
    &\geq \sum_{k=1}^KD_{\mathrm{JS}_{\vpi'}}(\dve{k},\vu)
\end{align}
where the first inequality comes from the lower bound of Proposition \ref{prop:js-bounds}, and the second inequality comes from \\ $(1-\pi_1)KD_{\mathrm{GJS}_{\vpi''}}(\dvp{2},\dots,\dvp{M})$ being non-negative. The inequalities holds with equality if and only if \\ $\dvp{2}=\dots=\dvp{M}=\vu$. Notably, the lower bound of $\DJSnp$ is the same as that of $\DGJSnp$. \\

\textbf{Upper bound:} \\
Let's denote $A(\dvp{2},\dots,\dvp{M}) = \sum_{k=1}^K \LGJS(\dve{k}, \dvp{2}, \dots, \dvp{M})$.
First we start by making 5 observations:\\
Observation 1: $\Delta^{K-1}_{M-1} = \Delta^{K-1} \times \Delta^{K-1} \times \dots \times \Delta^{K-1}$ is a compact convex set.\\
Observation 2: $A$ is strictly convex over $\Delta^{K-1}_{M-1}$.\\
Observation 3: From Observations 1 and 2 we have that the maximizer of $A$ should be at extreme points of $\Delta^{K-1}_{M-1}$, \ie, a unit vector in every $M-1$ individual $\Delta^{K-1}$ subspaces of $\Delta^{K-1}_{M-1}$.\\
Observation 4: $A$ is symmetric w.r.t. permutations of the components of predictive distributions $\dvp{i}$.\\

Unlike for $\DJSnp$, the extreme points of $\Delta^{K-1}_{M-1}$ do not necessarily map to the same value of $A$. Hence, what is left to show is that the set of extreme points with all predictive distributions being \textit{distinct} unit vectors maps to the maximum value of $A$.

Given Observation 3, all the M distributions are unit vectors, therefore the maximum is of the form $A(\dvp{2},\dots,\dvp{M})=\sum_{k=1}^K H(\pi_1\dve{k} + (1-\pi_1)\dvm)$, where $\dvm \coloneqq \sum_{j=2}^M \pi_j\dvp{j}/(1-\pi_1)$. Furthermore, at most $M-1$ components of $\dvm$ are non-zero~(if all predictions are distinct). From Observation 4, we can WLOG permute $\dvm$ such that the first $M-1$ components are the largest ones. Let $\dvm^{\subset} \in \Delta^{M-2}$ denote the subset of these first $M-1$ components of $\dvm \in \Delta^{K-1}$. Then, for all predictive distributions being unit vectors, we have
{\scriptsize
\begin{align}
    A(\dvp{2},\dots,\dvp{M}) &= \sum_{k=1}^K H(\pi_1\dve{k} + (1-\pi_1)\dvm) \\
    &= \sum_{k=1}^{M-1} \Big[ H(\pi_1 + (1-\pi_1)m_{>1,k}) + (K-1)H((1-\pi_1)m_{>1,k})\Big] + \sum_{k=M}^KH(\pi_1)  \\
    &= \sum_{k=1}^{M-1} H(\pi_1 + (1-\pi_1)m_{>1,k}) + (K-1)H((1-\pi_1)\dvm^{\subset}) + \sum_{k=M}^KH(\pi_1) \\
    &\leq  (M-1)H(\frac{1}{M-1}\sum_{k=1}^{M-1}\Big[\pi_1 + (1-\pi_1)m_{>1,k}\Big]) + (K-1)H((1-\pi_1))\dvm^{\subset}) + \sum_{k=M}^KH(\pi_1) \\
    &= (M-1)H(\pi_1 + \frac{1-\pi_1}{M-1}) + (K-1)H((1-\pi_1)\dvm^{\subset}) + \sum_{k=M}^KH(\pi_1) \\
    &\leq (M-1)H(\pi_1 + \frac{1-\pi_1}{M-1}) + (K-1)H((1-\pi_1)\vu) + \sum_{k=M}^KH(\pi_1) 
\end{align}}The first inequality follows from Jensen's inequality and the second from the uniform distribution maximizes entropy. Both inequalities hold with equality iff $m_{>1,1}=\dots=m_{>1,M-1}$. Hence, the maximum is achieved if $\dvm^{\subset}=\vu \in \Delta^{M-2}$, which is only possible if all $M-1$ predictive distributions are distinct unit vectors. 

\end{proof}

\subsection{Robustness of Jensen-Shannon losses}
\label{sup:sec:JS-bound-diff}
In this section, we prove that the lower~($B_L$) and upper~($B_U$) bounds become the same for $\DJSnp$ and $\DGJSnp$ as $\pi_1 \rightarrow 1$ as stated in Remark \ref{r:JSGJSRobustness}. 

\rJSGJSRobustness*
\begin{proof}[Proof of Remark \ref{r:JSGJSRobustness} for $\DJSnp$] \quad \\
\textbf{Lower bound:}
\begin{align}
        \sum_{k=1}^K \JS(\dve{y}, \mathbf{u}) &= \sum_{k=1}^K H(\pi_1\dve{k} + \pi_2\mathbf{u}) - \pi_2H(\mathbf{u}) \\
        &= K[H(\pi_1\dve{1} + \pi_2\mathbf{u}) - \pi_2H(\mathbf{u})] \\
    &=K[H(\pi_1 + \pi_2/K) + (K-1)H(\pi_2/K) - K\pi_2H(\frac{1}{K})] \\
    &= / H(\pi_2/K) = -\pi_2/K (\log{\pi_2} + \log{1/K}) = \frac{1}{K}H(\pi_2) + \pi_2H(1/K) /  \\
    &=K[H(\pi_1 + \pi_2/K) + (K-1)( \frac{1}{K}H(\pi_2) + \pi_2H(\frac{1}{K})) - K\pi_2H(\frac{1}{K})] \\
    &=K[H(\pi_1 + \pi_2/K) + (K-1)\frac{1}{K}H(\pi_2)  - \pi_2H(\frac{1}{K})]
\end{align}
If one now normalize($Z=H(\pi_2)=H(1-\pi_1)$) and take the limit as $\pi_1 \rightarrow 1$ we get:
\begin{align}
        &\lim_{\pi_1 \rightarrow 1} \sum_{k=1}^K \LJS(\dve{y}, \mathbf{u}) = \lim_{\pi_1 \rightarrow 1}(K-1) + K\frac{H(\pi_1 + \pi_2/K) - \pi_2H(\frac{1}{K})}{H(\pi_2)} \label{eq:frac-hopital} \\
        &= \lim_{\pi_1 \rightarrow 1} (K-1) + K\frac{-(K-1) (1 + \log{(\pi_1 + \pi_2/K))/K} -\log{(1/K)}/K}{\log{(1-\pi_1) + 1}} \\
        &= \lim_{\pi_1 \rightarrow 1} (K-1) - \frac{(K-1) (1 + \log{(\pi_1 + \pi_2/K))} -\log{(1/K)}}{\log{(1-\pi_1) + 1}} \\
        &= \lim_{\pi_1 \rightarrow 1} (K-1) - ((K-1) (1 + \log{(\pi_1 + \pi_2/K))} -\log{(1/K)}) \frac{1}{\log{(1-\pi_1) + 1}} \\
        &=(K-1) - (K-1 -\log{(1/K)}) \cdot 0 \\
        &= K-1
\end{align}
where L'H\^{o}pital's rule was used for the fraction in Equation \ref{eq:frac-hopital} which is indeterminate of the form $\frac{0}{0}$. \\
\textbf{Upper bound:}

\begin{align}
    \sum_{k=1}^K \LJS(\dve{k},\dve{1}) &= \frac{1}{H(\pi_2)}\sum_{k=1}^K H(\pi_1\dve{k} + \pi_2\dve{1}) \\ 
    &= \frac{1}{H(\pi_2}[(K-1)H(\pi_2) + (K-1)H(\pi_1) + H(\pi_1 + \pi_2)] \\
    &=(K-1)[1 + \frac{H(\pi_1)}{H(\pi_2)}] \\
    &= (K-1)\Bigg[1 + \frac{\pi_1\log{\pi_1}}{(1-\pi_1)\log{(1-\pi_1)}}\Bigg] 
\end{align}
Taking the limit as $\pi_1 \rightarrow 1$ gives
\begin{align}
    \lim_{\pi_1 \rightarrow 1}\sum_{k=1}^K \LJS(\dve{k},\dve{1}) 
    &= \lim_{\pi_1 \rightarrow 1}(K-1)\Bigg[1 + \pi_1\frac{1}{\log{(1-\pi_1)}} \frac{\log{\pi_1}}{(1-\pi_1)}\Bigg] \\
    &= \lim_{\pi_1 \rightarrow 1}(K-1)\Bigg[1 + \pi_1\frac{1}{\log{(1-\pi_1)}}\frac{1}{\pi_1}\frac{1}{-1}\Bigg] \\
    &= (K-1)[1+ 1 \cdot 0 \cdot 1 \cdot -1] \\
    &= K-1
\end{align}
where L'H\^{o}pital's rule was used for $\lim_{\pi_1 \rightarrow 1}\frac{\log{\pi_1}}{(1-\pi_1)}$ which is indeterminate of the form $\frac{0}{0}$. \\
Hence, $B_L=B_U=K-1$.
\end{proof}
Next, we look at the robustness of the generalized Jensen-Shannon loss.
\begin{proof}[Proof of Remark \ref{r:JSGJSRobustness} for $\DGJSnp$] \quad \\
Proposition \ref{prop:GJS-JS-Consistency}, shows that $\DGJSnp$ can be rewritten as a $\DJSnp$ term and a consistency term. From the proof of Remark \ref{r:JSGJSRobustness} for $\DJSnp$ above, it follows that the $\DJSnp$ term satisfies $B_L=B_U$ as $\pi_1$ approaches 1. Hence, it is enough to show that the consistency term of $\DGJSnp$ also becomes a constant in this limit. The consistency term is the generalized Jensen-Shannon divergence 
\begin{align}
    \lim_{\pi_1 \rightarrow 1}(1-\pi_1)\mathcal{L}_{\mathrm{GJS_{\vpi''}}}(\dvp{2},\dots,\dvp{M}) &= \lim_{\pi_1 \rightarrow 1} \frac{(1-\pi_1)}{H(1-\pi_1)}D_{\mathrm{GJS}_{\vpi''}}(\dvp{2},\dots,\dvp{M}) \\
    &= \lim_{\pi_1 \rightarrow 1} -\frac{1}{\log{(1-\pi_1)}}D_{\mathrm{GJS}_{\vpi''}}(\dvp{2},\dots,\dvp{M}) \\
    &= 0
\end{align}
where $\vpi''=[\pi_2, \dots, \pi_M]/(1-\pi_1)$. $D_{\mathrm{GJS}_{\vpi''}}(\dvp{2},\dots,\dvp{M})$ is bounded and $-\frac{1}{\log{(1-\pi_1)}}$ goes to zero as $\pi_1 \rightarrow 1$, hence the limit of the product goes to zero.
\end{proof}

\subsection{Gradients of Jensen-Shannon Divergence}
\label{sup:js-grad}
The partial derivative of the Jensen-Shannon divergence is
\begin{align*}
    \frac{\partial \{H(\vm)  - \pi_1 H(\dve{y}) - (1-\pi_1)H(\vp)  \}}{ \partial z_i}
\end{align*}
where $\vm=\pi_1\dve{y} + \pi_2\vp=\pi_1\dve{y} + (1-\pi_1)\vp$, and $p_j = e^{z_j} / \sum_{k=1}^K e^{z_k}$. Note the difference between $e^{z}$ which is the exponential function while $\dve{y}$ is a onehot label. We take the partial derivative of each term separately, but first the partial derivative of the $j$th component of a softmax output with respect to the $i$th component of the corresponding logit
\begin{align}
    \frac{\partial p_j}{\partial z_i} &= \frac{\partial}{\partial z_i}\frac{e^{z_j}}{\sum_{k=1}^K e^{z_k}}  \\
    &= \frac{\frac{\partial e^{z_j}}{\partial z_i}\sum_{k=1}^K e^{z_k} - e^{z_j}\frac{\partial \sum_{k=1}^K e^{z_k} }{\partial z_i }}{\Big(\sum_{k=1}^K e^{z_k}\Big)^2}  \\
    &= \frac{\mathbbm{1}(i=j)e^{z_j}\sum_{k=1}^K e^{z_k} - e^{z_j}e^{z_i}}{\Big(\sum_{k=1}^K e^{z_k}\Big)^2}  \\
    &= \frac{\mathbbm{1}(i=j)e^{z_j} - p_je^{z_i} }{\sum_{k=1}^K e^{z_k}}  \\
    &= \mathbbm{1}(i=j)p_j - p_jp_i \\
    &= p_j(\mathbbm{1}(i=j)-p_i)  \\
    &= p_i(\mathbbm{1}(i=j)-p_j) = \frac{\partial p_i}{\partial z_j}
\end{align}
where $\mathbbm{1}(i=j)$ is the indicator function, \ie 1 when $i=j$ and zero otherwise. Using the above, we get
\begin{align}
    \sum_{j=1}^K \frac{\partial p_j}{\partial z_i} = p_i\sum_{j=1}^K (\mathbbm{1}(i=j)-p_j) = p_i(1-1)=0 
    \label{eq:partialSum}
\end{align}
First, the partial derivative of $H(\vp)$ wrt $z_i$
\begin{align}
    \frac{\partial H(\vp)}{\partial z_i} &= - \sum_{j=1}^K \frac{\partial p_j\log{p_j}}{\partial z_i}  \\ 
    &=  - \sum_{j=1}^K \frac{\partial p_j}{\partial z_i}\log{p_j} + p_j\frac{\partial \log{p_j}}{\partial z_i}  \\
    &= - \sum_{j=1}^K \frac{\partial p_j}{\partial z_i}\log{p_j} + p_j\frac{1}{p_j}\frac{\partial p_j}{\partial z_i}  \\
    &= - \sum_{j=1}^K \frac{\partial p_j}{\partial z_i}\left( \log{p_j} + 1\right) \\
    &= / \text{ Equation \ref{eq:partialSum} } /  \\
    &= - \sum_{j=1}^K \frac{\partial p_j}{\partial z_i}\log{p_j}
\end{align}
Next, the partial derivative of $H(\vm)$ wrt $z_i$
\vspace{-0cm}
\begin{align}
    \frac{\partial \{H(\vm) \}}{ \partial z_i} &= \frac{\partial  \{\pi_1 H(\dve{y},\vm) + (1-\pi_1)H(\vp,\vm) \}}{ \partial z_i} \\
    &= -\sum_{j=1}^K \Big[ \pi_1 \frac{ \edve{y}{j}\partial \log{(m_j)}}{\partial z_i} + (1-\pi_1) \frac{ \partial \{p_j \log{(m_j)} \} } {\partial z_i} \Big]  \\
    &= -\sum_{j=1}^K \Big[ \pi_1 \edve{y}{j} \frac{ \partial \log{(m_j)}}{\partial z_i} + (1-\pi_1)\Big( \frac{ \partial p_j} {\partial z_i}\log{(m_j)} + p_j\frac{ \partial  \log{(m_j)} } {\partial z_i} \Big) \Big]  \\
    &= -\sum_{j=1}^K \Big[ m_j\frac{ \partial \log{(m_j)}}{\partial z_i} + (1-\pi_1) \frac{ \partial p_j} {\partial z_i}\log{(m_j)} \Big]  \\
    &= -\sum_{j=1}^K \Big[ (1-\pi_1)\frac{ \partial p_j}{\partial z_i} + (1-\pi_1) \frac{ \partial p_j} {\partial z_i}\log{(m_j)} \Big] \\
    &= -\sum_{j=1}^K (1-\pi_1)\frac{ \partial p_j}{\partial z_i}\Big[ 1 + \log{(m_j)} \Big] = / \text{ Equation \ref{eq:partialSum} } / \\
    &= -(1-\pi_1)\sum_{j=1}^K \frac{ \partial p_j}{\partial z_i}\log{(m_j)}
\end{align}
\vspace{-0cm}
The partial derivative of the Jensen-Shannon divergence with respect to logit $z_i$ is
\vspace{-0cm}
\begin{align}
    \frac{\partial \{H(\vm)  - \pi_1 H(\dve{y}) - (1-\pi_1)H(\vp)  \}}{ \partial z_i}&=  \frac{\partial \{H(\vm) - (1-\pi_1)H(\vp)  \}}{ \partial z_i}\\
    &= -(1-\pi_1) \sum_{j=1}^K \frac{ \partial p_j}{\partial z_i}\Big(\log{(m_j)} -\log{p_j}\Big)   \\
     &= -(1-\pi_1)\Big[ \sum_{j=1}^K \frac{ \partial p_j}{\partial z_i}\log{\frac{m_j}{p_j}} \Big] 
\end{align}
If we now make use of the fact that the label is $\dve{y}$, we can write the partial derivative wrt to $z_i$ as
\begin{align}
     &\frac{\partial \{H(\vm)  - \pi_1 H(\dve{y}) - (1-\pi_1)H(\vp)  \}}{ \partial z_i} = \\
     &= -(1-\pi_1)\Big[ \sum_{j=1}^K \frac{ \partial p_j}{\partial z_i}\log{\Bigg(\frac{\pi_1 \edve{y}{j}}{p_j} + (1-\pi_1) \Bigg) } \Big]   \\
     &= -(1-\pi_1)\Big[ \frac{ \partial p_y}{\partial z_i}\log{\Bigg(\frac{\pi_1}{p_y} + (1-\pi_1) \Bigg) } +  \sum_{j\not=y}^K \frac{ \partial p_j}{\partial z_i}\log{\Bigg(1-\pi_1 \Bigg) } \Big]   \\
     &= -(1-\pi_1)\Big[ \frac{ \partial p_y}{\partial z_i}\log{\Bigg(\frac{\pi_1}{p_y} + (1-\pi_1) \Bigg) } + \log{\Bigg(1-\pi_1 \Bigg) } \sum_{j\not=y}^K \frac{ \partial p_j}{\partial z_i} \Big]   \\
     &= \Bigg/ \text{Eq \ref{eq:partialSum}} \Leftrightarrow \sum_{j\not=y}^K \frac{\partial p_j}{\partial z_i} = -\frac{\partial p_y}{\partial z_i}  \Bigg/   \\
      &= -(1-\pi_1)\frac{ \partial p_y}{\partial z_i}\Big[\log{\Bigg(\frac{\pi_1}{p_y} + (1-\pi_1) \Bigg) } - \log{\Bigg(1-\pi_1 \Bigg) }\Big]  \\
      &= -(1-\pi_1) \frac{ \partial p_y}{\partial z_i}\log{\Bigg(\frac{\pi_1}{(1-\pi_1)p_y} + 1 \Bigg)  } \label{eq:jsgrad} 
\end{align}

\section{Extended Related Works}
\label{sup:rel-works}
Most related to us is the avenue of handling noisy labels in deep learning via the identification and construction of \textit{noise-robust loss functions}~\citep{Ghosh_AAAI_2017_MAE,Zhang_NeurIPS_2018_Generalized_CE,Wang_ICCV_2019_Symmetric_CE,Ma_ICML_2020_Normalized_Loss}. Ghosh~\etal~\citep{Ghosh_AAAI_2017_MAE} derived sufficient conditions for a loss function, in empirical risk minimization (ERM) settings, to be robust to various kinds of sample-independent noise, including symmetric, symmetric non-uniform, and class-conditional. They further argued that, while CE is not a robust loss function, mean absolute error (MAE) is a loss that satisfies the robustness conditions and empirically demonstrated its effectiveness. On the other hand, Zhang~\etal~\citep{Zhang_NeurIPS_2018_Generalized_CE} pointed out the challenges of training with MAE and proposed GCE which generalizes both MAE and CE losses. Tuning for this trade-off, GCE alleviates MAE's training difficulties while retaining some desirable noise-robustness properties. In a similar fashion, symmetric cross entropy (SCE)~\citep{Wang_ICCV_2019_Symmetric_CE} spans the spectrum of reverse CE as a noise-robust loss function and the standard CE. Recently, ~Ma~\etal~\citep{Ma_ICML_2020_Normalized_Loss} proposed a normalization mechanism to make arbitrary loss functions robust to noise. They, too, further combine two complementary loss functions to improve the data fitting while keeping robust to noise. The current work extends on this line of works.

Several other directions are pursued to improve training of deep networks under noisy labeled datasets. This includes methods to \textit{identify and remove} noisy labels~\citep{Nguyen_ICLR_2020_self_ensemble,Northcutt_arXiv_2017} or \textit{identify and correct} noisy labels in a joint label-parameter optimization~\citep{ Tanaka_CVPR_2018_Joint_Optimization,Vahdat_NeurIPS_2017_CRF} and those works that design an \textit{elaborate training pipeline} for dealing with noise~\citep{Li_ICLR_2020_dividemix,Iscen_ECCV_2020,Seo_NeurIPS_2019}. In contrast to these directions, this work proposes a robust loss function based on Jensen-Shannon divergence (JS) without altering other aspects of training. In the following, we review the directions that are most related to this paper. 

A close line of works to ours \textit{reweight a loss function} by a known or estimated class-conditional noise model~\citep{Natarajan_NIPS_2013}. This direction has been commonly studied for deep networks with a standard cross entropy (CE) loss~\citep{Sukhbaatar_ICLR_2015_confusion_matrix,Patrini_CVPR_2017,Han_NeurIPS_2018,Xia_NeurIPS_2019}. Assuming a class-conditional noise model, loss correction is theoretically well motivated.  

A common regularization technique called \textit{label smoothing}~\citep{Szegedy_CVPR_2016_inception_label_smoothing} has been recently proposed that operates similarly to the loss correction methods. While its initial purpose was for deep networks to avoid overfitting, label smoothing has been shown to have a noticeable effect when training with noisy sets by alleviating the fit to the noise~\citep{Lukasik_ICML_2020_label_smoothing_label_noisy,Reed_arXiv_2014_bootstrapping}.

\textit{Consistency regularization} is a recently-developed technique that encourages smoothness in the learnt decision boundary by requiring minimal shifts in the learnt function when small perturbations are applied to an input sample. This technique has become increasingly common in the state-of-the-art semi-supervised learning~\citep{Miyato_PAMI_2018_VAT,Berthelot_NeurIPS_2019_mixmatch,Tarvainen_NIPS_2017_mean_teacher} and recently for dealing with noisy data~\citep{Li_ICLR_2020_dividemix}. These methods use various complicated pipelines to integrate consistency regularization in training. This work shows that a multi-distribution generalization of JS can neatly incorporate such regularization.\\
Hendrycks~\etal~\citep{hendrycks2020augmix} recently proposed AugMix, a novel data augmentation strategy in combination with a $\DGJSnp$ consistency loss to improve uncertainty estimation and robustness to image corruptions at test-time. Our work is orthogonal since we consider the task of learning under noisy labels at training time and conduct the corresponding experiments. We also investigate and derive the theoretical properties of the proposed loss functions. Finally, our losses are solely implemented based on $\DJSnp$/$\DGJSnp$ instead of a combination of CE and $\DGJSnp$ in case of AugMix.  However, we find it promising that $\DGJSnp$ improves robustness to both training-time label noise and test-time image corruption, which further strengthens the significance of the JS-based loss functions. 

Finally, recently, Xu~\etal~\citep{Xu_NeurIPS_2019_Information_Theoretic_Mutual_Info_Loss};~Wei~\&~Liu~\citep{Wei_ICLR_2021_f_Divergence} propose loss functions with \textit{information theory} motivations. Jensen-Shannon divergence, with inherent information theoretic interpretations, naturally posits a strong connection of our work to those. Especially, the latter is a close \textit{concurrent} work that studies the general family of $f$-divergences but takes a different and complementary angle. In this work, we analyze the role of $\pi_1$, which they treat as a constant. Varying $\pi_1$ is important because it leads to:
\begin{itemize}
    \item \textbf{Better empirical performance.} For our experiments on CIFAR, we provide the hyper-parameters used in Table \ref{tab:cifar-hps}, from which we can see that the optimal is equal to their setting ($\pi_1=0.5$) in only 3/14 cases.
    \item \textbf{Interesting theoretical connections to related work.} In Proposition \ref{prop:js-ce-mae}, we show that the JS loss has CE and MAE as asymptotes when $\pi_1$ goes to zero and one, respectively. This causes an interesting trade-off between learnability and robustness as discussed in Section \ref{sec:ablation}.
\end{itemize}
Furthermore, we consider the generalization to more than two distributions which have proved helpful while Wei~\&~Liu~\citep{Wei_ICLR_2021_f_Divergence} only study two distributions.


In this work, we use a generalization of the Jensen-Shannon divergence to more than two distributions, which was introduced by Lin~\citep{Lin_TIT_1991_JS_Divergence}. Recently, another generalization of JS was presented by Nielsen~\citep{nielsen2019JSMeans}, where the arithmetic mean is generalized to abstract means. JS is also a special case of a general family of divergences, the f-divergences~\citep{csiszar1967fdiv}.

\end{document}